\documentclass[10pt]{article}
\usepackage[margin=1in]{geometry}

\usepackage[utf8]{inputenc} 
\usepackage[T1]{fontenc}    

\usepackage{lmodern}
\usepackage[unicode=true,
 bookmarks=false,
 breaklinks=false,pdfborder={0 0 1},colorlinks=false]
 {hyperref}
\hypersetup{
 colorlinks,citecolor=blue,filecolor=blue,linkcolor=blue,urlcolor=blue}
\usepackage{url}            
\usepackage{booktabs}       
\usepackage{xcolor}
\usepackage{amsfonts}       
\usepackage{nicefrac}       
\usepackage{microtype}      
\usepackage{appendix}
\usepackage{amsthm}
\usepackage{amsmath}
\usepackage{amssymb}
\usepackage{amsfonts}
\usepackage{mathtools}
\usepackage{enumerate}
\usepackage{enumitem}
\usepackage[linesnumbered,ruled,vlined]{algorithm2e}
\usepackage{algorithmic} 
\usepackage{xspace}
\usepackage{mathabx}
\usepackage{graphicx}
\usepackage{tabularx, makecell, booktabs}

\usepackage{verbatim}
\usepackage{graphicx}
\usepackage{subfigure}

\usepackage{colortbl}
\usepackage{multirow}
\usepackage{dsfont}

\usepackage[authoryear]{natbib}

\newcommand{\Tone}{T_1}
\newcommand{\weighteddata}{\mathcal{S}}
\newcommand{\VC}{\mathsf{VC}\text{-}\mathsf{dim}}
\newcommand{\basis}{\mathrm{e}^{\mathsf{basis}}}

\allowdisplaybreaks

\newtheorem{theorem}{\textbf{Theorem}}

\newtheorem{assum}{\textbf{Assumption}}

\newtheorem{lemma}{\textbf{Lemma}}
\newtheorem{definition}{\textbf{Definition}}

\newtheorem{remark}{\textbf{Remark}}

\ifdefined\usebigfont
\usepackage{times}
\usepackage[fontsize=13pt]{scrextend}

\makeatletter
\@ifpackageloaded{geometry}{\AtBeginDocument{\newgeometry{left=1.56in,right=1.56in,top=1.7in,bottom=1.79in}}}{\usepackage[letterpaper,left=1.56in,right=1.56in,top=1.7in,bottom=1.79in]{geometry}} 
\makeatother
\else
\fi

\ifdefined\useregularfont
\usepackage[fontsize=10pt]{scrextend}

\makeatletter
\@ifpackageloaded{geometry}{\AtBeginDocument{\newgeometry{margin=1in}}}{\usepackage[margin=1in]{geometry}} 
\makeatother
\else
\fi

\ifdefined\usearxivfont
\usepackage[fontsize=12pt]{scrextend}
\makeatletter
\@ifpackageloaded{geometry}{\AtBeginDocument{\newgeometry{margin=1.25in}}}{\usepackage[margin=1.25in]{geometry}} 
\makeatother
\else
\fi

\definecolor{yxc}{RGB}{255,0,0}
\newcommand{\yxc}[1]{\textcolor{yxc}{[YXC: #1]}}

\title{Optimal Multi-Distribution Learning\footnotetext{Accepted for presentation at the
Conference on Learning Theory (COLT) 2024.}}

\author{%
	Zihan Zhang\thanks{Department of Electrical and Computer Engineering, Princeton University; email: \texttt{\{zz5478, wenhao.zhan, jasonlee\}@princeton.edu}.}\\
  Princeton  \and
 Wenhao Zhan\footnotemark[1] \\
 Princeton  
\and 
  Yuxin Chen\thanks{Department of Statistics and Data Science, University of Pennsylvania; email: \texttt{yuxinc@wharton.upenn.edu}.}\\
 UPenn
  \and
 Simon S.~Du\thanks{Paul G. Allen School of Computer Science and Engineering, University of Washington; email: \texttt{ssdu@cs.washington.edu}.}\\
 U.~Washington
 \and
 Jason D.~Lee\footnotemark[1] \\
 Princeton 
}

\date{December 2023;~ Revised: August 2025}

\begin{document}

\maketitle
\begin{abstract}

Multi-distribution learning (MDL), which seeks to learn a shared model that minimizes the worst-case risk across $k$ distinct data distributions, 
has emerged as a unified framework in response to the evolving demand for robustness, fairness, multi-group collaboration, etc.  
Achieving data-efficient MDL necessitates adaptive sampling, also called on-demand sampling, throughout the learning process. 
However, there exist substantial gaps between the state-of-the-art upper and lower bounds on the optimal sample complexity. 
Focusing on a hypothesis class of Vapnik–Chervonenkis (VC) dimension $d$,  
we propose a novel algorithm that yields an $\varepsilon$-optimal randomized hypothesis with a sample complexity on the order of $\frac{d+k}{\varepsilon^2}$ (modulo some logarithmic factor), 
matching the best-known lower bound. Our algorithmic ideas and theory are further extended to accommodate Rademacher classes. 
The proposed algorithms are oracle-efficient, which access the hypothesis class solely through an empirical risk minimization oracle. 
Additionally, we establish the necessity of improper learning, revealing a large sample size barrier when only deterministic, proper hypotheses are permitted. 
These findings resolve three open problems presented in COLT 2023 (i.e., \citet[Problems 1, 3 and 4]{awasthi2023sample}). 

\end{abstract}


\setcounter{tocdepth}{2}
\tableofcontents

\section{Introduction }

Driven by  the growing need of robustness, fairness and multi-group collaboration in machine learning practice, 
the multi-distribution learning (MDL) framework has emerged as a unified solution in response to these evolving demands \citep{blum2017collaborative,haghtalab2022demand,mohri2019agnostic,awasthi2023sample}. 
Setting the stage, 
imagine that we are interested in a collection of $k$ unknown data distributions $\mathcal{D}=\{\mathcal{D}_i\}_{i=1}^k$ supported on $\mathcal{X}\times \mathcal{Y}$,    
where $\mathcal{X}$ (resp.~$\mathcal{Y}$) stands for the instance (resp.~label) space. 
Given a hypothesis class $\mathcal{H}$ and a prescribed loss function\footnote{For example, for each hypothesis $h\in \mathcal{H}$ and each datapoint $(x,y)\in \mathcal{X}\times \mathcal{Y}$, we employ $\ell(h,(x,y))$ to measure the risk of using hypothesis $h$ to predict $y$ based on $x$.} $\ell:\mathcal{H}\times\mathcal{X}\times \mathcal{Y} \to [-1,1]$, 
we are asked to identify a (possibly randomized) hypothesis $\widehat{h}$  
achieving near-optimal {\em worst-case} loss across these data distributions, namely,\footnote{Here, the expectation on the left-hand side of \eqref{eq:defproblem} is taken over the randomness of both the datapoints $(x,y)$ and the (randomized) hypothesis $\widehat{h}$.} 
 \begin{align}
	 \max_{1\leq i\leq k} \mathop{\mathbb{E}}\limits_{(x,y)\sim \mathcal{D}_i,\widehat{h}}\big[\ell\big(\widehat{h},(x,y)\big)\big] \leq  \min_{h\in \mathcal{H}}\max_{1\leq i\leq k}\mathop{\mathbb{E}}\limits_{(x,y)\sim \mathcal{D}_i}\big[\ell\big(h,(x,y)\big)\big] + \varepsilon \label{eq:defproblem}
 \end{align}
with $\varepsilon\in (0,1]$ a target accuracy level. 
In light of the unknown nature of these data distributions, 
the learning process is often coupled with data collection, allowing the learner to sample from $\{\mathcal{D}_i\}_{i=1}^k$. 
The performance of a learning algorithm is then gauged by its sample complexity --- the number of samples required to fulfil \eqref{eq:defproblem}.
Our objective is to design a learning paradigm that achieves the optimal sample complexity.

The MDL framework described above, 
which can viewed as an extension of agnostic learning \citep{valiant1984theory,blumer1989learnability} tailored to multiple data distributions, 
has found a wealth of applications across multiple domains. 
Here, we highlight a few representative examples, and refer the interested reader to \citet{haghtalab2022demand} and the references therein for more extensive discussions. 
\begin{itemize}
	\item {\em Collaborative and agnostic federated learning.} 
		In the realm of collaborative and agnostic federated learning \citep{blum2017collaborative,nguyen2018improved,chen2018tight,mohri2019agnostic,blum2021one,du2021fairness,deng2020distributionally,blum2021communication}, a group of $k$ agents, each having access to distinct data sources as characterized by different data distributions $\{\mathcal{D}_i\}_{i=1}^k$, aim to learn a shared prediction model that ideally would achieve low risk for each of their respective data sources. 
		A sample-efficient MDL paradigm would help unleash the potential of collaboration and information sharing in jointly learning a complicated task.

	\item {\em Min-max fairness in learning.} 
		The MDL framework is well-suited to scenarios requiring fairness across multiple groups \citep{dwork2021outcome,rothblum2021multi,du2021fairness}. 
		For instance, 
		in situations where multiple subpopulations with distinct data distributions exist, 
		a prevailing objective is to ensure that the learned model does not adversely impact any of these subpopulations. 
		One criterion designed to meet this objective, known as ``min-max fairness'' in the literature \citep{mohri2019agnostic,abernethy2020active}, 
		plays a pivotal role in mitigating  the worst disadvantage experienced by any particular subpopulation.

	\item {\em Distributionally robust optimization/learning.} 
		Another context where MDL naturally finds applications is group distributionally robust optimization and learning (DRO/DRL).  
		Group DRO and DRL aim to develop algorithms that offer robust performance guarantees across a finite number of possible distributional models~\citep{sagawa2019distributionally,sagawa2020investigation,hashimoto2018fairness,hu2018does,xiong2023distributionally,zhang2020coping,wang2023distributionally,deng2020distributionally}, 
		and have garnered substantial attention recently due to the pervasive need for robustness in modern decision-making \citep{carmon2022distributionally,asi2021stochastic,haghtalab2022demand,kar2019meta}. When applying MDL to the context of group DRO/DRL, the resultant sample complexity reflects the price that needs to be paid for learning a robust solution.  
\end{itemize}

\noindent 
The MDL framework is also closely related to other topics like multi-source domain adaptation, maximum aggregation, to name just a few  \citep{mansour2008domain,zhao2020multi,buhlmann2015magging,guo2023statistical}.

In contrast to single-distribution learning, 
achieving data-efficient MDL necessitates adaptive sampling throughout the learning process, also known as on-demand sampling  \citep{haghtalab2022demand}.  
More specifically, pre-determining a sample-size budget for each distribution beforehand and sampling non-adaptively could result in loss of sample efficiency, 
as we lack knowledge about the complexity of learning each distribution before the learning process begins. 
The question then comes down to how to optimally adapt the online sampling strategy to effectively tackle diverse data distributions.

\paragraph{Inadequacy of prior results.}
The sample complexity of MDL has been explored in a strand of recent works under various settings. 
Consider first the case where the hypothesis class $\mathcal{H}$ comprises a {\em finite} number of hypotheses. 
If we sample non-adaptively and draw the same number of samples from each individual distribution $\mathcal{D}_i$, then this results in  a total sample size exceeding the order of $\frac{k\log(|\mathcal{H}|)}{\varepsilon^2}$ (given that learning each distribution requires a sample size at least on the order of $\frac{\log(|\mathcal{H}|)}{\varepsilon^2}$). 
Fortunately, this sample size budget can be significantly reduced with the aid of adaptive sampling.  
In particular, the state-of-the-art approach, proposed by \citet{haghtalab2022demand}, accomplishes the objective \eqref{eq:defproblem} with probability at least $1-\delta$ using $O\big( \frac{\log(|\mathcal{H}|)+k\log(k/\delta)}{\varepsilon^2}\big)$ samples. 
In comparison to agnostic learning on a single distribution, it only incurs an extra {\em additive cost} of $k\log(k/\delta)/\varepsilon^2$ as opposed to a multiplicative factor in $k$, 
thus underscoring the importance of adaptive sampling.

A more challenging scenario arises when $\mathcal{H}$ has a finite Vapnik–Chervonenkis (VC) dimension $d$. 
The sample complexity for VC classes has only been settled for the realizable case \citep{blum2017collaborative,chen2018tight,nguyen2018improved}, a special scenario where 
the loss function takes the form of $\ell\big(h,(x,y)\big) = \mathds{1}\{ h(x)\neq y\}$ and 
it is feasible to achieve zero mean loss. 
For the general non-realizable case,  
the best-known lower bound for such VC classes is \citep{haghtalab2022demand}\footnote{Let $\mathcal{X}=\big(k,d,\frac{1}{\varepsilon},\frac{1}{\delta}\big)$. 
Here and throughout, the notation $f(\mathcal{X})=O\big(g(\mathcal{X})\big)$ or $f(\mathcal{X}) \lesssim g(\mathcal{X})$ (resp.~$f(\mathcal{X})=\Omega\big(g(\mathcal{X})\big)$ or $f(\mathcal{X}) \gtrsim g(\mathcal{X})$) mean that there exists some universal constant $C_1>0$ (resp.~$C_2>0$) such that $f(\mathcal{X})\leq C_1g(\mathcal{X})$ (resp.~$f(\mathcal{X})\geq C_2g(\mathcal{X})$); the notation $f(\mathcal{X})=\Theta\big(g(\mathcal{X})\big)$ or $f(\mathcal{X})\asymp g(\mathcal{X})$ mean $f(\mathcal{X})=O\big(g(\mathcal{X})\big)$  and $f(\mathcal{X})=\Omega\big(g(\mathcal{X})\big)$ hold simultaneously.  
The notation $\widetilde{O}(\cdot)$, $\widetilde{\Theta}(\cdot)$ and $\widetilde{\Omega}(\cdot)$ are defined analogously except that all log factors in $\big(k,d,\frac{1}{\varepsilon},\frac{1}{\delta}\big)$ are hidden. }  
\begin{equation}
	\widetilde{\Omega}\left(\frac{d+k}{\varepsilon^2}\right),
	\label{eq:best-lower}
\end{equation}
which serves as a theoretical benchmark. 
By first collecting $\widetilde{O}\left(\frac{dk}{\varepsilon}\right)$ samples to help construct a cover of $\mathcal{H}$ with reasonable resolution, 
\citet{haghtalab2022demand} established a sample complexity upper bound of 
\begin{subequations}
\label{eq:previous-achievability-all}
\begin{equation}
	\text{\citep{haghtalab2022demand}} \qquad \widetilde{O}\left(\frac{d+k}{\varepsilon^2} + \frac{dk}{\varepsilon} \right). 
	\label{eq:previous-best}
\end{equation}
Nevertheless, the term $dk/\varepsilon$ in \eqref{eq:previous-best} fails to match the lower bound \eqref{eq:best-lower}; 
put another way, this term might result in a potentially large burn-in cost, 
	as the optimality of this approach is only guaranteed (up to log factors) when the total sample size already exceeds a (potentially large) threshold on the order of $\frac{d^2k^2}{d+k}$. 
In an effort to alleviate this $dk/\varepsilon$ factor,   
\citet{awasthi2023sample} put forward an alternative algorithm --- which utilizes an oracle to learn on a single distribution and obliviates the need for computing an epsilon-net of $\mathcal{H}$ --- yielding a sample complexity of 
\begin{equation}
	\text{\citep{awasthi2023sample}} \qquad \widetilde{O}\left( \frac{d}{\varepsilon^4}+\frac{k}{\varepsilon^2}\right).
	\label{eq:previous-best-2}
\end{equation}
\end{subequations}
However, this result \eqref{eq:previous-best-2} might fall short of optimality as well, 
given that the scaling $d/\varepsilon^4$ is off by a factor of $1/\varepsilon^2$ compared with the lower bound \eqref{eq:best-lower}. 
A more comprehensive  list of past results can be found in Table~\ref{table:rel}.

Given the apparent gap between the state-of-the-art lower bound \eqref{eq:best-lower} and achievability bounds \eqref{eq:previous-achievability-all}, 
a natural question arises:
{
\setlist{rightmargin=\leftmargin}
\begin{itemize}
	\item[]  {\em  \textbf{Question:} Is it plausible to design a multi-distribution learning algorithm with a sample complexity of $\widetilde{O}\left(\frac{d+k}{\varepsilon^2}\right)$ for VC classes, thereby matching the established lower bound \eqref{eq:best-lower}?}
\end{itemize}
}
\noindent 
Notably, this question has been posed as an open problem during the Annual Conference on Learning Theory (COLT) 2023; see \citet[Problem 1]{awasthi2023sample}.

\newcommand{\topsepremove}{\aboverulesep = 0mm \belowrulesep = 0mm} \topsepremove

\begin{table}[tb]
\centering

\begin{tabular}{ c|c }
\toprule
	{\bf Paper} &  {\bf Sample complexity bound} $\vphantom{\frac{1^{7}}{1^{7^{7}}}}$\\
 \toprule 
	\cite{haghtalab2022demand}  &  $ \frac{\log(|\mathcal{H}|)+k}{\varepsilon^2}$ $\vphantom{\frac{1^{7^{7}}}{1^{7^{7}}}}$ \\
  \hline
 
 \cite{haghtalab2022demand}  & $\frac{d+k}{\varepsilon^2} + \frac{dk}{\varepsilon} $ $\vphantom{\frac{1^{7}}{1^{7^{7}}}}$  \\ 
   \hline
    \cite{awasthi2023sample}  &$\frac{d}{\varepsilon^4} + \frac{k}{\varepsilon^2} $ $\vphantom{\frac{1^{7}}{1^{7^{7}}}}$ \\
    \hline
      \cite{peng2023sample} &$\frac{d+k}{\varepsilon^{2}}\left(\frac{k}{\varepsilon}\right)^{o(1)} $ $\vphantom{\frac{1^{7}}{1^{7^{7}}}}$ \\
    \hline
	\rowcolor{gray!25}
	{\bf
	our work (Theorem~\ref{thm:main})}  & $ \frac{d+k}{\varepsilon^2} $ $\vphantom{\frac{1^{7}}{1^{7^{7}}}}$  \\
   \toprule
   lower bound: \cite{haghtalab2022demand}  & $\frac{d+k}{\varepsilon^2} $ $\vphantom{\frac{1^{7^{7}}}{1^{7^{7}}}}$  \\
   \toprule
\end{tabular}\label{table:rel}
	\caption{Sample complexity bounds of MDL with $k$ data distributions and a hypothesis class of VC dimension $d$. Here, we only report the polynomial dependency and hide all logarithmic dependency on $\big(k,d,\frac{1}{\varepsilon}, \frac{1}{\delta}\big)$.}
\end{table}

\paragraph{A glimpse of our main contributions.} 
The present paper delivers some encouraging news: 
we come up with a new MDL algorithm that successfully resolves the aforementioned open problem in the affirmative.  
Specifically, focusing on a hypothesis class with VC dimension $d$ and a collection of $k$ data distributions, 
our main findings can be summarized in the following theorem.\footnote{Following the definition of the VC dimension, the hypotheses in $\mathcal{H}$ are assumed to be binary-valued for VC classes.} 
Note that the loss function $\ell(\cdot)$ in this theorem is not restricted to take the form of a zero-one loss and can be fairly general.  
\begin{theorem}\label{thm:main}
	There exists an algorithm (see Algorithm~\ref{alg:main1} for details) such that: with probability exceeding $1-\delta$, the randomized hypothesis $h^{\sf final}$ returned by this algorithm achieves
   \begin{align*}
\max_{1\leq i\leq k}\mathop{\mathbb{E}}\limits _{(x,y)\sim\mathcal{D}_{i},h^{\mathsf{final}}}\left[\ell\big(h^{\mathsf{final}},(x,y)\big)\right]\leq \min_{h\in\mathcal{H}}\max_{1\leq i\leq k}\mathop{\mathbb{E}}\limits _{(x,y)\sim\mathcal{D}_{i}}\big[\ell\big(h,(x,y)\big)\big] + \varepsilon, 
    \end{align*}
provided that the total sample size exceeds 
\begin{equation}
	 \frac{d+k}{\varepsilon^{2}} \mathrm{poly}\log\Big(k,d,\frac{1}{\varepsilon},\frac{1}{\delta}\Big) .
	\label{eq:sample-size-ours-intro}
\end{equation}
\end{theorem}
The polylog factor in \eqref{eq:sample-size-ours-intro} will be specified momentarily. 
In a nutshell, we develop the first algorithm that provably achieves a sample complexity matching the lower bound \eqref{eq:best-lower} modulo logarithmic factors. 
Following the game dynamics template adopted in previous methods --- namely, viewing MDL as a game between the learner (who selects the best hypothesis) and the adversary (who chooses the most challenging mixture of distributions) --- 
our algorithm is built upon a novel and meticulously designed sampling scheme that deviates significantly from previous methods. 
Further, we extend our algorithm and theory to accommodate Rademacher classes, 
establishing a similar sample complexity bound; see Section~\ref{sec:rad} for details. 
Our algorithm and theory can also be extended to accommodate the multi-objective setting, which we postpone to Section~\ref{sec:multiloss}. 


Additionally, we solve two other open problems posed by \cite{awasthi2023sample}:
\begin{itemize}
	\item {\em Oracle-efficient solutions.}  An algorithm is said to be oracle-efficient if it only accesses $\mathcal{H}$ through an empirical risk minimization (ERM) oracle \citep{dudik2020oracle}. \citet[Problem 4]{awasthi2023sample} then asked what the sample complexity of MDL is when confined to oracle-efficient paradigms. 
		Encouragingly, our algorithm (i.e., Algorithm~\ref{alg:main1}) adheres to the oracle-efficient criterion, 
		thus uncovering that the sample complexity of MDL remains unchanged when restricted to oracle-efficient algorithms. 

	\item {\em Necessity of improper learning.}  Both our algorithm and the most sample-efficient methods preceding our work produce randomized hypotheses. 
		As discussed around \citet[Problem 3]{awasthi2023sample}, a natural question concerns characterization of the sample complexity when restricting the final output to deterministic hypotheses from $\mathcal{H}$.  Our result (see Theorem~\ref{thm:lb}) delivers a negative message: 
		under mild conditions, for any MDL algorithm, there exists a hard problem instance such that
it requires at least $\Omega(dk/\varepsilon^2)$ samples to find a deterministic hypothesis $h\in \mathcal{H}$ that attains $\varepsilon$-accuracy. 
This constitutes an enormous sample complexity gap between what can be achieved under improper learning and and proper learning. 
\end{itemize}

\paragraph{Concurrent work.} 
We shall mention that a concurrent work \citet{peng2023sample}, posted around the same time as our work, 
also studied the MDL problem and significantly improved upon the prior results. 
More specifically,  \citet{peng2023sample} established a sample complexity of $O\left(\frac{(d+k)\log(d/\delta)}{\varepsilon^{2}}\left(\frac{k}{\varepsilon}\right)^{o(1)} \right)$,  
 which is optimal up to some sub-polynomial factor in $k/\varepsilon$; in comparison, our sample complexity is optimal up to polylogarithmic factor. 
 Additionally, it is worth noting that the algorithm therein relies upon a certain recursive structure to eliminate the non-optimal hypothesis, 
 thus incurring exponential computational cost even when an ERM oracle is available.

\paragraph{Notation.}  Throughout this paper, we denote $[N]\coloneqq\{1,\ldots, N\}$ for any positive integer $N$.  
Let $\mathsf{conv}(\mathcal{A})$ represent the convex hull of a set $\mathcal{A}$.   
Denote by $\Delta(n)$ the $n$-dimensional simplex for any positive integer $n$, 
and $\Delta(\mathcal{A})$ the probability simplex over the set $\mathcal{A}$.  
For two vectors $v=[v_i]_{1\leq i\leq n}$ and $v' =[v'_i]_{1\leq i\leq n}$ with the same dimension, we overload the notation by using $\max\{v,v'\}=\big[\max\{v_i,v'_i\}\big]_{1\leq i\leq n}$ to denote the coordinate-wise maximum of $v$ and $v'$. Also we say $v\leq v'$ iff $v_i \leq v'_i$ for all $i\in [n]$. For any random variable $X$, we use $\mathsf{Var}[X]$ to denote its variance, i.e., $\mathsf{Var}[X]=\mathbb{E}\left[(X-\mathbb{E}[X])^2\right]$.
Let $\basis_1,\dots, \basis_k$ represent the standard basis vectors in $\mathbb{R}^k$. 
For any two distributions $P$ and $Q$ supported on $\mathcal{X}$, 
the Kullback-Leibler (KL) divergence from $Q$ to $P$ is defined and denoted by
\begin{align}
	\mathsf{KL}( P \,\|\, Q) \coloneqq \mathbb{E}_{Q}\bigg[ \frac{\mathrm{d}P}{\mathrm{d}Q} \log \frac{\mathrm{d}P}{\mathrm{d}Q} \bigg].
\end{align}

\section{Problem formulation}
\label{sec:preliminary}

This section formulates the multi-distribution learning problem. 
We assume throughout that each datapoint takes the form of $(x,y) \in \mathcal{X}\times \mathcal{Y}$, 
with $\mathcal{X}$ (resp.~$\mathcal{Y}$) 
the instance space (resp.~label space).

\paragraph{Learning from multiple distributions.} 
The problem setting encompasses several elements below. 
\begin{itemize}
	\item {\em Hypothesis class.} Suppose  we are interested in a hypothesis class $\mathcal{H}$, comprising a set of candidate functions from the instance space $\mathcal{X}$ to the label space $\mathcal{Y}$. 	Overloading the notation, we use $h_{\pi}$ to represent a {\em randomized  hypothesis} associated with a probability distribution $\pi \in \Delta(\mathcal{H})$, 
		meaning that a hypothesis $h$ from $\mathcal{H}$ is randomly selected according to distribution $\pi$. 
		When it comes to a VC class, 
		we assume that the hypotheses in $\mathcal{H}$ are binary-valued, and that the VC dimension \citep{vapnik1994measuring} of $\mathcal{H}$ is 
		\begin{equation}
			\VC(\mathcal{H}) = d. 
		\end{equation}
		%

	\item {\em Loss function.} 
Suppose we are given a loss function $\ell: \mathcal{H}\times \mathcal{X}\times \mathcal{Y}\to [-1,1]$,  
so that $\ell(h,(x,y))$ quantifies the risk of using hypothesis $h\in \mathcal{H}$ to make prediction on a datapoint $(x,y)\in \mathcal{X}\times \mathcal{Y}$ (i.e., predicting $y$ based on $x$). 
		One example is the 0-1 loss function $\ell(h,(x,y))=\mathds{1}\{h(x)\neq y\}$, which is often used to measure the misclassification error. 
%

	\item {\em (Multiple) data distributions.} 
		Suppose that there are $k$ data distributions of interest supported on $\mathcal{X}\times \mathcal{Y}$, denoted by $\mathcal{D} = \{\mathcal{D}_1,\mathcal{D}_2,\ldots,\mathcal{D}_k\}$. 
		We are permitted to draw independent samples from each of these data distributions. 
\end{itemize}

\noindent 

Given a target accuracy level $\varepsilon\in (0,1)$, our objective is to identify a (randomized) hypothesis, represented by $h_{\pi}$ with $\pi\in \Delta(\mathcal{H})$, such that 
\begin{align}
	\max_{1\leq i\leq k} \mathop{\mathbb{E}}\limits_{(x,y)\sim \mathcal{D}_i, h\sim \pi}\big[ \ell \big(h, (x,y) \big) \big]
	\leq  
	\min_{h\in \mathcal{H}}\max_{1\leq i\leq k}\mathop{\mathbb{E}}\limits_{(x,y)\sim \mathcal{D}_i}\big[\ell \big(h,(x,y) \big) \big] + \varepsilon.
	\label{eq:goal-overall}
\end{align}

\paragraph{Sampling and learning processes.} 
In order to achieve the aforementioned goal \eqref{eq:goal-overall},  
we need to draw samples from the available data distributions in $\mathcal{D}$, and the current paper focuses on sampling in an online fashion. 
%
%
More precisely, the learning process proceeds as follows: in each step $\tau$,
\begin{itemize}
	\item the learner selects $i_{\tau}\in [k]$ based on the previous samples;
	\item the learner draws an {\em independent} sample $(x_{\tau},y_{\tau})$ from the data distribution $\mathcal{D}_{i_{\tau}}$. 
\end{itemize}
\noindent 
The sample complexity of a learning algorithm thus refers to the total number of samples drawn from $\mathcal{D}$ throughout the learning process. 
A desirable learning algorithm would yield an $\varepsilon$-optimal (randomized) hypothesis (i.e., a hypothesis that achieves \eqref{eq:goal-overall}) 
using as few samples as possible.


%

\section{Algorithm}
\label{sec:algorithm}

In this section, we present our proposed algorithm for learning VC classes. 
Before proceeding, we find it convenient to introduce some notation concerning the loss under mixed distributions. 
Specifically, for any distribution $w=[w_i]_{1\leq i\leq k}\in \Delta(k)$ and any hypothesis $h\in \mathcal{H}$, 
the risk over the mixture $\sum_{i\in[k]} w_i\mathcal{D}_i$  of data distributions is denoted by: 
\begin{subequations}
\begin{equation}
	L(h,w) \coloneqq \sum_{i=1}^k w_i \mathop{\mathbb{E}}\limits_{(x,y)\sim \mathcal{D}_i}\big[\ell \big(h,(x,y)\big)\big]; 
	\label{defn:L}
\end{equation}
similarly, the risk of a randomized hypothesis $h_{\pi}$ (associated with $\pi\in \Delta(\mathcal{H})$) over  $\sum_{i\in[k]} w_i\mathcal{D}_i$ is given by
\begin{equation}
	L (h_{\pi}, w) \coloneqq \sum_{i=1}^k w_i \mathop{\mathbb{E}}\limits_{(x,y)\sim \mathcal{D}_i,h\sim \pi}\big[\ell \big(h_{\pi},(x,y)\big)\big]
	= \mathop{\mathbb{E}}\limits_{h\sim \pi}  \big[ L(h,w) \big].
\end{equation}
\end{subequations}

\begin{algorithm}[]
\small
	\DontPrintSemicolon
	\caption{Hedge for multi-distribution learning on VC classes ($\mathtt{MDL}\text{-}\mathtt{Hedge}\text{-}\mathtt{VC}$)\label{alg:main1}}
	\textbf{input:} $k$ data distributions $\{\mathcal{D}_1,\mathcal{D}_2,\ldots,\mathcal{D}_k\}$,  hypothesis class $\mathcal{H}$, target accuracy level $\varepsilon$, target success rate $1-\delta$. \\
\textbf{hyper-parameters:} stepsize $\eta=\frac{1}{100}\varepsilon$, number of rounds $T= \frac{20000\log\left(\frac{k}{\delta}\right)}{\varepsilon^2}$, 
	auxiliary accuracy level $\varepsilon_1=\frac{1}{100}\varepsilon$, 
	auxiliary sub-sample-size $\Tone \coloneqq \frac{4000\left(k\log(k/\varepsilon_1) +d\log(kd/\varepsilon_1)+\log(1/\delta)\right)}{\varepsilon_1^2}.$  \label{line:hyper-pars}\\
	\textbf{initialization:} 
	for all $i \in [k]$, set $W^1_i=1$, $\widehat{w}^0_i=0$ and $n_i^0=0$;  
	$\weighteddata=\emptyset$.  \\
	\For{$t=1,2,\ldots, T$}{
		set $w^t= [w^t_i]_{1\leq i\leq k}$  and $\widehat{w}^t= [\widehat{w}^t_i]_{1\leq i\leq k}$,  
		with $w^t_i \leftarrow \frac{W^t_i}{\sum_{j}W^t_j}$ and $\widehat{w}^t_i \leftarrow \widehat{w}^{t-1}_i$ for all $i\in [k]$. 
		\label{line:wt-update-Wt}\\
		%
		%
		{\color{blue}\tcc{recompute $\widehat{w}^t$ \& draw new samples for $\weighteddata$ only if $w^t$ changes sufficiently.}}
		\If{there exists $j\in [k]$ such that $w_j^t \geq 2\widehat{w}_j^{t-1}$ \label{line:a1}}{ 
			$\widehat{w}^t_i \leftarrow  \max\{w^{t}_i, \widehat{w}^{t-1}_i\}$ for all $i\in [k]$; \\
			\For{$i=1,\ldots,k$}{  $n_i^t \leftarrow \left\lceil \Tone\widehat{w}^t_i \right\rceil$; \label{line:nit} \\
			draw $n_i^t - n_i^{t-1}$ independent samples from $\mathcal{D}_i$, and add these samples to $\weighteddata$.\label{line:sample1}
			}
		}
		{\color{blue}\tcc{estimate the near-optimal hypothesis for weighted data distributions.}}
		compute $h^t\leftarrow \arg\min_{h\in \mathcal{H}} \widehat{L}^t(h,w^t)$, where \label{line:compute-ht}
		\begin{align}
			\widehat{L}^t(h,w^t) \coloneqq \sum_{i=1}^k \frac{w^t_i}{n_i^t}\cdot \sum_{j=1}^{n_i^t}\ell\big(h,(x_{i,j},y_{i,j})\big)\label{eq:e1}
		\end{align}
		with $(x_{i,j},y_{i,j})$ being the $j$-th datapoint from $\mathcal{D}_i$ in $\weighteddata$. \label{line:a2}
  \\
		{\color{blue}\tcc{estimate the loss vector and execute weighted updates.}}
		$\overline{w}^t_i \leftarrow \max_{1\leq \tau \leq t}w^{\tau}_i$ for all $i \in [k]$. \\
		\For{$i=1,\ldots, k$}{ 
		draw $ \lceil k\overline{w}^t_i \rceil$ independent samples --- denoted by $\big\{(x^t_{i,j},y^t_{i,j})\big\}_{j=1}^{\lceil k\overline{w}_i^t\rceil}$ --- from $\mathcal{D}_i$, and set \label{line:updateS} 
		$$
			\widehat{r}^t_i  =\frac{1}{\lceil k\overline{w}^t_i\rceil}\sum_{j=1}^{\lceil k\overline{w}^t_i \rceil} \ell\big(h^t,(x_{i,j}^t,y_{i,j}^t) \big); 
		$$
		
		update the weight as $W^{t+1}_i = W^t_i  \exp(\eta \widehat{r}_i^t)$. \label{line:updateS11} {\color{blue}\tcp{Hedge updates.}}
		}
	}
	\textbf{output:}  a randomized hypothesis $h^{\mathsf{final}}$ uniformly distributed over $ \{h^t\}_{t=1}^T$.
\end{algorithm}
\begin{remark}
	Note that in Algorithm~\ref{alg:main1}, the quantities $\{n_i^t\}_{1\leq t\leq T}$ are designed to be non-decreasing in $t$. 
\end{remark}
\begin{remark}
	In line~\ref{line:compute-ht} of Algorithm~\ref{alg:main1}, 
	we also allow $h^t$ to be $\varepsilon_1$-best response that obeys $\widehat{L}^t(h^t,w^t) \leq \min_{h\in \mathcal{H}} \widehat{L}^t(h,w^t) + \varepsilon_1$;  our theory remains unchanged if we make this modification. 
\end{remark}

Following the game dynamics proposed in previous works \citep{awasthi2023sample,haghtalab2022demand}, 
our algorithm alternates between computing the most favorable hypothesis (performed by the learner) and estimating the most challenging mixture of data distributions (performed by the adversary),  
with the aid of no-regret learning algorithms \citep{roughgarden2016twenty,shalev2012online}.  
More specifically, in each round $t$, our algorithm performs the following two steps:  
\begin{itemize}
	\item[(a)] Given a mixture of data distributions $\mathcal{D}^{(t)}=\sum_{i\in [k]} w_i^t \mathcal{D}_i$ (with $w^t=[w_i^t]_{i\in [k]}\in\Delta(k)$), 
		we construct a dataset to compute a hypothesis $h^{t}$ that nearly minimizes the loss under $\mathcal{D}^{(t)}$, namely, 
		\begin{equation}
			h^t \approx \arg\min_{h\in \mathcal{H}} L(h,w^t). 
		\end{equation}
		This is accomplished by calling an empirical risk minimization oracle. 

	\item[(b)] Given hypothesis $h^{t}$, we compute an updated weight vector  $w^{t+1}\in\Delta(k)$ --- and hence an updated mixed distribution $\mathcal{D}^{(t+1)}=\sum_{i\in [k]} w_i^{t+1} \mathcal{D}_i$. 
		The weight updates are carried out using the celebrated Hedge algorithm \citep{freund1997decision} designed for online adversarial learning,\footnote{Note that the Hedge algorithm is closely related to Exponentiated Gradient Descent,  Multiplicative Weights Update, Online Mirror Descent, etc \citep{arora2012multiplicative,shalev2012online,hazan2022introduction}.} 
		in an attempt to achieve low regret even when the loss vectors are adversarially chosen. More precisely, we run
		\begin{equation}
			w^{t+1}_i ~\propto~ w^t_i \exp\big( \eta \widehat{r}_i^t \big), \qquad i\in [k], 
		\end{equation}
		where the loss vector $\widehat{r}^t=[\widehat{r}_i^t]_{i\in [k]}$ contains the empirical loss of $h^{t}$ under each data distribution, i.e., 
		$$
			\widehat{r}_i^t\approx \mathop{\mathbb{E}}\limits_{(x,y)\sim \mathcal{D}_i}\big[\ell \big(h ^ t,(x,y)\big)\big], \qquad i \in [k],
		$$ 
		computed over another set of data samples. 

\end{itemize}
\noindent 
In words, the min-player and the max-player in our algorithm follow the best-response dynamics and no-regret dynamics, respectively (note that the Hedge algorithm is known to be a no-regret algorithm \citep{roughgarden2016twenty}). 
At the end of the algorithm, we output a randomized hypothesis $h^{\sf final}$ that is uniformly distributed over the hypothesis iterates $\{h^t\}_{1\leq t\leq T}$ over all $T$ rounds, 
following common practice in online adversarial learning.



While the above paradigm has been adopted in past works \citep{awasthi2023sample,haghtalab2022demand}, 
the resulting sample complexity depends heavily upon how data samples are collected and utilized throughout the learning process. 
For instance, \citet[Algorithm~1]{awasthi2023sample} --- which also alternates between best response and no-regret algorithm --- draws {\em fresh data} at each step of every round, 
in order to ensure reliable estimation of the loss function of interest through elementary concentration inequalities.  
This strategy, however, becomes wasteful over time,  constituting one of the main sources of its sample sub-optimality. 

In order to make the best use of data, we propose the following key strategies. 
\begin{itemize}
	\item {\em Sample reuse in Step (a).} In stark contrast to \citet[Algorithm~1]{awasthi2023sample} that draws new samples for estimating each $h^t$, 
		we propose to reuse all samples collected in Step (a) up to the $t$-th round to assist in computing $h^t$.  
		As will be made precise in lines~\ref{line:a1}-\ref{line:a2} of Algorithm~\ref{alg:main1}, 
		we shall maintain a {\em growing dataset} $\weighteddata$ for conducting Step (a) throughout, 
		ensuring that there are $n_i^t$ samples drawn from distribution $\mathcal{D}_i$ in the $t$-th round. 
		These datapoints are employed to construct an empirical loss estimator $\widehat{L}^t(h,w^t)$ for each $h\in \mathcal{H}$  in each round $t$, 
		with the aim of achieving uniform convergence $|\widehat{L}^t(h,w^t)-L(h,w^t)|\leq O(\varepsilon)$ over all $h\in \mathcal{H}$. 
		More detailed explanations are provided in Section~\ref{sec:tec1}.




	\item {\em Weighted sampling for Step (b).}  
As shown in line~\ref{line:updateS} of Algorithm~\ref{alg:main1}, 
		in each round $t$, we sample each $\mathcal{D}_i$ a couple of times to compute the empirical estimator for $\mathbb{E}_{(x,y)\in \mathcal{D}_i}\big[ \ell(h^t, (x,y)) \big]$, where the number of samples depends upon the running weights $\{w_i^\tau\}$. 
		More precisely, we collect $\left\lceil k\overline{w}_i^t \right \rceil$ fresh samples from each $\mathcal{D}_i$, where
		$\overline{w}_i^t:=\max_{1\leq \tau \leq t}w_{i}^{\tau}$ is the maximum weight assigned to $\mathcal{D}_i$ up to now. 
		Informally speaking, this strategy is chosen carefully to ensure reduced variance of the estimators given a sample size budget. 
		The interested reader is referred to Section~\ref{sec:tec2} and Lemma~\ref{lemma:cut} for more detailed explanations. 

\end{itemize}


\noindent 
The whole procedure can be found in Algorithm~\ref{alg:main1}.

\section{A glimpse of key technical novelties}\label{sec:tec}

In this section, we highlight two technical novelties that empower our analysis: (i) uniform convergence of the weighted sampling estimator that allows for sample reuse (see Section~\ref{sec:tec1}), and (ii) tight control of certain $\|\cdot\|_{1,\infty}$ norm of the iterates  $\{w^t\}_{1\leq t\leq T}$ that dictates the sample efficiency (see Section~\ref{sec:tec2}). 

\subsection{Towards sample reuse: uniform concentration and a key quantity}\label{sec:tec1}

Recall that in Algorithm~\ref{alg:main1}, 
we invoke the empirical risk estimator 
$\widehat{L}^t(h,w^t)$ as an estimate of the true risk of hypothesis $h$ over the weighted distribution specified by $w^t$ (cf.~\eqref{eq:e1}). 
In order to facilitate sample reuse when constructing such risk estimators across all iterations, 
it is desirable to establish uniform concentration results to control the errors of such risk estimators throughout the execution of the algorithm. 
Towards this end, our analysis strategy proceeds as follows. 

\paragraph{Step 1: concentration for any fixed set of parameters.}  
Consider any given set of integers $n = \{n_i\}_{i=1}^k$ 
and any given vector $w\in \Delta(k)$. 
Suppose, for each $i\in [k]$,  we have $n_i$ i.i.d.~samples drawn from $\mathcal{D}_i$ --- denoted by $\{(x_{i,j},y_{i,j})\}_{j=1}^{n_i}$ --- 
and let us look at the empirical risk estimator,   
\begin{align}
	\widehat{L}_{n}(h,w) \coloneqq \sum\nolimits_{i=1}^k w_i \cdot \frac{1}{n_i}\sum\nolimits_{j=1}^{n_i} \ell\big(h, (x_{i,j},y_{i,j}) \big),
	\label{eq:xe}
\end{align}
which is a sum of independent random variables. Evidently, for a given hypothesis $h$, 
the variance of $\widehat{L}_{n}(h,w)$ is upper bounded by
\[
\mathsf{Var}\big(\widehat{L}_{n}(h,w)\big)\leq\sum\nolimits _{i=1}^{k}\frac{w_{i}^{2}}{n_{i}}\leq\left(\sum\nolimits _{i=1}^{k}w_{i}\right)\frac{1}{\min_{i}n_{i}/w_{i}}=\frac{1}{\min_{i}n_{i}/w_{i}} .
\]
Assuming that the central limit theorem is applicable, one can derive
\[
\mathbb{P}\left\{ \big|\widehat{L}_{n}(h,w)-L(h,w)\big|\geq\varepsilon\right\} \lesssim \exp\bigg(-\frac{\varepsilon^{2}}{2\mathsf{Var}\big(\widehat{L}_{n}(h,w)\big)}\bigg)\lesssim \exp\bigg(-\frac{\varepsilon^{2}}{2}\min_{i}\frac{n_{i}}{w_{i}}\bigg).
\]
Armed with this result, we can extend it to accommodate all $h\in \mathcal{H}$ through the union bound. 
For a VC class with $\VC(\mathcal{H})=d$,  
the celebrated Sauer–Shelah lemma \citep[Proposition~4.18]{wainwright2019high} tells us that 
the set of hypotheses can be effectively compressed into a subset with cardinality no larger than $\exp\big(\widetilde{O}(d)\big)$.  
Taking the union bound then yields
\[
\mathbb{P}\left(\max_{h\in\mathcal{H}}\big|\widehat{L}_{n}(h,w)-L(h,w)\big|\geq\varepsilon\right)\lesssim 
\exp\left( \widetilde{O}(d) - \frac{\varepsilon^{2}}{2} \min_i \frac{n_i}{w_i} \right). \nonumber
\]

\paragraph{Step 2: uniform concentration.} 
Next, we would like to extend the above result to establish uniform concentration over all $n$ and $w$ of interest. 
Towards this, we shall invoke the union bound as well as the standard epsilon-net arguments. 
Let the set $\mathcal{X}\subseteq \Delta(k)$ be a proper discretization of $\Delta(k)$,  with cardinality 
$\exp\big(\widetilde{O}(k) \big)$. 
In addition, given the trivial upper bound $n_i\leq T_1$ for all $i\in [k]$, 
we know that there exist at most $T_1^k=\exp\big(\widetilde{O}(k)\big)$ possible combinations of $\{n_i\}_{i\in [k]}$. 
We can then apply the union bound to show that 
\begin{align}
\mathbb{P}\left\{ \exists w\in\mathcal{X}\text{ and feasible }n\text{ s.t.}\big|\widehat{L}_{n}(h,w)-L(h,w)\big|\geq\varepsilon\right\}  & \lesssim\exp\bigg(\widetilde{O}(k)+\widetilde{O}(d)-\frac{\varepsilon^{2}}{2}\min_{i}\frac{n_{i}}{w_{i}}\bigg). 
\end{align}
When the discretized set $\mathcal{X}$ is chosen to have sufficient resolution, 
we can straightforwardly employ the standard covering argument to extend the above inequality to accommodate all $w\in \Delta(k)$ of interest.

 \paragraph{Key takeaways.} 
The above arguments reveal 
the following high-probability property: 
\emph{whenever} we collect $n =\{n_i\}_{i=1}^k$ samples in the learning process, we could obtain $\varepsilon$-approximation $\widehat{L}_n(h,w)$ (see \eqref{eq:xe}) of $L(h,w)$  for all $h\in \mathcal{H}$ and all $w\in \Delta(k)$ with high probability,  
provided that 
\begin{align}
	\min_i \frac{n_i}{w_i} \gtrsim  \widetilde{O}\bigg( \frac{k+d}{\varepsilon^2} \bigg).  
\end{align}
This makes apparent the pivotal role of the quantity $\min_i n_i / w_i$. 
In our algorithm, 
we design the update rule (cf.~line~\ref{line:nit} of Algorithm~\ref{alg:main1}) to guaranteed that
\begin{align}
	\min_i \frac{n_i^t}{w_i^t} \gtrsim T_1 \geq  \widetilde{\Omega}\bigg( \frac{k+d}{\varepsilon^2} \bigg)   
	\label{eq:Tone-LB-explain}
\end{align}
for all $1\leq t\leq T$. 
In fact, this explains our choice of $T_1$ in Algorithm~\ref{alg:main1}. 
Crucially, the aforementioned uniform concentration result allows us to reuse samples throughout the learning process instead of drawing fresh samples to estimate $L(h,w^t)$ in each round $t$ (note that the latter approach clearly falls short of data efficiency). 
To conclude, 
to guarantee $\varepsilon$-uniform convergence for all rounds, 
it suffices to choose $T_1 = \widetilde{\Omega}\big(\frac{k+d}{\varepsilon^2}\big)$.

Finally, recall that $n_i^t \asymp T_1 \overline{w}_i^t$ for each $i\in [k]$ and $t\leq T$, 
with $\overline{w}^t_i \coloneqq  \max_{1\leq \tau \leq t} w_i^{\tau}$; 
taking $n_i^t \asymp T_1 \overline{w}_i^t$ (as opposed to $n_i^t \asymp T_1 w_i^t$) ensures that 
the sample size $n_i^t$ is monotonically non-decreasing in $t$.  
With \eqref{eq:Tone-LB-explain} in mind, 
the total number of samples collected within $T$ rounds in Algorithm~\ref{alg:main1} obeys
\begin{align}
\frac{1}{T_1}\sum\nolimits_{i=1}^kn_i^T \asymp  \sum\nolimits_{i=1}^k\overline{w}_i^T 
\eqqcolon \|\overline{w}^{T}\|_1 .
\label{eq:l1-linf-norm-w} 
\end{align}
%
This threshold $\|\overline{w}^{T}\|_1$ --- or equivalently, the $\|\cdot\|_{1,\infty}$ norm of $\{w^t\}_{1\leq t\leq T}$ ---
is a  critical quantity that  we wish to control. 
In particular, in the desirable scenario where $\|\overline{w}^{T}\|_1\leq \widetilde{O}(1)$, 
the total sample size obeys $\sum\nolimits_{i=1}^kn_i^T\asymp T_1 \|\overline{w}^{T}\|_1  = \widetilde{O}\big(\frac{k+d}{\varepsilon^2}\big)$.

%
%

\subsection{Bounding the key quantity $\|\overline{w}^{T}\|_1$ by tracking the Hedge trajectory}\label{sec:tec2}

Perhaps the most innovative (and most challenging) part of our  analysis lies in controlling the $\|\cdot \|_{1,\infty}$ norm of $\{w_i^{t}\}_{1\leq t\leq T}$, 
whose critical importance has been pointed out in Section~\ref{sec:tec1}.

%
%

Towards this end, the key lies in carefully tracking the dynamics of the Hedge algorithm. 
To elucidate the high-level idea, 
let us look at a simpler minimax optimization problem w.r.t.~the set of loss vectors in the convex hull of a set $\mathcal{Y}\subseteq \mathbb{R}^k$: 
\begin{align}
	\min_{y\in \mathsf{conv}(\mathcal{Y})}\max_{w\in \Delta(k)}w^{\top}y \qquad (\text{or equivalently, } \max_{w\in \Delta(k)}\min_{y\in \mathsf{conv}(\mathcal{Y})}w^{\top}y),
	\label{eq:bilinear-games-formulation}
\end{align}
where the equivalence arises from von Neumann's minimax theorem \citep{v1928theorie}. 
Consider the following algorithm  (cf.~Algorithm~\ref{alg:2}) tailored to this minimax problem,  assuming perfect knowledge about the loss vectors.\footnote{Note that in Algorithm~\ref{alg:main1}, we can only estimate the loss vector using the collected samples. Additional efforts are needed to reduce the variability (see line~\ref{line:updateS} in Algorithm~\ref{alg:main1}).}

\begin{algorithm}[]
	\DontPrintSemicolon
	\caption{The Hedge algorithm for bilinear games. \label{alg:2}}
	\textbf{Input:}  $\mathcal{Y}\subseteq [-1,1]^k$, target accuracy level $\varepsilon \in (0,1)$.\\
	\textbf{Initialization: } $T=\frac{100\log(k)}{\varepsilon^2}$, $\eta = \frac{1}{10}\varepsilon$, and $W^1_i = 1$ for all $1\leq i\leq k$.\\
	\For{$t=1,2,\ldots,T$}{
		compute $w_i^t \leftarrow \frac{W_i^t}{\sum_j W_j^t}$ for every $1\leq i \leq k$.\\
		compute $y^t \leftarrow\arg\min_{y\in \mathcal{Y}}~\langle w^t,y \rangle$.\\
		update $W_i^{t+1}\leftarrow W_i^t \exp(\eta y_i^t)$ for every $1\leq i\leq k$.\\
	}
\end{algorithm}
\noindent 
This algorithm is often referred to as the  Hedge algorithm, 
which is known to yield an $\varepsilon$-minimax solution within  $O\big(\frac{\log(k)}{\varepsilon^2}\big)$ iterations. 
A challenging question relevant to our analysis is:  
\begin{itemize}
\item[] {\em \textbf{Question:} can we bound $\|\overline{w}^T\|_1 \coloneqq \sum_{i=1}^k \max_{1\leq t\leq T}w_i^t$ in Algorithm~\ref{alg:2} by poly-logarithmic terms?}
\end{itemize}
As it turns out, 
we can answer this question affirmatively (see Lemma~\ref{lemma:main}), 
and the key ideas will be elucidated in the remainder of this section.

\subsubsection{First attempt: bounding the number of distributions with $\max_{1\leq t\leq T}w^t_i = \Omega( 1)$}
\label{sec:tec21}




Instead of bounding $\|\overline{w}^T\|_1$ directly, our first  attempt is to look at those $i\in[k]$ with large $\max_{1\leq t\leq T}w^t_i$ (more specifically, $\max_{1\leq t\leq T}w^t_i \geq 1/4$) and show that: 
\begin{itemize}
	\item there exist at most $\widetilde{O}(1)$ coordinates $i\in [k]$ obeying $\max_{1\leq t\leq T}w^t_i \geq 1/4$
 (or some other universal constant).  
\end{itemize}
 In other words, we would like to demonstrate that the cardinality of the following set is small: 
\begin{equation}
	\mathcal{W}_{\mathsf{large}} \coloneqq \big\{ i\in [k] \mid \max\nolimits_{1\leq t\leq T}w_i^{t}\geq 1/4 \big\}.
\end{equation}

To do so, note that some standard ``continuity''-type argument tells us that: 
for a sufficiently small stepsize $\eta$, one can find, for each $i\in \mathcal{W}_{\mathsf{large}}$, a time interval $[s_i,e_i] \subseteq [0,T]$ obeying 
\begin{equation}
1/16\leq w_i^{s_i}\leq 1/8, \qquad  w_i^{e_i}\geq 1/4, 
  \qquad \text{and} \qquad w_i^t\geq 1/8~~\forall t\in (s_i, e_i].
	\label{eq:large-wi-si-ei-defn}
  \end{equation}
In words, $w_i^t$ at least doubles from $t=s_i$ to $t=e_i$. 
We claim for the moment that 
\begin{equation}
e_i-s_i \geq {\Omega}(1/\eta^2) = {\Omega}(1/\varepsilon^2) 
\qquad \forall i\in \mathcal{W}_{\mathsf{large}}.
\label{eq:length-ei-si}
\end{equation}
Additionally, observe that for any $t$, there exist at most $8$ coordinates $i\in \mathcal{W}_{\mathsf{large}}$ such that $s_i \leq t \leq e_i$ (since $w_i^t \geq 1/8$ for every $t\in [s_i,e_i]$). This reveals that 
\begin{equation}
	8T\geq \sum_{i\in \mathcal{W}_{\mathsf{large}}} (e_i-s_i)\geq |\mathcal{W}_{\mathsf{large}}|\cdot \Omega \left(1/\varepsilon^2\right),
	\label{eq:8T-Wlarge}
\end{equation}
which combined with our choice of $T=\widetilde{O}( 1/\varepsilon^2 )$ (cf.~line~\ref{line:hyper-pars} of Algorithm~\ref{alg:main1}) yields
$$
	|\mathcal{W}_{\mathsf{large}}|\leq  O(T\varepsilon^2) = \widetilde{O}(1).
$$
In words, the number of distributions with large $\max_{1\leq t\leq T}w^t_i$ (i.e., $\max_{1\leq t\leq T}w^t_i \geq 1/4$) is fairly small.

\begin{proof}[Proof sketch for \eqref{eq:length-ei-si}]
Let us briefly discuss the high-level proof ideas. 
Following standard analysis for the Hedge algorithm (e.g., \citet{shalev2012online,lattimore2020bandit}), we can often obtain (under certain mild conditions) 
$$
	\mathsf{KL}\big(w^{s_i} \,\|\, w^{e_i}\big) \leq  O\big(\eta^2 (e_i-s_i)\big).
$$
Combine this relation with basic properties about the KL divergence (see, e.g., Lemmas~\ref{lemma:klbound} and \ref{lemma:klcmp} in Appendix~\ref{sec:auxiliary-lemmas}) and the choice $\eta = \varepsilon/10$ to obtain
\begin{align*}
e_{i}-s_{i} & \geq\Omega\big(\eta^{-2}\mathsf{KL}\big(w^{s_{i}}\,\|\,w^{e_{i}}\big)\big)\geq\Omega\Big(\eta^{-2}\mathsf{KL}\big(\mathsf{Ber}(w_i^{s_{i}})\,\|\,\mathsf{Ber}(w_i^{e_{i}})\big)\Big)\\
 & \geq\Omega\bigg(\eta^{-2}\frac{\frac{1}{16}\cdot2^{2}}{4}\bigg)=\Omega\left(1/\eta^{2}\right)=\Omega\left(1/\varepsilon^{2}\right).	
\end{align*}
\end{proof}

\subsubsection{More general cases: issues and solutions}\label{sec:tec22}

Naturally, one would hope to generalize the arguments in Section~\ref{sec:tec21} to cope with more general cases. 
More specifically, let us look at the following set 
\begin{equation}
	\mathcal{W}(p) \coloneqq \left\{ i\in [k] \mid \max\nolimits_{1\leq t\leq T}w_i^t \in [2p,4p] \right\}, 
	\label{eq:defn-Wp-intuition}
\end{equation}
defined for each $p\in [0,1]$. If one could show that
\begin{equation}
	|\mathcal{W}(p)| =  \widetilde{O}(1/p) \qquad \text{for each } p ,
	\label{eq:desired-Wp-bound}
\end{equation}
then a standard doubling argument would immediately lead to 
\begin{align*}
\|\overline{w}^{T}\|_{1} & =\sum_{i\in[k]}\max_{1\leq t\leq T}w_{i}^{t}\approx\sum_{j=1}^{\log_{2}k}\sum_{i\in\mathcal{W}(2^{-j})}\max_{1\leq t\leq T}w_{i}^{t}\leq\sum_{j=1}^{\log_{2}k}4\cdot2^{-j}\cdot\big|\mathcal{W}(2^{-j})\big|\\
 & \leq\sum_{j=1}^{\log_{2}k}4\cdot2^{-j}\cdot\widetilde{O}\bigg(\frac{1}{2^{-j}}\bigg)=\widetilde{O}(1).
\end{align*}

\paragraph{A technical issue.} 
Nevertheless, simply repeating the arguments in Section~\ref{sec:tec21} fails to deliver the desirable bound \eqref{eq:desired-Wp-bound} on $|\mathcal{W}(p)|$ when $p$ is small. 
Briefly speaking, for each $i\in \mathcal{W}(p)$, let $[s_i,e_i]$ represent a time interval (akin to \eqref{eq:large-wi-si-ei-defn}) such that 
\begin{equation}
p/2\leq w_i^{s_i}\leq p, \qquad w_i^{e_i}\geq 2p 
\qquad \text{and}\qquad w_i^t\geq p 
\quad \text{for any }s_i<t\leq e_i.
	\label{eq:defn-s-e-issue}
\end{equation}
Repeating the heuristic arguments in Section~\ref{sec:tec21} leads to
\begin{align}
e_{i}-s_{i}\geq\Omega\big(\eta^{-2}\mathsf{KL}\big(w^{s_{i}}\,\|\,w^{e_{i}}\big)\big)\geq\Omega\Big(\eta^{-2}\mathsf{KL}\big(\mathsf{Ber}(w_i^{s_{i}})\,\|\,\mathsf{Ber}(w_i^{e_{i}})\big)\Big)\geq\Omega(\eta^{-2}p)=\Omega\left(p/\varepsilon^{2}\right). 
	\label{eq:repeat-e-s-diff-issue}
\end{align}
Given that each $t$ is contained within at most $1/(p/2)$ intervals associated with $\mathcal{W}(p)$, repeat the arguments for \eqref{eq:8T-Wlarge} to derive
\[
\frac{1}{p/2}\cdot T\geq\sum_{i\in\mathcal{W}(p)}(e_{i}-s_{i})\geq|\mathcal{W}(p)|\cdot\Omega\left(p/\varepsilon^{2}\right)\qquad\Longrightarrow\qquad|\mathcal{W}(p)|\leq\widetilde{O}\bigg(\frac{1}{p^{2}}\bigg).
\]
This bound, however, is clearly loose compared to the desirable one in \eqref{eq:desired-Wp-bound}.

\paragraph{Our solution.} 
To address this issue, we make two key observations below that inspire our approach: 
\begin{itemize}
	\item {\em Shared intervals.} For the interval $[s_i,e_i]$ associated with $i$ as defined above, if there exist other indices sharing the same interval $[s_i,e_i]$ (in the sense that relations analogous to \eqref{eq:defn-s-e-issue} are satisfied for other indices), then it is plausible to improve the bound. For instance, if a set $\mathcal{M}_i\subseteq [k]$ of indices  share the same interval $[s_i,e_i]$, then one can follow the heuristic argument in \eqref{eq:repeat-e-s-diff-issue} to obtain
\begin{equation}
	e_{i}-s_{i}\geq\Omega\big(\eta^{-2}\mathsf{KL}\big(w^{s_{i}}\,\|\,w^{e_{i}}\big)\big)\geq\Omega\left(\eta^{-2}\mathsf{KL}\bigg(\mathsf{Ber}\Big(\sum_{j\in\mathcal{M}_{i}}w_{j}^{s_{i}}\Big)\,\|\,\mathsf{Ber}\Big(\sum_{j\in\mathcal{M}_{i}}w_{j}^{e_{i}}\Big)\bigg)\right)\geq\Omega\left(p|\mathcal{M}_{i}|/\varepsilon^{2}\right),
	\label{eq:ei-si-mi-LB}
\end{equation}
which clearly strengthens the original bound \eqref{eq:repeat-e-s-diff-issue} if $|\mathcal{M}_{i}|$ is large. 

%

	\item {\em Disjoint intervals.} Consider the special case where $\mathcal{W}(p)$ can be divided into subsets $\{\mathcal{V}_n\}_{n=1}^N$ obeying  
	\begin{itemize}
		\item[(i)] for each $n\in[N]$, all indices in $\mathcal{V}_n$ share the same interval $[s_n,e_n]$ (defined analogously as \eqref{eq:defn-s-e-issue}); 
		\item[(ii)] the intervals  $\{[s_n,e_n]\}_{n=1}^N$ are {\em disjoint}. 
	\end{itemize}
	Then one can derive the desired bound on $|\mathcal{W}(p)|$. More precisely, it follows from \eqref{eq:ei-si-mi-LB} that 
	\begin{equation}
		e_{n}-s_{n}\geq\Omega\left(p|\mathcal{V}_{n}|/\varepsilon^{2}\right),
	\end{equation}
	which together with the disjointness property yields
	\begin{equation}
		|\mathcal{W}(p)|=\sum_{n=1}^{N}|\mathcal{V}_{n}|\leq\sum_{n=1}^{N}O\bigg(\frac{(e_{n}-s_{n})\varepsilon^{2}}{p}\bigg)\leq O\bigg(\frac{T\varepsilon^{2}}{p}\bigg)=\widetilde{O}\bigg(\frac{1}{p}\bigg).
	\end{equation}	

\end{itemize}
\noindent 
In light of the above discussion, it is helpful to (a) merge those indices that share similar intervals, and (b) identify disjoint intervals whose associated indices can cover a good fraction of $\mathcal{W}(p)$.

%
Motivated by the aforementioned observation about ``shared intervals,'' we introduce the notion of ``\emph{segments}'' to facilitate analysis. 
%
\begin{definition}[Segment] \label{def:seg}
For any $p,x>0$ and $i\in [k]$, we say that $(t_1,t_2)$ is a $(p,q,x)$-\emph{segment}  if there exists a subset $\mathcal{I}\subseteq [k]$  such that 
\begin{itemize}
	\item[$\mathrm{(i)}$] $\sum_{ i\in \mathcal{I}}w^{t_1}_{i}\in \left[p/2 ,p\right]$, 
	\item[$\mathrm{(ii)}$] $\sum_{i\in \mathcal{I}}w^{t_2}_i \geq p \exp(x)$,  
	\item[$\mathrm{(iii)}$] $\sum_{i\in \mathcal{I}}w_i^{t}\geq  q$ for any $t_1\leq t\leq t_2$. 
\end{itemize}
We shall refer to $t_1$ as the starting point and $t_2$ as the end point, and call $\mathcal{I}$ the associated index set. Moreover,  two segments $(s_1,e_1)$ and $(s_2,e_2)$ are said to be disjoint if $s_1<e_1 \leq s_2<e_2$ or $s_2<e_2\leq s_1<e_1$. 
\end{definition}
\noindent 
This definition allows one to pool indices with similar intervals together.


In general, however, it is common for two segments to be overlapping (see  Figure~\ref{fig2} in Appendix~\ref{sec:additional-figs}), 
which precludes us from directly invoking our aforementioned observation about ``disjoint intervals.''  
 To address this issue, 
 our strategy is to extract out shared sub-segments\footnote{A sub-segment refers to a sub-interval of a segment, as illustrated in Figure~\ref{fig8}.} of (a subset of) these segments in a meticulous manner.  Encouragingly, it is possible to find such sub-segments that taken collectively cover a good fraction (i.e., $\frac{1}{\mathrm{poly}\log (k,T)}$) of all segments, 
 meaning that we do not have to discard too many segments. 
 The construction is built upon careful analysis of these segments, and will be elucidated in Appendix~\ref{sec:control-trajectory}.

\section{Analysis for VC classes (proof of Theorem~\ref{thm:main})}\label{sec:pf}

\subsection{Key lemmas underlying the proof}

The main steps for establishing Theorem~\ref{thm:main} lie in proving  three key lemmas, as stated below. 

The first lemma is concerned with the hypothesis $h^t = \arg\min_{h\in \mathcal{H}}\widehat{L}^t(h,w^t)$ 
(cf.~line~\ref{line:compute-ht} of Algorithm~\ref{alg:main1});  
in words, $h^t$ is the minimizer of the empirical loss function $\widehat{L}^t(\cdot,w^t)$, computed using samples obtained up to the $t$-th round. The following lemma tells us that: even though $h^t$ is an empirical minimizer, it almost optimizes the weighted population loss  
$L(\cdot,w^t)$.  
In other words, this lemma justifies that the adaptive sampling scheme proposed in Algorithm~\ref{alg:main1} 
ensures faithfulness of the empirical loss and its minimizer; here, we recall that $\varepsilon_1=\varepsilon/100$ (cf.~line~\ref{line:hyper-pars} of Algorithm~\ref{alg:main1}). 
\begin{lemma}\label{lemma:opth}
With probability at least $1-\delta/4$,  
\begin{align}
L(h^t,w^t)\leq \min_{h\in \mathcal{H}}L(h,w^t) +\varepsilon_1
	\label{eq:quality-L-ht-wt}
\end{align}
holds for all $1\leq t\leq T$, where $h^t$ (resp.~$w^t$) is the hypothesis (resp.~weight vector) computed in round $t$ of Algorithm~\ref{alg:main1}.
\end{lemma}
\begin{proof} See Appendix~\ref{sec:proof-lemma:opth}. \end{proof}

Next, assuming that \eqref{eq:quality-L-ht-wt} holds, 
we can resort to standard analysis for the Hedge algorithm to demonstrate the quality of the final output  $h^{\mathsf{final}}$. 

\begin{lemma}\label{lemma:opt} Suppose that lines~\ref{line:a1}-\ref{line:a2} in Algorithm~\ref{alg:main1} are replaced with some oracle that returns a hypothesis $h^{t}$ satisfying $L(h^t,w^t)\leq \min_{h\in \mathcal{H}}L(h,w^{t})+\varepsilon_1$ in the $t$-th round for each $1\leq t\leq T$. 
 With probability exceeding $1-\delta/4$, the hypothesis $h^{\mathsf{final}}$ output by Algorithm~\ref{alg:main1} is $\varepsilon$-optimal in the sense that
\begin{align}
\max_{1\leq i\leq k}L(h^{\mathsf{final}},\basis_{i})
	\leq  \min_{\pi\in \Delta(\mathcal{H})}\max_{1\leq i\leq k} L(h_{\pi},\basis_i)+\varepsilon \leq \min_{h\in \mathcal{H}}\max_{1\leq i\leq k} L(h,\basis_i)+\varepsilon .
\end{align}
Here, we recall that $\basis_i$ indicates the $i$-th standard basis vector. 
\end{lemma}
\begin{proof} See Appendix~\ref{sec:opt}. \end{proof}

Taking Lemma~\ref{lemma:opth} and Lemma~\ref{lemma:opt} together, 
one can readily see that Algorithm~\ref{alg:main1} 
returns an $\varepsilon$-optimal randomized hypothesis $h^{\mathsf{final}}$ 
with probability at least $1-\delta/2$.  
The next step then lies in bounding the total number of samples that has been collected in Algorithm~\ref{alg:main1}. 
Towards this end, 
recall that  $\overline{w}^T_i  =\max_{1\leq t\leq T}w_i^t$ for each $i\in [k]$. 
Recognizing that $\widehat{w}^t_i \leq \overline{w}^t_i$ for each $t\in [T]$ and $i\in [k]$, 
we can  bound the total sample size by
\begin{align}
	(\text{sample size})\qquad 
\Tone\sum_{i=1}^{k}\widehat{w}_{i}^{T}+k+T\bigg(k\sum_{i=1}^{k}\overline{w}_{i}^{T}+k\bigg)
	&\leq\left(\Tone\|\overline{w}^{T}\|_{1}+kT\|\overline{w}^{T}\|_{1}\right)+k(T+1) \notag\\
	&\lesssim \frac{ d\log\left(\frac{kd}{\varepsilon}\right)+ k\log\left(\frac{k}{\delta\varepsilon}\right)}{\varepsilon^2} \cdot \|\overline{w}^{T}\|_1 ,
	\label{eq:sample-size-wT}
\end{align}
where the last relation follows from our choices $T\asymp \frac{\log(k/\delta)}{\varepsilon^2} $ and $T_1 \asymp \frac{k\log(k/\varepsilon)+d\log(kd/\varepsilon)+\log(1/\delta)}{\varepsilon^2} $
(cf.~line~\ref{line:hyper-pars} of Algorithm~\ref{alg:main1}) and the basic property that $\|\overline{w}^{T}\|_1\geq \sum_i w_i^1 = 1$. 
Consequently, everything then comes down to bounding  $\|\overline{w}^{T}\|_1$, 
for which we resort to the following lemma. 
\begin{lemma}\label{lemma:main} Assume that lines~\ref{line:a1}-\ref{line:a2} in Algorithm~\ref{alg:main1} are replaced with some oracle which returns a hypothesis $h^{t}$ satisfies that $L(h^t,w^t)\leq \min_{h\in \mathcal{H}}L(h,w^{t})+\varepsilon_1$ in the $t$-th round for each $1\leq t\leq T$. 
With probability at least $1-\delta/4$, the quantity 
$\|\overline{w}^{T}\|_1$ is bounded above by
$$
	\|\overline{w}^{T}\|_1\leq 
 O\left(  \log^8\left( \frac{k}{\delta\varepsilon}\right) \right). 
 $$
\end{lemma}

\noindent 
It is noteworthy that the proof of  Lemma~\ref{lemma:main} is the most technically challenging part of the analysis, and  
we shall present this proof in Section~\ref{sec:pflmain}.

Combining Lemma~\ref{lemma:main} with \eqref{eq:sample-size-wT} immediately reveals that, with probability at least $1-\delta$, the sample complexity of Algorithm~\ref{alg:main1} is bounded by 
$$
O\left(\frac{ d\log\left(\frac{kd}{\varepsilon}\right)+ k\log\left(\frac{k}{\delta\varepsilon}\right)}{\varepsilon^2} \cdot 
 \log^8\left( \frac{k}{\delta\varepsilon}\right) \right),
$$
as claimed in Theorem~\ref{thm:main}.

\subsection{Proof of Lemma~\ref{lemma:main} (controlling the Hedge trajectory)}\label{sec:pflmain}

As alluded to previously, the proof of Lemma~\ref{lemma:main} forms the most technically challenging part of our analysis. 
%
%
Let us begin by introducing several convenient notation. 
Set $\delta' = \delta/(32T^4k^2)$, and let 
\begin{align}
\overline{j} = \left\lfloor \log_2\left(\frac{k\log^2(2)}{50(\log_2(1/\eta)+1)^2\log_2^2(k)}\right) \right\rfloor -2.
\label{eq:defn-jbar-Lemma3}
\end{align}
Let us define 
\begin{subequations}
\begin{align}
	\mathcal{W}_j &\coloneqq \big\{i\in [k] \mid \max_{1\leq t\leq T}w_i^t \in ( 2^{-j},2^{-(j-1)}]  \big\}, \qquad 1\leq j \leq \overline{j}  \label{eq:defn-Wj-main-123}\\
	\overline{\mathcal{W}} &\coloneqq [k]/\cup_{j}\mathcal{W}_j.
\end{align}
\end{subequations}
%
%
In other words, we divide the $k$ distributions into a logarithmic number of groups $\{\mathcal{W}_j\}$, 
where each $\mathcal{W}_j$ consists of those distributions whose corresponding  $\max_{t}w_i^t$ are on the same order.   
The main step in establishing Lemma~\ref{lemma:main} lies in bounding the size of  each $\mathcal{W}_j$, 
as summarized below.
\begin{lemma}\label{lemma:bdcount} Suppose that the assumptions of Lemma~\ref{lemma:main} hold. 
  Then with probability exceeding $1-8T^4k\delta'$, 
\begin{equation}
	|\mathcal{W}_j|\leq 8\cdot 10^7\cdot \left(\big(\log_2(1/\eta)+1\big)^2\log_2^2(k)\left(\log(k)+\log(1/\delta')\right)^3\big(\log_2(T)+1\big)\right) \cdot 2^j
	\label{eq:Wj-UB-lemma}
\end{equation}
holds all $1\leq j \leq \overline{j}$, 
with $\overline{j}$ defined in \eqref{eq:defn-jbar-Lemma3}. 
\end{lemma}
In words, Lemma~\ref{lemma:bdcount} asserts that the cardinality of each $\mathcal{W}_j$ is upper bounded by 
$$ |\mathcal{W}_j| \leq \widetilde{O}(2^j),
\qquad 1\leq j \leq \overline{j}.$$
Importantly, this lemma tells us that, with probability at least $1-8T^4k^2\delta' = 1-\delta/4$, one has
\begin{align}
	\| \overline{w}^T\|_1 = \sum_{i=1}^k \max_{1\leq t\leq T}w_i^t & \leq  k\cdot 2^{-(\overline{j}-1)} +\sum_{j=1}^{\overline{j}} |\mathcal{W}_j| 2^{-(j-1)} \nonumber \\
	&\leq k\cdot \frac{800\big(\log_2(1/\eta)+1\big)^2\log_2^2(k)}{k\log^2(2)} +\sum_{j=1}^{\overline{j}} |\mathcal{W}_j| 2^{-(j-1)} \notag
\\ & \leq 2\cdot 10^8\cdot \left(\big(\log_2(1/\eta)+1\big)^2\log_2^2(k)\big(\log(Tk)+\log(1/\delta)\big)^3\big(\log_2(T)+1\big)\right), \nonumber
\end{align}
where the first inequality is valid since $\max_{1\leq t\leq T}w_i^t \leq 2^{-(j-1)}$ holds for any $i\in \mathcal{W}_j$.  
This immediately concludes the proof of Lemma~\ref{lemma:main}, 
as long as Lemma~\ref{lemma:bdcount} can be established. 

Noteworthily, proving Lemma~\ref{lemma:bdcount} is the most challenging part of our analysis, 
and we dedicate the next subsection (Section~\ref{sec:control-trajectory})  to the proof of Lemma~\ref{lemma:bdcount}.

\subsection{Proof of Lemma~\ref{lemma:bdcount} (bounding $|\mathcal{W}_j|$ for each $j$)}
\label{sec:control-trajectory}

%
The proof relies heavily on the concepts of ``segments'' introduced in Section~\ref{sec:tec2}. 
%
%
%
%
Throughout this section, we shall take $\delta' = \delta/(32T^4k^2)$,  
and focus on any $j$ obeying 
\begin{align}
	1\leq j \leq \log_2(k)-2.
	\label{eq:j-range-proof}
\end{align}
%

%



Before continuing, we find it helpful to underscore a high-level idea: 
if $|\mathcal{W}_j|$ is large,  then there exist many disjoint segments, thereby requiring the total length $T$ to be large enough in order to contain these segments. 
The key steps to construct such disjoint segments of interest are as follows:



%
%
\begin{enumerate}
\item Construct a suitable segment for each $i\in \mathcal{W}_j$ (see Lemma~\ref{lemma:aux1});
\item Identify sufficiently many disjoint blocks such that the segments within each block have nonempty intersection (see Lemma~\ref{lemma:part2} 
and Figure~\ref{fig5});  
		%
%
\item From the above disjoint blocks, identify sufficiently many disjoint subsets such that the distributions associated with each subset can be linked with a common (sub)-segment,  
	and that each of these (sub)-segments experiences sufficient changes between its starting and end points 
		(see Lemma~\ref{lemma:part3}, Figure~\ref{fig6} and Figure~\ref{fig8}).
\end{enumerate}
In the sequel, we shall present the details of our proof, which consist of multiple steps.

\subsubsection{Step 1: showing existence of a segment for each distribution in $\mathcal{W}_j$}
\label{sec:lemma:aux1}

Recall that $\mathcal{W}_j$ contains those distributions whose corresponding weight iterates obey $\max_{1\leq t\leq T}w_i^{t}\in (2^{-j}, 2^{-j+1}]$  (cf.~\eqref{eq:defn-Wj-main-123}). 
%
%
As it turns out, for any $i\in \mathcal{W}_j$, one can find an $\big(\frac{1}{2^{j+1}},\frac{1}{2^{j+2}},\log(2)\big)$-segment,  
as stated in the lemma below. 
This basic fact allows one to link each distribution in $ \mathcal{W}_j$ with a segment of suitable parameters.


\begin{lemma}\label{lemma:aux1}  
For each $i \in \mathcal{W}_j$,
there exists $1\leq s_i<e_i\leq T$, such that 
\begin{align}
	\frac{1}{2^{j+2}}<w^{s_i}_i \leq \frac{1}{2^{j+1}}, 
	\qquad 
	w^{e_i}_i > \frac{1}{2^{j}} , \qquad
	\text{and} \qquad
	w^t_i >2^{-(j+2)} \quad \forall t \in[ s_i , e_i].
\end{align}
In other words, there exists a $\left(\frac{1}{2^{j+1}},\frac{1}{2^{j+2}},\log(2)\right)$-segment $(s_i,e_i)$ with the index set as $\{i\}$  
	(see Definition~\ref{def:seg}).
\end{lemma}
\begin{proof}
From the definition \eqref{eq:defn-Wj-main-123} of $\mathcal{W}_j$, it is straightforward to find a time point $e_i$ obeying $w^{e_i}_i > \frac{1}{2^{j}}$. It then remains to identify a valid point $s_i$. 
To this end, let us define 
\[
	\tau =\max\big\{t \mid t \leq e_i, w^t_i \leq 2^{-(j+2)} \big\},
\]
which is properly defined since $w^1_i = 1/k \leq 2^{-(j+2)}$ (see \eqref{eq:j-range-proof}). 
With this choice in mind, we have
$$
	w^t_i > 2^{-(j+2)}, \qquad \forall t \text{ obeying }\tau +1 \leq t\leq e_i.
$$
In addition, it follows from the update rule (cf.~lines~\ref{line:wt-update-Wt} and \ref{line:updateS11} of Algorithm~\ref{alg:main1}) that
\begin{align*}
\log(w_{i}^{t+1}/w_{i}^{t}) & =\log(W_{i}^{t+1}/W_{i}^{t})-\log\Big(\sum_{j}W_{j}^{t+1}/\sum_{j}W_{j}^{t}\Big)\\
 & \leq\eta-\log\Big(\sum_{j}W_{j}^{t+1}/\sum_{j}W_{j}^{t}\Big)\leq2\eta\leq1/10,
\end{align*}
where the last inequality results from our choice of $\eta$. This in turn allows us to show that
\begin{align}
	w^{\tau+1}_i \leq w^{\tau}_i  \exp(1/10)
	\leq \frac{1}{2^{j+2}}\cdot \exp(1/10) \leq \frac{1}{2^{j+1}}.
\end{align}
As a result, it suffices to choose $s_i = \tau+1$, thus concluding the proof.
\end{proof}

\subsubsection{Step 2: constructing disjoint segments with good coverage}
While Lemma~\ref{lemma:aux1} justifies the existence of suitable segments $\{(s_i,e_i)\}$ associated with each distribution in $\mathcal{W}_j$, 
we need to divide (a nontrivial subset of) them into certain disjoint blocks, where the segments in each block have at least one common inner points. 
This is accomplished in the following lemma. 
\begin{lemma}\label{lemma:part2} 
Recall the definition of $\mathcal{W}_j$ in \eqref{eq:defn-Wj-main-123}. 
	For each $i \in \mathcal{W}_j$, denote by $(s_i,e_i)$ the segment identified in Lemma~\ref{lemma:aux1}. 
Then there exist a group of disjoint subsets $\{\mathcal{W}_j^p\}_{p=1}^P$ of $\mathcal{W}_j$ obeying
\begin{enumerate}
	\item[$\mathrm{(i)}$]  $\mathcal{W}_j^p\subseteq \mathcal{W}_j$, 
$\mathcal{W}_j^p \cap \mathcal{W}_j^{p'}=\emptyset$,
$\forall p\neq p'$; 

	\item[$\mathrm{(ii)}$] $ \sum_{p=1}^P |\mathcal{W}_j^p|\geq \frac{|\mathcal{W}_j|}{3(\log_2(T)+1)}$;

	\item[$\mathrm{(iii)}$] Let  $\widetilde{s}_p = \min_{i\in \mathcal{W}_j^p}s_i$ and $\widetilde{e}_p  =\max_{i\in \mathcal{W}_j^p}e_i$ 
			for each $1\leq p \leq P$. One has
 $1\leq \widetilde{s}_1<\widetilde{e}_1\leq \widetilde{s}_2<\widetilde{e}_2\leq \dots \leq \widetilde{s}_P <\widetilde{e}_P\leq T$
		and $\max_{i\in \mathcal{W}_j^p}s_i \leq \min_{i \in\mathcal{W}_j^p}e_i$ for each $1\leq p\leq P$.
\end{enumerate}
\end{lemma}
\begin{proof} See Appendix~\ref{sec:proof-lemma:part2}. \end{proof}
\noindent 
In words, Lemma~\ref{lemma:part2} reveals the existence a collection of {\em disjoint} subsets of $\mathcal{W}_j$ such that (a) they account for 
a sufficiently large fraction of the indices contained in $\mathcal{W}_j$, 
and (b) 
the segments in each subset $\mathcal{W}_j^p$ share at least one common inner point.

Thus far, each of the segments constructed above is associated with a single distribution in $\mathcal{W}_j$. 
Clearly, it is likely that many of these segments might have non-trivial overlap; in other words, many of them might have shared sub-segments. 
What we intend to do next is to further group the indices in $\{\mathcal{W}_j^p\}$ into disjoint subgroups, 
and identify a common (sub)-segment for each of these subgroups.  
What remains unclear, however, is whether each of these (sub)-segments experiences sufficient weight changes between its starting and end points. 
We address these in the following lemma.


\begin{lemma}\label{lemma:part3}
%
Recall the definition of $\mathcal{W}_j$ in \eqref{eq:defn-Wj-main-123}. 
Then there exists  a group of subsets $\{\mathcal{V}_j^n\}_{n=1}^N$ of $\mathcal{W}_j$  satisfying the following properties: 
\begin{enumerate}
	\item[$\mathrm{(i)}$] $\mathcal{V}_j^n\subseteq \mathcal{W}_j$, $\mathcal{V}_j^n\cap \mathcal{V}_j^{n'}=\emptyset$, $\forall n\neq n'$; 
	\item[$\mathrm{(ii)}$] $\sum_{n=1}^N |\mathcal{V}_j^n|\geq \frac{|\mathcal{W}_j|}{24\log_2(k)(\log_2(T)+1)}$; 
	\item[$\mathrm{(iii)}$] There exist $1\leq \widehat{s}_1<\widehat{e}_1\leq \widehat{s}_2 <\widehat{e}_2\leq \dots \leq \widehat{s}_N <\widehat{e}_N\leq T$, and $\{g_n\}_{n=1}^N\in \left[1,\infty\right)^N$, such that for each $1\leq n \leq N$, $(\widehat{s}_n,\widehat{e}_n)$ is a $\left(2^{-(j+1)}g_n|\mathcal{V}_j^n|,  2^{-(j+2)}|\mathcal{V}_j^n| , \frac{\log(2)}{2\log_2(k)}  \right)$-segment with index set as $\mathcal{V}_j^n$. That is, the following properties hold for each $1\leq n \leq N$:
\begin{itemize}
		\item[$\mathrm{(a)}$] $\frac{g_n |\mathcal{V}_j^n|}{2^{j+2}}\leq \sum_{i\in \mathcal{V}_j^n} w_i^{\widehat{s}_n}\leq \frac{g_n |\mathcal{V}_j^n|}{2^{j+1}}$;
		\item[$\mathrm{(b)}$] $\frac{g_n|\mathcal{V}_j^n|}{2^{j}}\cdot \exp\left( \frac{\log(2)}{2\log_2(k)}\right)\leq \sum_{i\in \mathcal{V}_j^n}w_i^{\widehat{e}_n}$;
		\item[$\mathrm{(c)}$] $\sum_{i\in \mathcal{V}_j^n}w_i^{t}\geq \frac{|\mathcal{V}_j^n|}{2^{j+2}}$ for any $t$ obeying $\widehat{s}_n \leq t\leq \widehat{e}_n$.
\end{itemize}
\end{enumerate}
\end{lemma}
\begin{proof} See Appendix~\ref{sec:proof-lemma-part3}. \end{proof}
\begin{remark} The endpoints $\{\widehat{s}_n\}$ and $\{\widehat{e}_n\}$ mentioned above are carefully constructed in the proof. The quantities $\{g_n\}$ are only implicitly defined here, since their precise values (except the property $g_n\geq 1$) are not needed in the subsequent analysis; the interested reader can find more precise details about $\{g_n\}$ in the proof. 
\end{remark}


	In brief, Lemma~\ref{lemma:part3} identifies a collection of subsets $\{\mathcal{V}_j^n\}_{n=1}^N$ of $\mathcal{W}_j$ enjoying the following useful properties: 
    \begin{itemize}
    \item  They are disjoint (see property (i) in the lemma); as explained in Section~\ref{sec:tec22}, having disjoint subsets that can cover a good fraction of $\mathcal{W}_j$ facilitates analysis.  
    
    \item 
	While they might be unable to fully cover the set $\mathcal{W}_j$, 
    these subsets 
    taken collectively still cover a highly nontrivial fraction (i.e., at least $\frac{1}{\mathrm{poly}\log (k,T)}$)  of the elements in $\mathcal{W}_j$.  
    
\item Each subset $\mathcal{V}_j^n$ is linked with a suitable segment, whose associated weights have increased sufficiently from its starting point to its end point.  
    \end{itemize}

\subsubsection{Step 3: bounding the length of segments}
In this step, we turn attention to the length of segments, unveiling the interplay between the segment length and certain sub-optimality gaps.

Recall the definition \eqref{eq:defn-vt-gap} of $v^t$ as follows
\begin{align}
	v^t \coloneqq L(h^t, w^t) - \mathsf{OPT} 
	\qquad \text{with }\mathsf{OPT}\coloneqq \min_{\pi\in \Delta(\mathcal{H})}\max_{1\leq i\leq k} L(h_{\pi},\basis_i),
	\label{eq:defn-vt-gap-restated} 
\end{align}
with $\basis_i$ the $i$-th standard basis vector. 
The following lemma assists in bounding the length of the segments defined in Definition~\ref{def:seg}.  
\begin{lemma}\label{lemma:cut} Let $j_{\mathrm{max}} =\left\lfloor \log_2(1/\eta)\right\rfloor+1$.  Assume the conditions in Lemma~\ref{lemma:main} hold.  Suppose $(t_1,t_2)$ is a $\left(p,q,x\right)$-segment 
	satisfying $p\geq 2q>0$. Then one has 
	\begin{equation}
		t_2 - t_1 \geq \frac{x}{2\eta}.
		\label{eq:segment-length-t2-t1}
	\end{equation}

	Moreover, if  
	%
 \begin{equation}
		\frac{qx^2}{50\big(\log_2(1/\eta)+1\big)^2}\geq  \frac{1}{k}
	\end{equation}
	holds, then with probability exceeding $1-6T^4k\delta'$, 
 	at least one of the following two claims holds: 
	\begin{itemize}
		\item[(a)] the length of the segment satisfies
			\begin{equation}
				t_2-t_1
				\geq \frac{qx^2 }{200\big(\log_2(1/\eta)+1\big)^2\eta^2}.
			\end{equation}
		
		\item[(b)] the quantities $\{v^t\}$ obey
			%
			\begin{equation}
				%
				4\sum_{\tau=t_{1}}^{t_{2}-1}(-v^{\tau}+\varepsilon_1)\geq\frac{qx^{2}}{100\big(\log_{2}(1/\eta)+1\big)^{2}\eta}.
			\end{equation}
	\end{itemize}
\end{lemma}
\begin{proof} See Appendix~\ref{sec:pfcut}. \end{proof}
	In words, Lemma~\ref{lemma:cut} shows that (i) in general the length of a segment scales at least linearly in $1/\eta$; (ii) if the parameter $q$ of a segment (i.e., some lower bound on the weights of interest within this segment) is sufficiently large, then either the length of this segment scales at least {\em quadratically} in $1/\eta$, or the sum of certain sub-optimality gaps needs to be large enough.

\subsubsection{Step 4: putting all this together}

With the above lemmas in place, we are positioned to establish Lemma~\ref{lemma:bdcount}. 
In what follows, we denote by $\{\mathcal{V}_j^n\}_{n=1}^N$ and $\{(\widehat{s}_n,\widehat{e}_n)\}_{n=1}^N$ the construction in Lemma~\ref{lemma:part3}.

 %

To begin with, it is observed that: for any $1\leq j \leq \overline{j}$  
(with $\overline{j}$ defined in \eqref{eq:defn-jbar-Lemma3}),  one has
\begin{align}
2^{-(j+2)}|\mathcal{V}_j^n|\cdot \frac{\log^2(2)}{50(\log_2(1/\eta)+1)^2 \log_2^2(k)}
\geq 
2^{-(\overline{j}+2)}\cdot \frac{\log^2(2)}{50(\log_2(1/\eta)+1)^2 \log_2^2(k)}
\geq \frac{1}{k}.
\end{align}
Recall that $v^{\tau} \leq \varepsilon_1$ (see \eqref{eq:vt-UB-135}). 
Combining this fact with Lemma~\ref{lemma:cut} (by setting $q = 2^{-(j+2)}|{\mathcal{V}}_{j}^{n}|$ and $x = \frac{\log(2)}{\log_2(k)}$) reveals that,  for each $1\leq n \leq N$, 
\begin{align}
T\eta+\bigg\{4T\varepsilon_{1}+4\sum_{t=1}^{T}(-v^{t})\bigg\} & \geq\sum_{n=1}^{{N}}(\widehat{e}_{n}-\widehat{s}_{n})\eta+4\sum_{n=1}^{{N}}\sum_{\tau=\widehat{s}_{n}}^{\widehat{e}_{n}-1}(-v^{\tau}+\varepsilon_{1})\\
 & \geq\frac{2^{-(j+2)}\sum_{n=1}^{{N}}|{\mathcal{V}}_{j}^{n}|\log^{2}(2)}{800\log_{2}^{2}(k)\big(\log_{2}(1/\eta)+1\big)^{2}\eta},	
	\label{eq:Teta-sum-vt-bound-135}
\end{align}
where the first inequality results from the disjoint nature of the segments $\{[\widehat{s}_n,\widehat{e}_n]\}_{1\leq n\leq N}$, and the second inequality comes from Lemma~\ref{lemma:cut}. 
%
%
Moreover, it follows from \eqref{eq:s1} that
\begin{align}
\sum_{t= 1}^T (-v^{t}) \leq  \frac{\log(k)}{\eta}+\eta T + 4\sqrt{T\log(1/\delta')},
\end{align}
 which taken together with \eqref{eq:Teta-sum-vt-bound-135} and the choice $\varepsilon_1=\eta$  gives
 \begin{align}
\sum_{n=1}^{{N}} |{\mathcal{V}}_j^n| & \leq  \frac{3200\eta\big(\log_2(1/\eta)+1\big)^2 \log_2^2(k) \cdot 2^{j+2}}{\log^2(2)}\cdot  \left( \frac{\log(k)}{\eta}+\eta T +4\sqrt{T\log(1/\delta')}\right) \notag
\\ & \qquad \qquad \qquad \qquad \qquad \qquad \qquad \qquad +\frac{4000T\big(\log_2(1/\eta)+1\big)^2 \log_2^2(k)\cdot   2^{j+2}\eta^2}{\log^2(2)}  .
\label{eq:sum-V-overline-Nn}
 \end{align}

 To finish up, it follows from Property (ii) of Lemma~\ref{lemma:part3} that
\begin{align}
|\mathcal{W}_{j}| & \leq 24\log_{2}(k)\big(\log_{2}(T)+1\big) \left(\sum_{n=1}^{N}|\mathcal{V}_{j}^{n}|\right) \notag\\
 & \leq8\cdot10^{7}\cdot\left(\big(\log_{2}(1/\eta)+1\big)^{2}\log_2^2(k)\left(\log_2(k)+\log(1/\delta')\right)^{3}\big(\log_{2}(T)+1\big)\right)\cdot2^{j}\nonumber
\end{align}
for any $1\leq j \leq \overline{j}$. The proof is completed by recalling that $\delta'=\frac{\delta}{4(T+k+1)}$.

\section{Limitations of proper learning?}\label{sec:exact_learning}

Given that the best-known sample complexities prior to our work were derived for algorithms that  either  output randomized hypotheses or invoke majority votes, 
\citet{awasthi2023sample} raised the question about how the sample complexity is impacted if only deterministic (or ``proper'') hypotheses from $\mathcal{H}$ are permitted as the output of the learning algorithms. 
As it turns out, the restriction to proper learning algorithms substantially worsens the sample efficiency, 
as revealed by the following theorem.


\begin{theorem}\label{thm:lb} Assume that $d\geq 2\log(8k)$. Consider any $\varepsilon\in (0,1/100)$, and let $N_0 =\frac{2^d-1}{k} $. 
	One can find
	\begin{itemize}
		\item a hypothesis class $\mathcal{H}$ containing at most $kN_0 + 1$ hypothesis, 
		\item a collection of $k$ distributions $\mathcal{D} \coloneqq \{\mathcal{D}_i\}_{i=1}^k$, 
		\item  a loss function $\ell:\mathcal{H}\times \mathcal{X}\times \mathcal{Y}\to [-1,1]$,
	\end{itemize}
	such that: for any algorithm $\mathcal{G}$ that finds $h\in \mathcal{H}$ obeying
\begin{align}
	\max_{1\leq i\leq k} \mathop{\mathbb{E}}_{(x,y)\sim \mathcal{D}_i}\big[\ell\big(h,(x,y)\big)\big] \leq \min_{h'\in \mathcal{H}}\max_{i\in [k]} \mathop{\mathbb{E}}_{(x,y)\sim \mathcal{D}_i}\big[\ell\big(h',(x,y)\big)\big] + \varepsilon
\end{align}
	with probability exceeding $3/4$, its sample size --- denoted by $M(\mathcal{G})$ --- must exceed  $\mathbb{E}[M(\mathcal{G})]\geq \frac{dk}{240000\varepsilon^2}$.
\end{theorem}
\noindent 
In words, the sample complexity when a deterministic output from $\mathcal{H}$ is required scales at least with $\Omega(\frac{dk}{\varepsilon^2}\big)$, 
which is considerably larger than the sample complexity $\widetilde{O}\big(\frac{d+k}{\varepsilon^2}\big)$ achievable via Algorithm~\ref{alg:main1} (which outputs a randomized hypothesis). This demonstrates the benefits of randomization, or more broadly, the necessity of improper learning. 
The proof of this theorem is deferred to Appendix~\ref{sec:proof-thm:lb}.


\section{Extension: learning Rademacher classes}\label{sec:rad}

In this section, we adapt our algorithm and theory to accommodate MDL for Rademacher classes. Note that when learning Rademacher classes, the hypotheses in $\mathcal{H}$ do not need to be binary-valued. 

\subsection{Preliminaries: Rademacher complexity}

Let us first introduce the formal definition of the Rademacher complexity; more detailed introduction can be found in, e.g.,  \citet{shalev2014understanding}.

 \begin{definition}[Rademacher complexity]\label{def:rad}
Given a distribution $\mathcal{D}$ supported on $\mathcal{Z} \coloneqq \mathcal{X}\times \mathcal{Y}$ and a positive integer $n$, the Rademacher complexity is defined as
\begin{align}
	\mathsf{Rad}_{n}(\mathcal{D})
	\coloneqq 
	\mathop{\mathbb{E}}_{\{z_i\}_{i=1}^n  }\left[ \mathop{\mathbb{E}}_{\{\sigma_i\}_{i=1}^n  }  \left[ \max_{h\in \mathcal{H}}\frac{1}{n}\sum_{i=1}^n \sigma_i\ell(h,z_i)\right] \right],
	\label{eq:defrad}
\end{align}
	 where $\{z_i\}_{i=1}^n$ are drawn independently from $\mathcal{D}$, and $ \{\sigma_i\}_{i=1}^n$ are i.i.d.~Rademacher random variables obeying $\mathbb{P}\{\sigma_i = 1\}=\mathbb{P}\{\sigma_i = -1\} = 1/2$ for each $1\leq i \leq n$.
\end{definition}

Next, we would like to make an assumption concerning the Rademacher complexity of mixtures of distributions. 
Denoting by $\mathcal{D}(w)$ the mixed distribution 
\begin{equation}
	\mathcal{D}(w)\coloneqq \sum_{i=1}^k w_i \mathcal{D}_i \label{eq:defn-Dw}
\end{equation}
for any probability vector $w=[w_i]_{1\leq i\leq k}\in \Delta(k)$, we can state our assumption as follows. 
\begin{assum}\label{assump:rad} 
	For each $n\geq 1$, there exists a quantity $C_n>0$ (known to the learner) such that
\begin{align}
	C_n \geq \sup_{w\in\Delta(k)}\mathsf{Rad}_{n}\big(\mathcal{D}(w)\big) .  \label{eq:radbb}
\end{align}
%
\end{assum}

\noindent 
For instance, if the hypothesis class $\mathcal{H}$ obeys $\mathsf{VC}\text{-}\mathsf{dim}(\mathcal{H})\leq d$, 
then it is well-known that Assumption~\ref{assump:rad} holds with the choice (see, e.g., \citet{mohri2018foundations})
$$
	C_n = \sqrt{\frac{2d\log(en/d)}{n}}.
$$

\begin{remark}\label{remark:rad}
	One might raise a natural question about Assumption~\ref{assump:rad}: can we use $\widetilde{C}_n \coloneqq  \max_{i\in [k]}\mathsf{Rad}_n(\mathcal{D}_i)$ instead of ${C}_n$ without incuring a worse sample complexity? The answer is, however, negative. In fact, the Rademacher complexity $\mathsf{Rad}_n(\mathcal{D}(w))$ is not convex in $w$, and hence we fail to use $\max_i\mathsf{Rad}_n(\mathcal{D}_i)$ to bound $\max_{w\in \Delta(k)}\mathsf{Rad}_n(\mathcal{D}(w))$. The interested reader is referred to Appendix~\ref{app:radlb} for more details. 
\end{remark}


To facilitate analysis, we find it helpful to introduce another notion called weighted Rademacher complexity.  

\begin{definition}[Weighted Rademacher complexity] Given a collection of distributions $\mathcal{D}=\{\mathcal{D}_i\}_{i=1}^k$ and a set of positive integers $\{n_i\}_{i=1}^k$, the weighted (average) Rademacher complexity is defined as
\begin{align}
\widetilde{\mathsf{Rad}}_{\{n_{i}\}_{i=1}^{k}}\coloneqq\mathop{\mathbb{E}}\limits _{\{z_{i}^{j}\}_{j=1}^{n_{i}},\forall i\in[k]}\left[\mathop{\mathbb{E}}\limits _{\{\sigma_{i}^{j}\}_{j=1}^{n_{i}},\forall i\in[k]}\left[\frac{1}{\sum_{i=1}^{k}n_{i}}\max_{h\in\mathcal{H}}\sum_{i=1}^{k}\sum_{j=1}^{n_{i}}\sigma_{i}^{j}\ell\big(h,z_{i}^{j}\big)\right]\right],
\end{align}
where $\big \{\{z_i^j\}_{j=1}^{n_k} \big\}_{i=1}^k$ are independently generated with each $z_i^j$ drawn from $\mathcal{D}_i$, and $\big\{ \{\sigma_i^j\}_{j=1}^{n_i} \big\}_{i=1}^k $ are independent Rademacher random variables obeying $\mathbb{P}\{\sigma_i^j = 1\}=\mathbb{P}\{\sigma_i^j = -1\} = 1/2$. Throughout the rest of this paper, we shall often abbreviate $\widetilde{\mathsf{Rad}}_{\{n_{i}\}_{i=1}^{k}}=\widetilde{\mathsf{Rad}}_{\{n_{i}\}_{i=1}^{k}}(\mathcal{D})$.   
\end{definition}
The following two lemmas provide useful properties about the weighted Rademacher complexity. 
\begin{lemma}\label{fact:wrc} For any two groups of positive integers $\{n_i\}_{i=1}^k$ and $\{m_i\}_{i=1}^k$, it holds that
\begin{align}
\left(\sum_{i=1}^k n_i\right) \widetilde{\mathsf{Rad}}_{ \{n_i\}_{i=1}^k } & \leq \left(\sum_{i=1}^k(m_i+n_i) \right) \widetilde{\mathsf{Rad}}_{ \{m_i+n_i \}_{i=1}^k }\nonumber
\\ & \leq \left(\sum_{i=1}^k n_i\right) \widetilde{\mathsf{Rad}}_{ \{n_i\}_{i=1}^k } + \left(\sum_{i=1}^k m_i\right)   \widetilde{\mathsf{Rad}}_{ \{m_i\}_{i=1}^k  }.
\end{align}
\end{lemma}
\begin{proof} See Appendix~\ref{sec:proof-fact:wrc}.\end{proof}
%
%
\begin{lemma}\label{lemma:bdwrc}
	Consider any $\{n_i\}_{i=1}^k$ obeying $n_i\geq 12\log(2k)$ for each $i\in [k]$. By taking $w\in\Delta(k)$ with $w_{i}=\frac{n_{i}}{\sum_{l=1}^{k}n_{l}}$,
one has
\begin{align}
	\widetilde{\mathsf{Rad}}_{\{n_{i}\}_{i=1}^{k}}\leq72\mathsf{Rad}_{\sum_{i=1}^{k}n_{i}}\big(\mathcal{D}(w)\big).
	\nonumber
\end{align}
\end{lemma}
\begin{proof} See Appendix~\ref{sec:proof-fact:bdwrc}.\end{proof}

\subsection{Algorithm and sample complexity}

We are now positioned to introduce our algorithm that learns a Rademacher class in the presence of multiple distributions, 
which is also based on a Hedge-type strategy to learn a convex-concave game; 
see Algorithm~\ref{alg:rad} for full details. 
Its main distinction from Algorithm~\ref{alg:main1} lies in the subroutine to learn $h^t$ (see lines~\ref{line:low-switching-r}-\ref{line:a2r} in Algorithm~\ref{alg:rad}) as well as the choice of $\Tone$ (see line~\ref{line:higher-par-rad} in Algorithm~\ref{alg:rad}). More precisely, to compute the estimator $\widehat{L}^t(h,w^t)$ for $L(h,w^t)$, instead of using the first $n_i^t$ samples from $\mathcal{D}_i$ for each $i\in [k]$, we choose to use the first 
$$n_i^{t,\mathsf{rad}}  =  \min\big\{\left\lceil \Tone w_i^t +12\log(2k)\right\rceil, \Tone \big\}$$ 
samples from $\mathcal{D}_i$ for each $i$, where $\Tone$ is taken to be
\begin{equation}
	\label{eq:Tone-defn-rad}
	\Tone =\min\Big\{t\geq \frac{4000(k\log (k/\varepsilon_1 )+\log(1/\delta))}{\varepsilon_1^2} \,\Big|\, C_t\leq \frac{\varepsilon_1}{4800} \Big\} .
\end{equation}
Here, $\Tone$ needs to take advantage of the quantities $\{C_n\}$ (cf.~Assumption~\ref{assump:rad}) that upper bound the associated Rademacher complexity. 

\begin{algorithm}[t]
\footnotesize
	\DontPrintSemicolon
	\caption{Hedge for multi-distribution learning on Rademacher Classes ($\mathtt{MDL}\text{-}\mathtt{Hedge}\text{-}\mathtt{Rad}$)\label{alg:rad}}
	\textbf{input:} $k$ data distributions $\{\mathcal{D}_1,\mathcal{D}_2,\ldots,\mathcal{D}_k\}$,  hypothesis class $\mathcal{H}$, target accuracy level $\varepsilon$, target success rate $1-\delta$, quantities $\{C_n\}_{n\geq 1}$ as in Assumption~\ref{assump:rad}. \\
\textbf{hyper-parameters:} stepsize $\eta=\frac{1}{100}\varepsilon$, number of rounds $T= \frac{20000\log\left(\frac{k}{\delta}\right)}{\varepsilon^2}$, 
	auxiliary accuracy level $\varepsilon_1=\frac{1}{100}\varepsilon$, 
	auxiliary sub-sample-size $\Tone =\min\Big\{t\geq \frac{4000(k\log (k/\varepsilon_1 )+\log(1/\delta))}{\varepsilon_1^2} \,\Big|\, C_t\leq \frac{\varepsilon_1}{4800} \Big\} $. \label{line:higher-par-rad}\\
	\textbf{initialization:} 
	for all $i \in [k]$, set $W^1_i=1$, $\widehat{w}^0_i=0$ and $n_i^0=0$;  
	$\weighteddata=\emptyset$.  \\
   draw $\left\lceil 12\log(2k)\right\rceil$ samples from $\mathcal{D}_i$ for each $i$, and add these samples to $\weighteddata$.\label{line:rad}  \\
	\For{$t=1,2,\ldots, T$}{
		set $w^t= [w^t_i]_{1\leq i\leq k}$  and $\widehat{w}^t= [\widehat{w}^t_i]_{1\leq i\leq k}$,  
		with $w^t_i \leftarrow \frac{W^t_i}{\sum_{j}W^t_j}$ and $\widehat{w}^t_i \leftarrow \widehat{w}^{t-1}_i$ for all $i\in [k]$. \\
		%
		%
		{\color{blue}\tcc{recompute $\widehat{w}^t$ \& draw new samples for $\mathcal{S}_{\mathrm{w}}$ only if $w^t$ changes sufficiently.}\label{line:a1r}}
		\If{there exists $j\in [k]$ such that $w_j^t \geq 2\widehat{w}_j^{t-1}$ \label{line:low-switching-r}}{ 
			$\widehat{w}^t_i \leftarrow  \max\{w^{t}_i, \widehat{w}^{t-1}_i\}$ for all $i\in [k]$; \\
			\For{$i=1,\ldots,k$}{  $n_i^t \leftarrow \left\lceil \Tone \widehat{w}^t_i \right\rceil$; \\
			draw $n_i^t - n_i^{t-1}$ independent samples from $\mathcal{D}_i$, and add these samples to $\weighteddata$.\label{line:sample1r}
			}
		}
		{\color{blue}\tcc{estimate the near-optimal hypothesis for weighted data distributions.}}
		compute $h^t\leftarrow \arg\min_{h\in \mathcal{H}} \widehat{L}(h,w^t)$, where 
		\begin{align}
			\widehat{L}^t(h,w^t):=\sum_{i=1}^k \frac{w^t_i}{n_i^{t,\mathsf{rad}}}\cdot \sum_{j=1}^{n_i^{t,\mathsf{rad}}}\ell\big(h,(x_{i,j},y_{i,j})\big)\label{eq:e1r}
		\end{align}
		with $n_i^{t,\mathsf{rad}}  =  \min\{\left\lceil \Tone w_i^t +12\log(2k)\right\rceil, \Tone \}$ and $(x_{i,j},y_{i,j})$ being the $j$-th datapoint from $\mathcal{D}_i$ in $\weighteddata$.\label{line:a2r}\\
		{\color{blue}\tcc{estimate the loss vector and execute weighted updates.}}
		$\overline{w}^t_i \leftarrow \max_{1\leq \tau \leq t}w^{\tau}_i$ for all $i \in [k]$. \\
		\For{$i=1,\ldots, k$}{ 
		draw $ \lceil k\overline{w}^t_i \rceil$ independent samples --- denoted by $\big\{(x^t_{i,j},y^t_{i,j})\big\}_{j=1}^{\lceil k\overline{w}_i^t\rceil}$ --- from $\mathcal{D}_i$, and set \label{line:updateSr} 
		$$
			\widehat{r}^t_i  =\frac{1}{\lceil k\overline{w}^t_i\rceil}\sum_{j=1}^{\lceil k\overline{w}^t_i \rceil} \ell\big(h^t,(x_{i,j}^t,y_{i,j}^t) \big); 
		$$
		
		update the weight as $W^{t+1}_i = W^t_i  \exp(\eta \widehat{r}_i^t)$. \label{line:updateS11r} {\color{blue}\tcp{Hedge updates.}}
		}
	}
	\textbf{output:}  a randomized hypothesis $h^{\mathsf{final}}$ as a uniform distribution over $ \{h^t\}_{t=1}^T$.
\end{algorithm}

Equipped with this algorithm, we are ready to establish the following theoretical guarantees.
\begin{theorem}\label{thm:rad}
Suppose Assumption~\ref{assump:rad} holds. With probability at least $1-\delta$, the output $h^{\mathsf{final}}$ returned by Algorithm~\ref{alg:rad} satisfies 
    \begin{align}
	    \max_{1\leq i\leq k}\mathop{\mathbb{E}}\limits _{(x,y)\sim\mathcal{D}_{i},h^{\mathsf{final}}}\left[\ell\big(h^{\mathsf{final}},(x,y)\big)\right]\leq \min_{h\in\mathcal{H}}\max_{1\leq i\leq k}\mathop{\mathbb{E}}\limits _{(x,y)\sim\mathcal{D}_{i}}\big[\ell\big(h,(x,y)\big)\big] + \varepsilon.  
    \end{align}
In particular, the total number of samples collected by Algorithm~\ref{alg:rad} is bounded by  
$$
	\left(\frac{k}{\varepsilon^2} + \min \big\{ n \mid C_n \leq c_1 \varepsilon\big\} \right) \,\mathrm{poly}\log\Big(k,\frac{1}{\varepsilon},\frac{1}{\delta}\Big)
$$
with probability exceeding  $1-\delta$, 
	where $c_1>0$ is some sufficiently small constant, and $\Tone$ is defined in \eqref{eq:Tone-defn-rad}. 
\end{theorem}
In contrast to the VC classes, 
the sample complexity for learning Rademacher classes 
entails a term related to the Rademacher complexity --- namely, $\min \big\{ n \mid C_n \leq c_1 \varepsilon\big\}$ --- 
as opposed to the term $d/\varepsilon^2$ concerning the VC-dimension. 
When it comes to the special case where $\mathsf{VC}\text{-}\mathsf{dim}(\mathcal{H})\leq d$, 
taking $C_n \leq \sqrt{\frac{2d\log(en/d)}{n}}$ in \eqref{eq:Tone-defn-rad} leads to $\Tone = \widetilde{O}\left( \frac{d+k}{\varepsilon^2}\right)$, 
which recovers the  sample complexity bound of $\widetilde{O}\left( \frac{d+k}{\varepsilon^2}\right)$ derived in Theorem~\ref{thm:main} for VC classes. 

\begin{proof}[Proof of Theorem~\ref{thm:rad}]
In view of Lemma~\ref{lemma:opt} and Lemma~\ref{lemma:main}, 
it boils down to showing that  running Algorithm~\ref{alg:rad} results in $L(h^t,w^t)\leq \min_{h\in \mathcal{H}} L(h,w^t)+\varepsilon_1$ for any $1\leq t\leq T$, 
a property that holds with probability at least $1-\delta/4$. To accomplish this, we have the lemma below.

\begin{lemma}\label{lemma:conrad} Suppose Assumption~\ref{assump:rad} holds. 
With probability at least $1-\delta/4$, the iterates of Algorithm~\ref{alg:rad} satisfy 
\begin{align}
L(h^t,w^t)\leq \min_{h\in \mathcal{H}}L(h,w^t)+\varepsilon_1
\end{align}
for any $1\leq t\leq T$.
\end{lemma}
\noindent 
The proof of Lemma~\ref{lemma:conrad} is postponed to Appendix~\ref{sec:mp_rad}.
Combine this with Lemma~\ref{lemma:opt} and Lemma~\ref{lemma:main} to show that: 
the total number of samples collected in Algorithm~\ref{alg:rad} is upper bounded by 
\[
	\Tone \,\mathrm{poly}\log\Big(k,\frac{1}{\varepsilon},\frac{1}{\delta}\Big) \leq \left(\frac{k}{\varepsilon^2} + \min \bigg\{ n \mid C_n \leq \frac{1}{4800} \varepsilon\bigg\} \right) \,\mathrm{poly}\log\Big(k,\frac{1}{\varepsilon},\frac{1}{\delta}\Big)
\]
as claimed. 
\end{proof}




\section{Extension: multi-loss multi-distribution learning}
\label{sec:multiloss}

In this section, we consider an extension of the MDL problem to the multi-loss (or multi-objective) setting.  

\medskip
\noindent 
{\bf Problem setting.} 
Formally, consider again the hypothesis class $\mathcal{H}$ and the $k$ distributions of interest $\{\mathcal{D}_i\}_{i=1}^k$. Suppose that there are $R\geq 2$ loss functions $\mathcal{L} \coloneqq \{\ell^j\}_{j=1}^{R}$ to consider, where $\ell^j: \mathcal{H}\times \mathcal{X}\times \mathcal{Y} \rightarrow [-1,1]$ denotes the $j$-th loss function of interest.
%
%
The goal is to find a (possibly randomized) hypothesis $\widehat{h}\in \mathcal{H}$ such that 
\begin{align}
\max_{1\leq i\leq k}\max_{\ell\in \mathcal{L}}\mathop{\mathbb{E}}\limits_{(x,y)\sim \mathcal{D}_i, \widehat{h} }\big[\ell\big(\widehat{h}, (x,y)\big) \big]\leq \min_{h\in \mathcal{H}}\max_{1\leq i \leq k, \ell\in \mathcal{L}}\mathbb{E}_{(x,y)\sim \mathcal{D}_i}\big[\ell\big(h,(x,y)\big)\big]+\varepsilon.
\end{align}

Evidently, this formulation is a natural extension of the single-loss MDL problem. It is particularly relevant in scenarios where multiple criteria need to be satisfied or optimized simultaneously, such as those involving fairness and robustness. It has also been shown that this extension could be helpful towards solving the multicalibration problem~\citep{haghtalab2023unifying}.

\medskip
\noindent 
{\bf Additional notation.} Let us introduce one more notation to be used throughout. For any $i\in [k]$ and $\ell\in \mathcal{L}$, we denote by $L^{\ell}_i(h)$ the expected loss of $h$ w.r.t.~$\mathcal{D}_i$ and $\ell$, that is, 
\begin{align}
L_i^{\ell}(h) \coloneqq  \mathop{\mathbb{E}}\limits_{(x,y)\sim \mathcal{D}_i}\big[   \ell\big(h,(x,y)\big)  \big].
\end{align}
Then for any $u\in \Delta(\mathcal{D}\times \mathcal{L})$, define
\begin{equation}
L(h, u) \coloneqq\sum_{i,\ell} u_{i,\ell}L^{\ell}_{i}(h).
\end{equation}

\medskip
\noindent 
{\bf The proposed algorithm and main theorem.} 
As it turns out, Algorithm~\ref{alg:main1} can be extended to tackle the above multi-loss multi-distribution learning problem. The full algorithm is presented in Algorithm~\ref{alg:obj}. As a key distinction from Algorithm~\ref{alg:main1},  we need to maintain a Hedge-type algorithm over $\mathcal{D}\times \mathcal{L}$ in an attempt to solve the following problem:
\begin{align}
\max_{u\in \Delta(\mathcal{D}\times \mathcal{L})}\min_{h\in \mathcal{H}}L(h, u).
\nonumber
\end{align}
Akin to Algorithm~\ref{alg:main1}, in each round $t$, we maintain a dataset $\mathcal{S}$ to construct estimation of $L(h,u^t)$ for all $h\in \mathcal{H}$, which help us to compute the optimal response $h^t$ with respect to the current weight $u^{t}$. Then we query a small number (at most $\tilde{O}(k)$) of new samples to assist in estimating each $L^{\ell}_i(h^t)$. Notably, we need to choose $T_1$ to be larger than the case of Algorithm~\ref{alg:main1}, in order to apply the union bound argument. Given that $u^t$ is in the $kR$-dimensional simplex $\Delta(\mathcal{D}\times \mathcal{L})$, applying the union bound uniformly over all possible choices of $u^t$ would lead to an undesirable factor of $\widetilde{O}(kR)$. Instead of this analysis strategy, we consider controlling an upper bound on the error , which allows us to pay a smaller factor of $O(k\log(R))$ instead.

The following theorem establishes the sample complexity of Algorithm~\ref{alg:obj}.
\begin{theorem}\label{thm:multi-obj}
There exists an algorithm (see Algorithm~\ref{alg:obj} for details) such that: with probability exceeding $1-\delta$, the randomized hypothesis $h^{\mathsf{final}}$ returned by this algorithm achieves
\begin{align}
\max_{1\leq i\leq k}\max_{\ell\in \mathcal{L}}\mathop{\mathbb{E}}\limits_{(x,y)\sim \mathcal{D}_i, h^{\mathsf{final}} }\big[\ell\big(h^{\mathsf{final}}, (x,y)\big) \big]\leq \min_{h\in \mathcal{H}}\max_{1\leq i \leq k, \ell\in \mathcal{L}}\mathop{\mathbb{E}}\limits_{(x,y)\sim \mathcal{D}_i}\big[\ell\big(h,(x,y) \big) \big]+\varepsilon,\nonumber
\end{align}
provided that the total sample size exceeds
\begin{align}
 \frac{\big(d+k\log(R)\big)\min\{\log(R),k\}}{\varepsilon^2}\mathrm{poly}\log\left(k,d,\frac{1}{\varepsilon},\frac{1}{\delta}, \log(R)\right).
\end{align}
\end{theorem}

 The full proof of Theorem~\ref{thm:multi-obj} is provided in Appendix~\ref{app:multiobj}. 
 Most parts of the proof are identical to that of Theorem~\ref{thm:main}, except that: (a) the new method in estimating $\{L(h,u^t  )\}_{h\in \mathcal{H}}$ mentioned above (see also Lemma~\ref{lemma:opth_multi}); (b) a more refined analysis of the trajectory bound to remove some $\mathrm{poly}(\log(R))$ factors (see Lemma~\ref{lemma:sc_obj}).

 In comparison to the single-loss case (see Theorem~\ref{thm:main}), the multi-loss version of our algorithm only incurs an additional price of $\mathrm{poly}(\log(R))$ factors, which is particularly appealing when the number $R$ of loss functions is no greater than a polynomial function in $d,k,$ etc.




\begin{algorithm}
	\DontPrintSemicolon
	\caption{Hedge for multi-loss multi-distribution learning ($\mathtt{MLMDL}\text{-}\mathtt{Hedge}\text{-}\mathtt{VC}$)\label{alg:obj}}
	\textbf{input:} $k$ data distributions $\{\mathcal{D}_i\}_{i=1}^k$, loss function class  $\mathcal{L}=\{\ell^j\}_{j=1}^R$, hypothesis class $\mathcal{H}$, target accuracy level $\varepsilon$, target success rate $1-\delta$. \\
\textbf{hyper-parameters:} stepsize $\eta=\frac{1}{100}\varepsilon$, number of rounds $T= \frac{20000\log\left(\frac{kR}{\delta\varepsilon}\right)}{\varepsilon^2}$, 
	auxiliary accuracy level $\varepsilon_1=\frac{1}{100}\varepsilon$, 
	auxiliary sub-sample-size $T_1 := 
 \frac{40000\left(k\log ( \frac{kR}{\varepsilon_1} )+d\log\big( (\frac{kd}{\varepsilon_1} )+\log(\frac{1}{\delta}) \big) \right)}{\varepsilon_1^2}.$ \\
	\textbf{initialization:} 
	for all $i \in [k]$ and $\ell\in \mathcal{L}$, set $U^1_{i,\ell}=1$, $\widehat{w}^0_{i}=0$ and $n_i^0=0$;  
	$\weighteddata=\emptyset$.  \\
	\For{$t=1,2,\ldots, T$}{
		set $u^t= [u^t_{i,\ell}]_{1\leq i\leq k,\ell \in \mathcal{L}}$
		with $u^t_{i,\ell} \leftarrow \frac{U^t_{i,\ell}}{\sum_{i'\in [k],\ell'\in \mathcal{L}}U^t_{i',\ell'}}$ .\\
  set  $\widehat{w}^t= [\widehat{w}^t_{i}]_{1\leq i\leq k}$ with  $\widehat{w}^t_{i} \leftarrow \widehat{w}^{t-1}_{i}$ for all $i\in [k]$. \\
		%
		%
		{\color{blue}\tcc{recompute $\widehat{w}^t$ \& draw new samples for $\mathcal{S}$ only if $w^t$ changes sufficiently.}}
		\If{there exists $j\in [k]$ such that $\sum_{\ell\in \mathcal{L}}u_{j,\ell}^t \geq 2\hat{w}_{j}^{t-1}$ \label{line:a1-obj}}{ 
			$\widehat{w}^t_{i} \leftarrow  \max\big\{\sum_{\ell\in \mathcal{L}}u_{i,\ell}^t, \widehat{w}^{t-1}_{i} \big\}$ for all $i\in [k]$; \\
			\For{$i=1,\ldots,k$}{  $n_i^t \leftarrow \left\lceil T_1 \widehat{w}^t_{i} \right\rceil$; \\
			draw $n_i^t - n_i^{t-1}$ independent samples from $\mathcal{D}_i$, and add these samples to $\weighteddata$.\label{line:sample1-obj}
			}
		}
		{\color{blue}\tcc{estimate the near-optimal hypothesis for weighted data distributions.}}
		compute $h^t\leftarrow \arg\min_{h\in \mathcal{H}} \widehat{L}^t(h,u^t)$, where 
		\begin{align}
			\widehat{L}^t(h,u^t):=\sum_{i=1}^k\sum_{\ell\in \mathcal{L}} \frac{u^t_{i,\ell}}{n_i^t}\cdot \sum_{j=1}^{n_i^t}\ell\big(h,(x_{i,j},y_{i,j})\big)\label{eq:e1obj}
		\end{align}
		with $(x_{i,j},y_{i,j})$ being the $j$-th datapoint from $\mathcal{D}_i$ in $\weighteddata$. \label{line:a2obj}
  \\
		{\color{blue}\tcc{estimate the loss vector and execute weighted updates.}}
		$\overline{w}^t_i \leftarrow \max_{1\leq \tau \leq t}\sum_{\ell\in \mathcal{L}}u^{\tau}_{i,\ell}$ for all $i \in [k]$. \\
		\For{$i=1,\ldots, k$}{ 
		draw $ \lceil k\overline{w}^t_i \rceil$ independent samples --- denoted by $\big\{(x^t_{i,j},y^t_{i,j})\big\}_{j=1}^{\lceil k\overline{w}_i^t\rceil}$ --- from $\mathcal{D}_i$, and set \label{line:updateSobj} 
		$$
			\widehat{r}^t_{i,\ell} =\frac{1}{\lceil k\overline{w}^t_i\rceil}\sum_{j=1}^{\lceil k\overline{w}^t_i \rceil} \ell\big(h^t,(x_{i,j}^t,y_{i,j}^t) \big) \qquad \text{for each }\ell\in \mathcal{L}; 
		$$
		
		update the weight as $U^{t+1}_{i,\ell} = U^t_{i,\ell}  \exp(\eta \widehat{r}_{i,\ell}^t)$. \label{line:updateS12-obj} {\color{blue}\tcp{Hedge updates.}}
		}
	}
	\textbf{output:}  a randomized hypothesis $h^{\mathsf{final}}$ as a uniform distribution over $ \{h^t\}_{t=1}^T$.
\end{algorithm}

\section{Discussion}

In this paper, we have settled the problem of achieving optimal sample complexity in multi-distribution learning, assuming availability of  adaptive (or on-demand) sampling. 
We have put forward a novel oracle-efficient algorithm that provably attains a sample complexity of $\widetilde{O}\left(\frac{d+k}{\varepsilon^2}\right)$ for VC classes, 
which matches the best-known lower bound up to some logarithmic factor.  
From the technical perspective, the key novelty of our analysis lies in carefully bounding the trajectory of the Hedge algorithm on a convex-concave optimization problem. 
We have further unveiled the necessity of improper learning,  revealing that a considerably larger sample size is necessary if the learning algorithm is constrained to return deterministic hypotheses from $\mathcal{H}$. 
Notably, our work manages to solve three open problems presented in COLT 2023 (namely, \citet[Problems 1, 3 and 4]{awasthi2023sample}).

Our work not only addresses existing challenges but also opens up several directions for future exploration. 
To begin with, while our sample complexity results are optimal up to logarithmic factors,  further studies are needed in order to sharpen the logarithmic dependency.  
Additionally, the current paper assumes a flexible sampling protocol that allows the learner to take samples arbitrarily from any of the $k$ distributions; 
how will the sample complexity be impacted under additional constraints imposed on the sampling process? 
Furthermore,  can we extend our current analysis (which bounds the dynamics of the Hedge algorithm) to control the trajectory of more general first-order/second-order algorithms, 
in the context of robust online learning?  
Another venue  for exploration  is the extension of our multi-distribution learning framework to tackle other related tasks like multi-calibration \citep{hebert2018multicalibration,haghtalab2023unifying}.  
We believe that our algorithmic and analysis framework can shed light on making progress in all of these directions.

\section*{Acknowledgements}
We thank Eric Zhao for answering numerous questions about the open problems.  
We would also like to thank Chicheng Zhang for pointing out an error in a preliminary version of this paper. 
YC is supported in part by the Alfred P.~Sloan Research Fellowship, 
the NSF grants CCF-1907661, DMS-2014279, IIS-2218713 and IIS-2218773. JDL acknowledges support of the ARO under MURI Award W911NF-11-1-0304,  the Sloan Research Fellowship, NSF CCF 2002272, NSF IIS 2107304,  NSF CIF 2212262, ONR Young Investigator Award, and NSF CAREER Award 2144994.

\appendix

\section{Auxiliary lemmas}
\label{sec:auxiliary-lemmas}

In this section, we introduce several technical lemmas that are used multiple times in our analysis. 

We begin by introducing three handy concentrations inequalities. The first result is the well-renowned Freedman inequality 
\citep{freedman1975tail}, which assists in deriving variance-aware concentration inequalities for martingales.  
%
\begin{lemma}[Freedman's inequality~\citep{freedman1975tail}]\label{freedman}
    	Let $(M_{n})_{n\geq 0}$ be a  martingale obeying $M_{0}=0$. Define $V_{n} \coloneqq \sum_{k=1}^{n}\mathbb{E}[(M_{k}-M_{k-1})^{2} \,|\, \mathcal{F}_{k-1}]$ for each $n\geq 0$,
	where $\mathcal{F}_{k}$ denotes the $\sigma$-algebra generated by $(M_{1},M_{2},\dots,M_{k})$. 
	Suppose that $M_k - M_{k-1}\leq 1$ for all $k\geq 1$. 
	Then for any $x>0$ and $y>0$, one has
	\begin{equation}\label{Bernstein2}
		\mathbb{P}\big(M_{n}\geq nx,V_{n}\leq ny \big)
		\leq \exp\left(-\frac{nx^{2}}{2(y+\frac{1}{3}x)} \right).
	\end{equation}
\end{lemma}
The second concentration result bounds the difference between the sum of a sequence of random variables and the sum of their respective conditional means (w.r.t.~the associated $\sigma$-algebra).

%

\begin{lemma}[Lemma 10 in \cite{zhang2022horizon}]\label{lemma:con1}
Let $X_1,X_2,\ldots$ be a sequence of random variables taking value in the interval $[0,l]$. 
For any $k\geq 1$, let $\mathcal{F}_k$ be the $\sigma$-algebra generated by $(X_1,X_2,\ldots,X_k)$, and define 
	$Y_k := \mathbb{E}[X_k \mid \mathcal{F}_{k-1}]$. Then for any $\delta>0$, we have 
\begin{align}
	& \mathbb{P}\left\{ \exists n\in \mathbb{N}, \sum_{k=1}^n X_k \geq  3\sum_{k=1}^n Y_k+ l\log\frac{1}{\delta}\right\}
	\leq \delta,\nonumber
	\\  & \mathbb{P}\left\{  \exists n\in \mathbb{N},  \sum_{k=1}^n Y_k \geq 3\sum_{k=1}^n X_k + l\log\frac{1}{\delta}  \right\}    \leq \delta .\nonumber 
\end{align}
\end{lemma}

The third concentration result is the Mcdiarmid inequality, a celebrated inequality widely used to control the flucutaion of multivariate functions when the input variables are independently generated. 
\begin{lemma}[Mcdiarmid’s inequality]\label{lemma:mcinequality}
Let $X_1,X_2,\ldots, X_n$ be a sequence of independent random variables, with $X_i$ supported on $\mathcal{X}_i$. Let $f:\mathcal{X}_1\times \mathcal{X}_2\times\dots\times \mathcal{X}_n\to \mathbb{R}$ be a function such that: for any $i\in [n]$ and any $\{x_1,\ldots,x_n\}\in \mathcal{X}_1\times  \dots\times \mathcal{X}_n$, 
\begin{equation*}
	\sup_{x_i' \in \mathcal{X}_i}\big|f(x_1,\cdots,x_i,\cdots,x_n)-f(x_1,\cdots,x'_{i},\cdots,x_n) \big|\leq c
\end{equation*}
holds for some quantity $c>0$. It then holds that 
\begin{align}
	\mathbb{P}\Big\{ \big|f(X_1,X_2,\cdots, X_n) - \mathbb{E}\big[f(X_1,X_2,\cdots,X_n)\big] \big|\geq \varepsilon \Big\} 
	\leq 2\exp\left( -\frac{2\varepsilon^2}{nc^2}\right).\nonumber
\end{align}
\end{lemma}

Additionally, the following lemma presents a sort of the data processing inequality w.r.t.~the Kullback-Leibler (KL) divergence, which is a classical result from information theory. 
\begin{lemma}\label{lemma:klbound} Let $\mathcal{X}$ and $\mathcal{Y}$ be two sets, and consider any function $f:\mathcal{X}\to \mathcal{Y}$. 
	For any two random variables $X_1$ and $X_2$ supported on $\mathcal{X}$, it holds that
\begin{align}
	\mathsf{KL}\big( \mu(X_1) \,\|\, \mu(X_2) \big)\geq \mathsf{KL}\big( \mu\big(f(X_1)\big) \,\|\, \mu\big(f(X_2)\big) \big),
\end{align}
	where we use $\mu(Z)$ to denote the distribution of a random variable $Z$. 
\end{lemma}

Lastly, let us make note of an elementary bound regarding the KL divergence between two Bernoulli distributions.
\begin{lemma}\label{lemma:klcmp}
Consider any $q>0$ and $x\in [0,\log(2)]$. Also, consider any $y,y'\in (0,1)$ obeying $y\geq q$ and $y'\geq \exp(x)y$. It then holds that
\begin{align}
\mathsf{KL}\left( \mathsf{Ber}(y)\,\|\, \mathsf{Ber}(y')  \right) \geq \frac{qx^2}{4},\nonumber
\end{align}
where $\mathsf{Ber}(z)$ denotes the Bernoulli distribution with mean $z$.
\end{lemma}
\begin{proof}
To begin with, the function defined below satisfies
\begin{align*}
f(a,b) \coloneqq  \mathsf{KL}\big( \mathsf{Ber}(a)\,\|\, \mathsf{Ber}(b)  \big)  & = a\log\left(\frac{a}{b}\right)+(1-a)\log\left(\frac{1-a}{1-b}\right).
\end{align*}
For any $ 0<a\leq b \leq 1$, it is readily seen that 
$$
	\frac{\partial f(a,b)}{\partial b} = -\frac{a}{b}+\frac{1-a}{1-b} =\frac{b-a}{b(1-b)} \geq 0.
$$
%
%
%
It follows from our assumptions $y\geq q$ and $y'\geq \exp(x)y$ that
\begin{align}
	\mathsf{KL}\left(\mathsf{Ber}(y)\,\|\,\mathsf{Ber}(y')\right) & =f(y,y')=f(y,y)+\int_{y}^{y'}\frac{\partial f(y,z)}{\partial z}\mathrm{d}z=\int_{y}^{y'}\frac{z-y}{z(1-z)} \mathrm{d}z\nonumber\\
 & \geq\frac{1}{y'}\int_{y}^{y'}(z-y)\mathrm{d}z\geq\frac{(y'-y)^{2}}{2y'}\nonumber\\
 & \geq\frac{(y'-y)(1-\exp(-x))}{2}\nonumber\\
 & \geq\frac{y(\exp(x)-1)^{2}}{4}\geq\frac{qx^{2}}{4}, \nonumber
\end{align}
	where the penultimate inequality uses $x\in [0, \log(2)]$, and the last inequality holds since $y\geq q$.  
\end{proof}

Finally, let us present a basic property related to Rademacher random variables, which will play a useful role in understanding the Rademacher complexity. 
\begin{lemma}\label{lemma:add12} Let $\mathcal{L}$ be a set of vectors in $\mathbb{R}^n$. Let $w^1,w^2\in \mathbb{R}^n$ be two vectors obeying $ |w^1_i|\leq |w_i^2|$ for all $i\in [n]$. Then it holds that
\begin{align}
	\mathop{\mathbb{E}}_{ \{\sigma_i\}_{i=1}^n  }\left[ \max_{f\in \mathcal{L}} \sum_{i=1}^n \sigma_i w^1_i f_i\right]
	\leq \mathop{\mathbb{E}}_{ \{\sigma_i\}_{i=1}^n  }\left[ \max_{f\in \mathcal{L}} \sum_{i=1}^n \sigma_i w^2_i f_i\right],\label{eq:wpo}
\end{align}
where $\{\sigma_i\}$ is a collection of independent Rademacher random variables obeying $\mathbb{P}(\sigma_i=1)=\mathbb{P}(\sigma_i=-1)=1/2$. 
\end{lemma}
\begin{proof}
Clearly, it suffices to prove \eqref{eq:wpo} for the special case where $w_i^1 =w^2_i$ for $1\leq i\leq n-1$, and $|w^1_n| \leq |w^2_n|$. 
Fixing $\sigma_{i}$ for $1\leq i \leq n-1$, we can deduce that 
\begin{align}
\mathop{\mathbb{E}}_{\sigma_{n}}\left[\max_{f\in\mathcal{L}}\sum_{i=1}^{n}\sigma_{i}w_{i}f_{i}\right] & =\frac{1}{2}\max_{f\in\mathcal{L}}\left(\sum_{i=1}^{n-1}\sigma_{i}w_{i}f_{i}+w_{n}f_{n}\right)+\frac{1}{2}\max_{f\in\mathcal{L}}\left(\sum_{i=1}^{n-1}\sigma_{i}w_{i}f_{i}-w_{n}f_{n}\right)\nonumber\\
 & =\frac{1}{2}\max_{f,\widetilde{f}\in\mathcal{L}}\left(\sum_{i=1}^{n-1}\sigma_{i}w_{i}(f_{i}+\widetilde{f}_{i})+w_{n}(f_{n}-\widetilde{f}_{n})\right)\nonumber\\
 & =\frac{1}{2}\max_{f,\widetilde{f}\in\mathcal{L}}\left(\sum_{i=1}^{n-1}\sigma_{i}\widetilde{w}_{i}(f_{i}+\widetilde{f}_{i})+w_{n}(f_{n}-\widetilde{f}_{n})\right)\nonumber\\
 & \leq\frac{1}{2}\max_{f,\widetilde{f}\in\mathcal{L}}\left(\sum_{i=1}^{n-1}\sigma_{i}\widetilde{w}_{i}(f_{i}+\widetilde{f}_{i})+\big|\widetilde{w}_{n}(f_{n}-\widetilde{f}_{n})\big|\right)\nonumber\\
 & =\frac{1}{2}\max_{f,\widetilde{f}\in\mathcal{L}}\left(\sum_{i=1}^{n-1}\sigma_{i}\widetilde{w}_{i}(f_{i}+\widetilde{f}_{i})+\widetilde{w}_{n}(f_{n}-\widetilde{f}_{n})\right)\nonumber\\
 & =\frac{1}{2}\max_{f\in\mathcal{L}}\left(\sum_{i=1}^{n-1}\sigma_{i}\widetilde{w}_{i}f_{i}+\widetilde{w}_{n}f_{n}\right)+\frac{1}{2}\max_{f\in\mathcal{L}}\left(\sum_{i=1}^{n-1}\sigma_{i}\widetilde{w}_{i}f_{i}-\widetilde{w}_{n}f_{n}\right)\nonumber\\
 & =\mathop{\mathbb{E}}_{\sigma_{n}}\left[\max_{f\in\mathcal{L}}\sum_{i=1}^{n}\sigma_{i}\widetilde{w}_{i}f_{i}\right].\nonumber
\end{align}
The proof is thus completed by taking expectation over $\{\sigma_i\}_{i=1}^{n-1}$.
\end{proof}

\section{Proofs of auxiliary lemmas for VC classes}
\label{sec:proof-lemmas-VC}

\subsection{Proof of Lemma~\ref{lemma:opth}}
	\label{sec:proof-lemma:opth}

For ease of presentation, suppose there exists a dataset $\widetilde{\mathcal{S}}$ containing $\Tone$ independent samples drawn from each distribution $\mathcal{D}_i$ ($1\leq i\leq k$), 
	so that in total it contains $k\Tone$ samples.   
We find it helpful to introduce the following notation. 
\begin{itemize}

	\item For each $i\in [k]$ and $ j\in [n_i]$, denote by $(x_{i,j},y_{i,j})$  the $j$-th sample in $\widetilde{\mathcal{S}}$ that is drawn from $\mathcal{D}_i$.

	\item For each set of integers $n=\{n_i\}_{i=1}^k\in \mathbb{N}^k$, 
we define
 $\widetilde{\mathcal{S}}(n)$ to be  
		the dataset containing $\big\{ (x_{i,j},y_{i,j}) \big\}_{1\leq j\leq n_i}$ for all $i \in [k]$; namely, 
		it comprises, for each $i \in [k]$, the first $n_i$ samples in $\widetilde{\mathcal{S}}$ that are drawn from $\mathcal{D}_i$.

	\item We shall also let $\widetilde{\mathcal{S}}^+(n) =\big\{ \big\{\big(x_{i,j}^+,y_{i,j}^+\big) \big\}_{j=1}^{n_i} \big\}_{i=1}^k$ 
		be an {\em independent copy} of $\widetilde{\mathcal{S}}(n)$, 
		where for each $i \in [k]$,  $\big\{ \big(x_{i,j}^+,y_{i,j}^+\big) \big\}$ are independent samples drawn from $\mathcal{D}_i$.

\end{itemize}

\noindent 
Equipped with the above notation, we are ready to present our proof.

\paragraph{Step 1: concentration bounds for any fixed $n = \{n_i\}_{i=1}^k$ and $w\in \Delta(k)$.} 

Consider first any fixed $n = \{n_i\}_{i=1}^k$ obeying $0 \leq n_i \leq \Tone$ for all $i \in [k]$, 
and any fixed $w\in \Delta(k)$. For any quantity $\lambda \in \big[0, \min_{i\in [k]} \frac{n_i}{w_i}\big]$, 
	if we take
\begin{align}
	E(\lambda,n,w)\coloneqq\mathop{\mathbb{E}}\limits _{\widetilde{\mathcal{S}}(n)}\left[\max_{h\in\mathcal{H}}\exp\left(\lambda\left\{\sum_{i=1}^{k}w_{i}\frac{1}{n_{i}}\sum_{i=1}^{n_{i}}\ell\big(h,(x_{i,j},y_{i,j})\big)-L(h, w)\right\}\right)\right]	
\end{align}
with the expectation taken over the randomness of $\widetilde{\mathcal{S}}(n)$, 
then we can apply a standard ``symmetrization'' trick to bound $E(\lambda,n,w)$ as follows:  
\begin{align}
E(\lambda,n,w) & \coloneqq\mathop{\mathbb{E}}\limits _{\widetilde{\mathcal{S}}(n)}\left[\max_{h\in\mathcal{H}}\exp\left(\lambda\left\{ \sum_{i=1}^{k}\frac{w_{i}}{n_{i}}\sum_{i=1}^{n_{i}}\ell\big(h,(x_{i,j},y_{i,j})\big)- L(h, w)\right\} \right)\right]\nonumber\\
 & =\mathop{\mathbb{E}}\limits _{\widetilde{\mathcal{S}}(n)}\left[\max_{h\in\mathcal{H}}\exp\left(\lambda\left\{ \sum_{i=1}^{k}\frac{w_{i}}{n_{i}}\sum_{i=1}^{n_{i}}\ell\big(h,(x_{i,j},y_{i,j})\big)-\mathop{\mathbb{E}}\limits _{\widetilde{\mathcal{S}}^{+}(n)}\left[\sum_{i=1}^{k}\frac{w_{i}}{n_{i}}\sum_{i=1}^{n_{i}}\ell\big(h,(x_{i,j}^{+},y_{i,j}^{+})\big)\right]\right\} \right)\right]\nonumber\\
 & \leq\mathop{\mathbb{E}}\limits _{\widetilde{\mathcal{S}}(n)}\left[\max_{h\in\mathcal{H}}\mathop{\mathbb{E}}\limits _{\widetilde{\mathcal{S}}^{+}(n)}\left[\exp\left(\lambda\left\{ \sum_{i=1}^{k}\frac{w_{i}}{n_{i}}\sum_{j=1}^{n_{i}}\Big(\ell\big(h,(x_{i,j},y_{i,j})\big)-\ell\big(h,(x_{i,j}^{+},y_{i,j}^{+})\big)\Big)\right\} \right)\right]\right]
	\notag\\
 & \leq\mathop{\mathbb{E}}\limits _{\widetilde{\mathcal{S}}(n),\widetilde{\mathcal{S}}^{+}(n)}\left[\max_{h\in\mathcal{H}}\exp\left(\lambda\left\{ \sum_{i=1}^{k}\frac{w_{i}}{n_{i}}\sum_{j=1}^{n_{i}}\Big(\ell\big(h,(x_{i,j},y_{i,j})\big)-\ell\big(h,(x_{i,j}^{+},y_{i,j}^{+})\big)\Big)\right\} \right)\right], \label{eq:xxxx}
\end{align}
where the last two inequalities follow from Jensen's inequality. 
%
%

Next, let $\sigma(n)\coloneqq \big\{\{\sigma_{i,j}\}_{j=1}^{n_i}\big\}_{i=1}^k$ be a collection of i.i.d.~Rademacher random variables obeying $\mathbb{P}(\sigma_{i,j}=1)=\mathbb{P}(\sigma_{i,j}=-1)= 1/2$.  
Denoting $\mathcal{C}=\big\{ (x_{i,j},y_{i,j})\big\} \bigcup \big\{ (x_{i,j}^+,y_{i,j}^+) \big\} $, 
we obtain
\begin{align}
 & \mathop{\mathbb{E}}\limits _{\widetilde{\mathcal{S}}(n),\widetilde{\mathcal{S}}^{+}(n)}\left[\max_{h\in\mathcal{H}}\exp\left(\lambda\left\{ \sum_{i=1}^{k}\frac{w_{i}}{n_{i}}\sum_{j=1}^{n_{i}}\Big(\ell\big(h,(x_{i,j},y_{i,j})\big)-\ell\big(h,(x_{i,j}^{+},y_{i,j}^{+})\big)\Big)\right\} \right)\right]\nonumber\\
 & =\mathop{\mathbb{E}}\limits _{\widetilde{\mathcal{S}}(n),\widetilde{\mathcal{S}}^{+}(n)}\left[\mathop{\mathbb{E}}\limits _{\sigma(n)}\left[\max_{h\in\mathcal{H}}\exp\left(\lambda\left\{ \sum_{i=1}^{k}\frac{w_{i}}{n_{i}}\sum_{j=1}^{n_{i}}\sigma_{i,j}\Big(\ell\big(h,(x_{i,j},y_{i,j})\big)-\ell\big(h,(x_{i,j}^{+},y_{i,j}^{+})\big)\Big)\right\} \right)\,\Big|\,\mathcal{C}\right]\right].\label{eq:base1}
\end{align}
Note that for any dataset  $\mathcal{C}$ with cardinality $|\mathcal{C}|$, 
 the Sauer–Shelah lemma \cite[Proposition~4.18]{wainwright2019high} 
 together with our assumption that $\mathsf{VC}\text{-}\mathsf{dim}(\mathcal{H})\leq d$ 
 tells us that the cardinality of the following set obeys
\begin{align}
	\big|\mathcal{H}(\mathcal{C})\big|\leq (|\mathcal{C}|+1)^{d} 
	\leq \big( |\widetilde{\mathcal{S}}|+|\widetilde{\mathcal{S}}^+|+1 \big)^{d}
	\leq(2k\Tone+1)^{d}, 
\end{align}
where $\mathcal{H}(\mathcal{C})$ denotes the set obtained
by applying all $h\in\mathcal{H}$ to the data points in $\mathcal{C}$,
namely, 
\begin{equation}
	\mathcal{H}(\mathcal{C}) \coloneqq \Big\{\big(h(x_{1,1}),h(x_{1,1}^{+}),h(x_{1,2}),h(x_{1,2}^{+}),\cdots\big)\mid h\in\mathcal{H}\Big\}.
\end{equation}
We shall also define $\mathcal{H}_{\min,\mathcal{C}}\subseteq \mathcal{H}$ to be the {\em minimum-cardinality subset} of  $\mathcal{H}$ that results in the same outcome as  $\mathcal{H}$ when applied to $\mathcal{C}$, 
namely, 
\[
	\mathcal{H}_{\min,\mathcal{C}}(\mathcal{C})  = \mathcal{H}(\mathcal{C}) 
	\qquad \text{and} \qquad \big| \mathcal{H}_{\min,\mathcal{C}} \big| = \big|\mathcal{H}(\mathcal{C}) \big|.
\]
With these in place, we can demonstrate that
\begin{align}
 & \mathop{\mathbb{E}}\limits _{\sigma(n)}\left[\max_{h\in\mathcal{H}}\exp\left(\lambda\left\{ \sum_{i=1}^{k}\frac{w_{i}}{n_{i}}\sum_{j=1}^{n_{i}}\sigma_{i,j}\Big(\ell\big(h,(x_{i,j},y_{i,j})\big)-\ell\big(h,(x_{i,j}^{+},y_{i,j}^{+})\big)\Big)\right\} \right)\,\Big|\,\mathcal{C}\right]\nonumber\\
 & =\mathop{\mathbb{E}}\limits _{\sigma(n)}\left[\max_{h\in\mathcal{H}_{\min,\mathcal{C}}}\exp\left(\lambda\left\{ \sum_{i=1}^{k}\frac{w_{i}}{n_{i}}\sum_{j=1}^{n_{i}}\sigma_{i,j}\Big(\ell\big(h,(x_{i,j},y_{i,j})\big)-\ell\big(h,(x_{i,j}^{+},y_{i,j}^{+})\big)\Big)\right\} \right)\,\Big|\,\mathcal{C}\right]\nonumber\\
 & \leq\mathop{\mathbb{E}}\limits _{\sigma(n)}\left[\sum_{h\in\mathcal{H}_{\min,\mathcal{C}}}\exp\left(\lambda\left\{ \sum_{i=1}^{k}\frac{w_{i}}{n_{i}}\sum_{j=1}^{n_{i}}\sigma_{i,j}\Big(\ell\big(h,(x_{i,j},y_{i,j})\big)-\ell\big(h,(x_{i,j}^{+},y_{i,j}^{+})\big)\Big)\right\} \right)\,\Big|\,\mathcal{C}\right]\nonumber\\
 & \leq\big|\mathcal{H}_{\min,\mathcal{C}}\big|\max_{h\in\mathcal{H}_{\min,\mathcal{C}}}\mathop{\mathbb{E}}\limits _{\sigma(n)}\left[\exp\left(\lambda\left\{ \sum_{i=1}^{k}\frac{w_{i}}{n_{i}}\sum_{j=1}^{n_{i}}\sigma_{i,j}\Big(\ell\big(h,(x_{i,j},y_{i,j})\big)-\ell\big(h,(x_{i,j}^{+},y_{i,j}^{+})\big)\Big)\right\} \right)\,\Big|\,\mathcal{C}\right]\nonumber\\
 & \leq\big(2k\Tone+1\big)^{d}\max_{h\in\mathcal{H}}\prod_{i=1}^{k}\prod_{j=1}^{n_{i}}\mathop{\mathbb{E}}\limits _{\sigma_{i,j}}\left[\exp\left(\lambda\left\{ \frac{w_{i}}{n_{i}}\sigma_{i,j}\Big(\ell\big(h,(x_{i,j},y_{i,j})\big)-\ell\big(h,(x_{i,j}^{+},y_{i,j}^{+})\big)\Big)\right\} \right)\,\Big|\,\mathcal{C}\right]\nonumber\\
 & \leq\big(2k\Tone+1\big)^{d}\exp\left(2\lambda^{2}\sum_{i=1}^{k}\frac{(w_{i})^{2}}{n_{i}}\right).\label{eq:xxxx2}
\end{align}
Here,  the last inequality makes use of fact $\big|\ell\big(h,(x_{i,j},y_{i,j})\big)-\ell\big(h,(x_{i,j}^{+},y_{i,j}^{+})\big)|\leq 2$ as well as the following elementary inequality 
\[
\mathop{\mathbb{E}}\limits _{\sigma_{i,j}}\big[\exp(\sigma_{i,j}x)\big]=\frac{1}{2}\big(\exp(x)+\exp(-x)\big)\leq\exp\big(x^{2}\big). 
\]
Taking \eqref{eq:xxxx}, \eqref{eq:base1} and \eqref{eq:xxxx2} together reveals that
%
\begin{align}
E(\lambda)\leq (2k\Tone+1)^{d}\exp\left(2\lambda^2 \sum_{i=1}^k \frac{(w_i)^2}{n_i}\right) .\label{eq:xxxx3}
\end{align}

Repeating the same arguments also yields an upper bound on the following quantity: 
\begin{align}
\overline{E}(\lambda) & \coloneqq\mathop{\mathbb{E}}\limits _{\widetilde{\mathcal{S}}(n)}\left[\max_{h\in\mathcal{H}}\exp\left(\lambda\left\{ L(h,w)-\sum_{i=1}^{k}\frac{w_{i}}{n_{i}}\sum_{i=1}^{n_{i}}\ell\big(h,(x_{i,j},y_{i,j})\big)\right\} \right)\right]\nonumber\\
 & \leq(2k\Tone+1)^{d}\exp\left(2\lambda^{2}\sum_{i=1}^{k}\frac{(w_{i})^{2}}{n_{i}}\right)\nonumber
\end{align}
for any $\lambda \in \big[0, \min_{i\in [k]} \frac{n_i}{w_i}\big]$. 
Taking the above two inequalities and applying the Markov inequality reveal that, 
for any $0<\varepsilon'\leq 1$,  
\begin{align}
 & \mathbb{P}\left(\max_{h\in\mathcal{H}}\left|\sum_{i=1}^{k}w_{i}\frac{1}{n_{i}}\sum_{i=1}^{n_{i}}\ell\big(h,(x_{i,j},y_{i,j})\big)-
	L(h,w)\right|\geq\varepsilon'\right)
	\notag\\
 &\qquad \leq\min_{0\leq\lambda\leq\min_{i}\frac{n_{i}}{w_{i}}}\frac{E(\lambda)+\overline{E}(\lambda)}{\exp(\lambda\varepsilon')}\nonumber\\
 & \qquad\leq\min_{0\leq\lambda\leq\min_{i}\frac{n_{i}}{w_{i}}}2\cdot(2k\Tone+1)^{d}\exp\left(2\lambda^{2}\sum_{i=1}^{k}\frac{(w_{i})^{2}}{n_{i}}-\lambda\varepsilon'\right).\label{eq:lxs}
\end{align}

\paragraph{Step 2: uniform concentration bounds over epsilon-nets w.r.t.~$n$ and $w$.}  
Next, we move on to extend the above result to uniform concentration bounds over all possible $n$ and $w$. 
Towards this, 
let us first introduce a couple of notation.
\begin{itemize}
	\item Let us use $\Delta_{\varepsilon_2}(k) \subseteq \Delta(k)$ to denote an $\varepsilon_2$-net of $\Delta(k)$ ---  
		namely, for any $x\in \Delta(k)$, there exists a vector $x_0\in \Delta_{\varepsilon_2}(k)$ obeying $\|x-x_0\|_{\infty} \leq \varepsilon_2$. 
		We shall choose $\Delta_{\varepsilon_2}(k)$ properly so that 
$$
	|\Delta_{\varepsilon_2}(k)|\leq (1/\varepsilon_2)^k.
$$

	\item Define the following set
$$
		\mathcal{B}=  \bigg\{n=\{n_i\}_{i=1}^k, w =\{w_i\}_{i=1}^k \,\Big|\,   \frac{n_i}{w_i} \geq \frac{\Tone}{2}  , 0\leq n_i\leq \Tone, \forall i \in [k], w\in \Delta_{\varepsilon_1/(8k)}(k) \bigg\},
$$
which clearly satisfies
\[
	|\mathcal{B}| \leq  \Tone^k \cdot \left(\frac{8k}{\varepsilon_1} \right)^k .
\]
\end{itemize}
Applying the union bound yields that, for any $0<\varepsilon'\leq 1$, 
\begin{align}
& \mathbb{P}\left( \exists (n,w)\in \mathcal{B}, \max_{h\in \mathcal{H}}\left|  \sum_{i=1}^k w_i \frac{1}{n_i}\sum_{i=1}^{n_i}\ell\big(h,(x_{i,j},y_{i,j})\big)-  L(h,w) \right| \geq \varepsilon'\right) \nonumber
	\\ & \leq \sum_{(n,w)\in \mathcal{B}}\min_{0\leq \lambda \leq \min_i \frac{n_i}{w_i}} 2\cdot (2k\Tone+1)^{d}\exp\left(2\lambda^2 \sum_{i=1}^k \frac{(w_i)^2}{n_i}-\lambda \varepsilon'\right) \nonumber
\\ & \leq \sum_{(n,w)\in \mathcal{B}}\min_{0\leq \lambda \leq \frac{\Tone}{2}} 2\cdot (2k\Tone+1)^{d}\exp\left(2\lambda^2 \cdot\frac{2}{\Tone}-\lambda \varepsilon'\right) \nonumber
\\ &\leq  \sum_{(n,w)\in \mathcal{B}}2\cdot (2k\Tone+1)^{d} \exp\left( -\frac{\Tone(\varepsilon')^2}{16}\right) \nonumber
\\ & \leq  |\mathcal{B}|\cdot 2 \cdot (2k\Tone+1)^{d}\exp\left(-\frac{\Tone(\varepsilon')^2}{16}\right)\nonumber
\\ & \leq 2\cdot  (8k\Tone/\varepsilon_1)^k (2k\Tone+1)^{d}\cdot \exp\left(-\frac{\Tone(\varepsilon')^2}{16}\right) ,\nonumber
\end{align}
where the second inequality holds since $\sum_{i=1}^{k}\frac{w_{i}^{2}}{n_{i}}\leq\frac{2}{\Tone}\sum_{i=1}^{k}w_{i}=\frac{2}{\Tone}$ 
(according to the definition of $\mathcal{B}$).


\paragraph{Step 3: concentration bounds w.r.t.~$n^t$ and $w^t$.} 

Let $\mathcal{S}^t$ denote the value of $\mathcal{S}$  after line~\ref{line:sample1} of Algorithm~\ref{alg:main1} 
 in the $t$-th round. 
Recall that $n^t=[n^t_i]_{1\leq i\leq k}$ denotes the number of samples for all $k$ distributions in $\mathcal{S}^t$, 
and let $w^t=[w^t_i]_{1\leq i\leq k}$ represent the weight iterates in the $t$-th round. 
It is easily seen from lines~\ref{line:a1} and \ref{line:nit} of Algorithm~\ref{alg:main1} that 
$n_i^t\leq \Tone$ and $n_{i}^{t}/w_{i}^{t}\geq n_{i}^{t}/(2\widehat{w}_{i}^{t})\geq \Tone/2$. 
For analysis purposes, it suffices to take $\mathcal{S}^t = \widetilde{\mathcal{S}}(n^t)$.

In view of the update rule in Algorithm~\ref{alg:main1}, one can always find $(n^t,\widetilde{w}^t)\in \mathcal{B}$ satisfying $\|\widetilde{w}^t-w^t\|_1\leq k\|\widetilde{w}^t-w^t\|_{\infty}\leq \varepsilon_1/8$. As a result, for any $0<\varepsilon'\leq 1$, we can deduce that
\begin{align}
& \mathbb{P}\left(\exists t\in [T], \max_{h\in \mathcal{H}}\left|  \sum_{i=1}^k w^t_i \frac{1}{n^t_i}\sum_{i=1}^{n^t_i}
	\ell\big(h,(x_{i,j},y_{i,j})\big)-  L(h, w^t) \right| \geq \varepsilon' +\frac{\varepsilon_1}{4}\right) \nonumber
\\ 
	& \qquad \leq \mathbb{P}\left(\exists t\in [T], \max_{h\in \mathcal{H}}\left|  \sum_{i=1}^k \widetilde{w}^t_i \frac{1}{n^t_i}\sum_{i=1}^{n^t_i}
	\ell\big(h,(x_{i,j},y_{i,j})\big)-  L(h, \widetilde{w}^t) \right| \geq \varepsilon' \right) \nonumber
\\
&\qquad \leq 2\cdot  (8k\Tone/\varepsilon_1)^k (2k\Tone+1)^{d}\cdot \exp\left(-\frac{\Tone(\varepsilon')^2}{16}\right),\label{eq:ulx2}
\end{align}
where the second inequality arises from the fact that 
$\frac{1}{n_{i}}\sum_{i=1}^{n_{i}}\ell\big(h,(x_{i,j},y_{i,j})\big)\in [-1,1]$ and $L(h,\widetilde{w}^t)\in [-1,1]$. 
Taking $\varepsilon' = \varepsilon_1/4$ and substituting $\Tone = \frac{4000\left(k\log(k/\varepsilon_1) +d\log(kd/\varepsilon_1)+\log(1/\delta)\right)}{\varepsilon_1^2}$ into \eqref{eq:ulx2}, we can obtain 
\begin{align}
& \mathbb{P}\left(\exists t\in [T], \max_{h\in \mathcal{H}}\left|  \sum_{i=1}^k w^t_i \cdot  \frac{1}{n^t_i}\sum_{i=1}^{n^t_i}\ell\big(h,(x_{i,j},y_{i,j})\big)-  L(h,w^t)\right| \geq \frac{\varepsilon_1}{2}\right) \nonumber
\\ & \leq 2\cdot  (8k\Tone/\varepsilon_1)^k (2k\Tone+1)^{d}\cdot \exp\left(-\frac{\Tone\varepsilon_1^2}{16}\right)\nonumber
\\ & \leq 2\cdot  (8k\Tone/\varepsilon_1)^k (2k\Tone+1)^{d}\cdot  \exp\Big( -10 \big(k\log(k/\varepsilon_1)+d\log(kd/\varepsilon_1)+\log(1/\delta)\big) \Big) \nonumber
\\ & \leq 2\cdot  (8k\Tone/\varepsilon_1)^k (2k\Tone+1)^{d}\cdot (k/\varepsilon_1)^{-10k} \cdot (kd/\varepsilon_1)^{-10d}\cdot \delta\nonumber
\\ & \leq \delta/4.\label{eq:uxv2}
\end{align}

\paragraph{Step 4: putting all this together.} 
Recalling that 
\begin{align}
\widehat{L}^t(h,w^t) = \sum_{i=1}^k w_i^t \cdot \frac{1}{n_i^t}\sum_{i=1}^{n_i^t}\ell\big(h,(x_{i,j},y_{i,j})\big),\nonumber 
\end{align}
one can see from \eqref{eq:uxv2} that,  
with probability exceeding $1-\delta/4$,  
\begin{align}
\left|\widehat{L}^t(h,w^t)-L(h,w^t) \right|\leq \frac{\varepsilon_1}{2} 
\end{align}
holds simultaneously for all $t\in [T]$ and all $h\in\mathcal{H}$. 
Additionally, observing that 
\begin{align}
h^t = \arg\min_{h\in \mathcal{H}} \widehat{L}^t(h,w^t),
\end{align}
we can immediately deduce that
\begin{align}
L(h^{t},w^{t})\leq\widehat{L}(h^{t},w^{t})+\frac{\varepsilon_{1}}{2}=\min_{h\in\mathcal{H}}\widehat{L}(h,w^{t})+\frac{\varepsilon_{1}}{2}\leq\min_{h\in\mathcal{H}}L(h,w^{t})+\varepsilon_{1}.	
\end{align}
This concludes the proof of Lemma~\ref{lemma:opth}.


\subsection{Proof of Lemma~\ref{lemma:opt}}
\label{sec:opt}

Before proceeding, let us introduce some additional notation. 
Let $\delta' \coloneqq \frac{\delta}{4(T+k+1)}$, and recall
$$
    \mathsf{OPT}=\min_{\pi\in \Delta(\mathcal{H})}\max_{1\leq i\leq k} L(h_{\pi},\basis_i) = \max_{w\in \Delta(k)}\min_{h\in \mathcal{H}}L(w,h)
$$
to be the optimal objective value. 
Additionally, set 
\begin{equation}
	v^t \coloneqq L(h^t,w^t)-\mathsf{OPT}. \label{eq:defn-vt-gap}
\end{equation}
It follows from the assumption of this lemma (i.e., $L(h^t,w^t)\leq \min_{h\in \mathcal{H}}L(h,w^t)+\varepsilon_1$) that 
\begin{align}
    v^t \leq \min_{h\in \mathcal{H}}L(h,w^t)- \mathsf{OPT} + \varepsilon_1 
	= \min_{h\in \mathcal{H}}L(h,w^t)- \max_{w\in \Delta(k)}\min_{h\in \mathcal{H}}L(h,w)+ \varepsilon_1 
    \leq 
    \varepsilon_1, 
    \qquad \forall 1\leq t\leq T. 
	\label{eq:vt-UB-135}
\end{align}

We now begin to present the proof. 
In view of the Azuma-Hoeffding inequality and the union bound, 
we see that with probability at least $1-(k+1)\delta'$, 
%
%
\begin{subequations}
\label{eq:concentration-sum-w-r}
\begin{align}
	\left| \sum_{t=1}^T \big\langle w^t, \widehat{r}^t \big\rangle -\sum_{t=1}^T  L(h^t,w^t)\right| &\leq 2\sqrt{T\log(1/\delta')},
	\\  \left| \sum_{t=1}^T \widehat{r}_i^t - \sum_{t=1}^T L(h^t, \basis_i)\right| &\leq 2\sqrt{T\log(1/\delta')}.
\end{align}
\end{subequations}
These motivate us to look at $\sum_{t=1}^T \big\langle w^t, \widehat{r}^t \big\rangle$ 
(resp.~$ \sum_{t=1}^T \widehat{r}_i^t$) as a surrogate for $\sum_{t=1}^T  L(h^t,w^t)$ (resp.~$\sum_{t=1}^T L(h^t, \basis_i)$).

We then resort to standard analysis for the Hedge algorithm. 
Specifically, direct computation gives 
\begin{align}
	\log\left(\frac{\sum_{i=1}^k W_{i}^{t+1}}{\sum_{i=1}^k W_i^t}\right)  
	& \overset{\text{(i)}}{=} \log\left(\sum_{i=1}^k w_i^t \exp(\eta \widehat{r}^t_i) \right)
	\overset{\text{(ii)}}{\leq} \log\left( \sum_{i=1}^k w_i^t \big(1+ \eta\widehat{r}_i^t + \eta^2 (\widehat{r}_i^t)^2 \big)  \right)  \notag\\
& \leq 
\log\left(1+\eta\sum_{i=1}^{k}w_{i}^{t}\widehat{r}_{i}^{t}+\eta^{2}\sum_{i=1}^{k}w_{i}^{t}(\widehat{r}_{i}^{t})^{2}\big)\right)
	\leq \eta \sum_{i=1}^k w_i^t \widehat{r}_i^t + \eta^2. \label{eq:bb1}
\end{align}
%
%
Here, (i) is valid since $w_i^t=\frac{W_i^t}{\sum_jW_j^t}$ and $W_i^{t+1}=W_i^t\exp(\eta \widehat{r}_i^t)$ 
(cf.~lines~\ref{line:wt-update-Wt} and \ref{line:updateS11} of Algorithm~\ref{alg:main1}); 
(ii) arises from the elementary inequality $e^x\leq 1+x+x^2$ for $x\in [0,1]$ as well as the facts that  $\eta \leq 1$ and $|\widehat{r}_i^t|\leq 1$. 
Summing the inequality \eqref{eq:bb1} over all $t$ and rearranging terms, we are left with 
\begin{align}
\eta\sum_{t=1}^{T}\big\langle w^{t},\widehat{r}^{t}\big\rangle & \geq\sum_{t=1}^{T}\left\{ \log\left(\frac{\sum_{i=1}^{k}W_{i}^{t+1}}{\sum_{i=1}^{k}W_{i}^{t}}\right)-\eta^{2}\right\} \notag\\
 & =\log\left(\sum_{i=1}^{k}W_{i}^{T+1}\right)-\log\left(\sum_{i=1}^{k}W_{i}^{1}\right)-T\eta^{2} \notag\\
 & \geq\max_{1\leq i\leq k}\log(W_{i}^{T+1})-\log(k)-T\eta^{2}\nonumber\\
 & \geq\eta\max_{1\leq i\leq k}\sum_{t=1}^{T}\widehat{r}_{i}^{t}-\log(k)-T\eta^{2},
\end{align}
where the penultimate lines makes use of $W_i^1=1$ for all $i\in [k]$, 
and the last line holds since $\log\big(W_{i}^{T+1}\big) = \log\big(W_{i}^{T}\exp(\eta\widehat{r}_{i}^{t})\big)\geq\eta\widehat{r}_{i}^{t}$.   
Dividing both sides by $\eta$ yields
\begin{align}
\sum_{t=1}^T \big\langle w^{t},\widehat{r}^{t}\big\rangle \geq \max_{i} \sum_{t=1}^T \widehat{r}^t_i - \left(   \frac{\log(k)}{\eta} +\eta T \right)
.
	\label{eq:wt-rt-inner-LB}
\end{align}

Combine the above inequality with \eqref{eq:concentration-sum-w-r} to show that, with probability at least $1-(k+1)\delta'$, 
\begin{align}
	\sum_{t=1}^T L(h^t,w^t) &\geq \max_{1\leq i\leq k} \sum_{t=1}^T L(h^t,\basis_i) - \left( \frac{\log(k)}{\eta}+\eta T + 4\sqrt{T\log(1/\delta')} \right)
	.\label{eq:s1}
\end{align}
Recalling that $\varepsilon_1 = \eta = \frac{1}{100}\varepsilon$ and $T = \frac{20000\log(\frac{k}{\delta' \varepsilon})}{\varepsilon^2}$, we can derive
\begin{align}
	\max_{1\leq i\leq k} \sum_{t=1}^T L(h^t,\basis_i) & \leq T\mathsf{OPT}+ \sum_{t=1}^T v^t+\left( \frac{\log(k)}{\eta}+\eta T + 4\sqrt{T\log(1/\delta')} \right) \nonumber
\\ & \leq T\mathsf{OPT}+ T\varepsilon_1  +\left( \frac{\log(k)}{\eta}+\eta T + 4\sqrt{T\log(1/\delta')} \right)  \nonumber
\\ & \leq T\mathsf{OPT}+ T\varepsilon,
\end{align}
where the penultimate line results from \eqref{eq:vt-UB-135}. 
Given that $h^{\mathsf{final}}$ is taken to be uniformly distributed over $\{h^t\}_{1\leq t\leq T}$, 
we arrive at
\begin{align}
\max_{1\leq i\leq k}L(h^{\mathsf{final}},\basis_{i})=\max_{1\leq i\leq k}\frac{1}{T}\sum_{t=1}^{T}L(h^{t},\basis_{i})\leq\mathsf{OPT}+\varepsilon 
\end{align}
with probability at least $1-(k+1)\delta'$. 
This concludes the proof by recalling that $\delta'=\frac{\delta}{4(T+k+1)}$.

\begin{remark}
	Note that the proof of this lemma works as long as $\widehat{r}_i^t\in [0,1]$ is an unbiased estimate of $L(h^t,\basis_i)$ for each $i\in [k]$,  
	regardless of how many samples are used to construct $\widehat{r}_i^t$. 
\end{remark}

%
%
%
%
%

\subsection{Proof of Lemma~\ref{lemma:part2}}
\label{sec:proof-lemma:part2}
For any integer $1\leq x \leq \log_2(T)+1$,  define 
$$
	\mathcal{W}_j(x) \coloneqq \big\{i \in [k] \mid 2^{x-1} \leq e_i-s_i \leq 2^x \big\},
$$ 
so that the length of each segment associated with $\mathcal{W}_j(x)$ lies within  $[2^{x-1}, 2^x]$. 
Let $x^{\star}$ indicate the one that maximizes the cardinality of $\mathcal{W}_j(x)$: 
$$
	x^{\star} =\arg\max_{1\leq x\leq \log_2(T)+1} | \mathcal{W}_j(x)|.
$$
Given that there are at most $\log_2(T)+1$ choices of $x$,  the pigeonhole principle gives 
\begin{align}
	|\mathcal{W}_j(x^{\star})|\geq \frac{|\mathcal{W}_j|}{\log_2(T)+1}.
	\label{eq:Wj-xstar-size-LB}
\end{align}
In the sequel, we intend to choose the subsets $\{\mathcal{W}_j^m\}_{m=1}^M$ from $\mathcal{W}_j(x^{\star})$.

To proceed, let us set 
\begin{subequations}
\begin{align}
	\kappa_1 \coloneqq \min_{i \in \mathcal{W}_j(x^{\star})}e_i, \qquad 
	\mathcal{U}_j^1 \coloneqq \big\{ i\in \mathcal{W}_j(x^{\star}) \mid s_i \leq \kappa_1 \big\}, 
\end{align}
and then for each $o\geq 1$,  take
\begin{align}
	 \kappa_{o+1} &\coloneqq \min_{ i \in \mathcal{W}_j(x^{\star})/\cup_{o'=1}^{o}\mathcal{U}_j^{o'}}e_i,
	\\  \mathcal{U}_j^{o+1} &\coloneqq \big\{i\in \mathcal{W}_j(x^{\star})/\cup_{o'=1}^{o}\mathcal{U}_j^{o'} \mid s_i \leq \kappa_{o+1} \big\}.
\end{align}
\end{subequations}
We terminate such constructions until $\cup_{o\geq 1}\mathcal{U}_j^o=\mathcal{W}_j(x^{\star})$. 
By construction, for each $o$, we have 
\begin{align}
	s_{i_2}\leq \kappa_o\leq e_{i_1}, \quad \forall i_1,i_2\in \mathcal{U}_j^o \qquad 
	\Longleftrightarrow \qquad \max_{i \in \mathcal{U}_j^o}s_i \leq \min_{i\in \mathcal{U}_j^o}e_i.
\end{align}
%
%
Let us look at the three groups of subsets of $\mathcal{W}_j(x^{\star})$: $\{\mathcal{U}_j^{3o-2}\}_{o\geq 1}$, $\{\mathcal{U}_j^{3o-1}\}_{o\geq 1}$ and $\{\mathcal{U}_j^{3o}\}_{o\geq 1}$. 
Clearly, there exists $l \in \{0,1,2\}$ such that $\sum_{o\geq 1}|\mathcal{U}_j^{3o-l}|\geq \frac{1}{3}\sum_{o\geq 1}|\mathcal{U}_j^o|$; without loss of generality,  assume that 
\begin{equation}
	\sum_{o\geq 1}|\mathcal{U}_j^{3o-2}|\geq \frac{1}{3}\sum_{o\geq 1} |\mathcal{U}_j^o| = \frac{1}{3} |\mathcal{W}_j(x^{\star})|.
	\label{eq:Uj-3o-i-LB}
\end{equation}

With the above construction in place, we would like to verify that $\{\mathcal{U}_j^{3o-2}\}_{o\geq 1}$ forms the desired group of subsets. 
	First of all, Condition (i) holds directly from the definition of $\{\mathcal{U}_j^o\}_{o\geq 1}$. 
	When it comes to Condition (ii), it follows from \eqref{eq:Uj-3o-i-LB} and \eqref{eq:Wj-xstar-size-LB} that
\begin{align}
\sum_{o\geq 1}|\mathcal{U}_j^{3o-2}|\geq 
	\frac{1}{3} |\mathcal{W}_j(x^{\star})|\geq \frac{|\mathcal{W}_j|}{3(\log_2(T)+1)}.\nonumber
\end{align}
Regarding Condition (iii), it suffices to verify that 
\begin{align}
	\max_{i\in \mathcal{U}_j^{3o-2}}e_i \leq \min_{i\in \mathcal{U}_j^{3o+1}}s_i
\end{align}
for any $o$. 
To do so, note that 
for each $o\geq 1$, there exists $i\in \mathcal{W}_j(x^{\star})$ such that $s_i\geq \kappa_o$ and $\kappa_{o+1}=e_i$. 
We can then deduce that 
\begin{equation}
	\kappa_{o+1}=e_i \geq s_i + 2^{x^{\star}-1}\geq \kappa_o + 2^{x^{\star}-1}.
\end{equation}
It then follows that, for any $i\in \mathcal{U}_j^{3o+1}$, one has 
$$
	s_i \geq \kappa_{3o}\geq \kappa_{3o-1}+2^{x^{\star}-1} \geq \kappa_{3o-2}+2^{x^{\star}}.
$$
In addition, for any $l \in\mathcal{U}_j^{3o-2}$, it is seen that 
$$
	e_l \leq s_l + 2^{x^{\star}}\leq \kappa_{3o-2}+2^{x^{\star}}.
$$
Putting all this together yields
\begin{align}
\max_{i\in \mathcal{U}_j^{3o-2}}e_i \leq \kappa_{3o-2}+2^{x^{\star}}\leq  \min_{i\in \mathcal{U}_j^{3o+1}}s_i. \nonumber
\end{align}
The proof is thus complete.
%

\subsection{Proof of Lemma~\ref{lemma:part3}}
\label{sec:proof-lemma-part3}
We shall begin by presenting our construction of the subsets, followed by justification of the advertised properties. 
In what follows, we set $x = \log(2)$. 

\paragraph{Our construction.} 
Let $\{\mathcal{W}_j^p\}_{p=1}^P$ and $\{(\widetilde{s}_p,\widetilde{e}_p)\}_{p=1}^P$ be the construction in Lemma~\ref{lemma:part2}.  
%

\medskip 
\noindent 
	{\em Step a): constructing $\widehat{\mathcal{W}}_j^p$.}  
Consider any $1\leq p \leq P$. Set
$$
	t_{\mathrm{mid}}^p \coloneqq \min_{i\in \mathcal{W}_j^p}e_i,
$$
which, in view of Lemma~\ref{lemma:part2}, is a common inner point of the segments in $\mathcal{W}_j^p$. 
We can derive, for each $i\in \mathcal{W}_j^p$,  
\begin{align}
\max\left\{ \log\left( \frac{w^{e_i}_i}{w^{t_{\mathrm{mid}}^p}_i} \right),\log\left( \frac{w^{t_{\mathrm{mid}}^p}_i}{w^{s_i}_i} \right)\right\}\geq \frac{1}{2}\log\left( \frac{w^{e_i}_i}{w^{t_{\mathrm{mid}}^p}_i} \right)+\frac{1}{2}\log\left( \frac{w^{t_{\mathrm{mid}}^p}_i}{w^{s_i}_i} \right)
	= \frac{1}{2}\log\left( \frac{w^{e_i}_i}{w^{s_i}_i} \right) 
	\geq \frac{x}{2},
	\nonumber
\end{align}
where the last inequality holds since $(s_i,e_i)$ is constructed to be a $\big(\frac{1}{2^{j+1}}, \frac{1}{2^{j+2}}, x \big)$-segment 
(see Lemma~\ref{lemma:aux1}). 	
It then follows that
\begin{align}
	\sum_{i\in \mathcal{W}_j^p} \left(\mathds{1}\left\{\log\left( \frac{w^{e_i}_i}{w^{t_{\mathrm{mid}}^p}_i} \right)\geq \frac{x}{2}\right\} + \mathds{1}\left\{\log\left( \frac{w^{e_i}_i}{w^{t_{\mathrm{mid}}^p}_i} \right)\geq \frac{x}{2} \right\} \right) \geq |\mathcal{W}_j^p|. \nonumber
\end{align}
Without loss of generality, we assume that 
\begin{align}
	\sum_{i\in \mathcal{W}_j^p} \mathds{1}\left\{\log\left( \frac{w^{e_i}_i}{w^{t_{\mathrm{mid}}^p}_i} \right)\geq \frac{x}{2}
	\right\} \geq \frac{|\mathcal{W}_j^p|}{2} .
\end{align}
This means that the set define below
\begin{align}
	\widehat{\mathcal{W}}_j^p \coloneqq \left\{i\in \mathcal{W}_j^p \mid \log\big(w^{e_i}_i/w^{t_{\mathrm{mid}}^p}_i \big) 
	\geq \frac{x}{2}  \right\}
	\label{eq:defn-widehat-W-jp}
\end{align}
satisfies
\begin{equation}
	|\widehat{\mathcal{W}}_j^p|\geq \frac{|\mathcal{W}_j^p|}{2}.
	\label{eq:size-Wjp-179}
\end{equation}
In what follows, we take\footnote{We assume $\widetilde{Q}\geq 2$ without loss of generality. In the case $\widetilde{Q}=1$, we simply choose an arbitrary element in $\widehat{\mathcal{W}}_j^p$ as a single subset. In this way, we can collect at least $\frac{1}{4}|\widehat{\mathcal{W}_j^p}|$ segments.}
$$
	Q(p) \coloneqq |\widehat{\mathcal{W}}_j^p| , \qquad 
	\widetilde{l}(p) \coloneqq \max \big\{l\geq 0 \mid 2^{l}\leq Q(p) \big\}
	\qquad \text{and} \qquad
	\widetilde{Q}(p) \coloneqq 2^{\widetilde{l}(p)}.
$$
%
Without loss of generality, we assume 
\begin{equation}
	\widehat{\mathcal{W}}_j^p = \big\{1,2,\ldots, Q(p) \big\} \qquad  \text{and} \qquad e_{1}\leq e_{2}\leq \dots \leq e_{Q(p)}.
	\label{eq:hat-W-ordering}
\end{equation}
In the sequel, we shall often abbreviate $Q(p)$, $\widetilde{l}(p)$ and $\widetilde{Q}(p)$ as $Q$, $\widetilde{l}$ and $\widetilde{Q}$, respectively, 
as long as it is clear from the context. 

%
%

\medskip 
\noindent 
{\em Step b): constructing $\mathcal{K}_{l}$ and $\widetilde{\mathcal{W}}_j^p(l)$.}  
Let us take $e_0 = t_{\mathrm{mid}}^p$, and employ $[e_0, e_k]\oplus a$ as a shorthand notation for $[e_{a},e_{k+a}]$. 
We can then define a group of disjoint intervals of $[T]$ as follows:  
\begin{subequations}
	\label{eq:construction-K1-Kl}
\begin{align}
	\mathcal{K}_1 &= \left\{  [ e_0, e_{2^{\widetilde{l}-1}} ]   \right\} ;
	\\  \mathcal{K}_2 &= \left\{   [e_0, e_{2^{\widetilde{l}-2}}] ,  [e_0, e_{2^{\widetilde{l}-2}}]  \oplus 2^{\widetilde{l}-1}  \right\};
	\\  \mathcal{K}_3 &= \left\{ [e_0 , e_{2^{\widetilde{l}-3}}], [e_0 , e_{2^{\widetilde{l}-3}}]\oplus 2^{\widetilde{l}-2},[e_0 , e_{2^{\widetilde{l}-3}}]\oplus 2\cdot 2^{\widetilde{l}-2},[e_0 , e_{2^{\widetilde{l}-3}}]\oplus 3\cdot 2^{\widetilde{l}-2}  \right\};
	\\  &\dots\nonumber
	\\  \mathcal{K}_{l} &= \left\{ [e_0, e_{2^{\widetilde{l}-l}}], [e_0,e_{2^{\widetilde{l}-l}}] \oplus 2^{\widetilde{l}-l+1}, [e_0, e_{2^{\widetilde{l}-l}}]\oplus 2\cdot 2^{\widetilde{l}-l+1}, \dots, [e_0,e_{2^{\widetilde{l}-l}}]\oplus (2^{l-1}-1) 2^{\widetilde{l}-l+1} \right\};
 \\ & \dots     \nonumber
	\\  \mathcal{K}_{\widetilde{l}} &=  \left\{ [e_{2i},e_{2i+1}]  \mid i = 0,1,2,\dots, 2^{\widetilde{l}-1}-1\right\};
	\\  \mathcal{K}_{\widetilde{l}+1} &= \left\{ [e_{2i+1},e_{2i+2}] \mid i = 0,1,2,\dots, 2^{\widetilde{l}-1}-1    \right\}. 
\end{align}
\end{subequations}
For each $i \in [\widetilde{Q}-1]$ with binary form $\{i_{l}\}_{l=1}^{\widetilde{l}}$ and $0 \leq l \leq \widetilde{l}$, we define $\mathsf{trunc}(i,l)$ to be the number with binary form $\{ i_{1},i_{2},\ldots, i_{l}, 0,0, \ldots, 0 \}$. For example, $\mathsf{trunc}(i,0)=0$ and $\mathsf{trunc}(i,\widetilde{l})=i$. 

From the definition \eqref{eq:defn-widehat-W-jp} of $\widehat{\mathcal{W}}_j^p$, 
we know that for each $i \in [\widetilde{Q}-1]$,  
 \begin{align}
	 \frac{x}{2}\leq  \log\left( \frac{w_i^{e_i}}{w_i^{e_0}}\right)=\sum_{l=1}^{\widetilde{l}}\log\left(  \frac{w_{i}^{e_{\mathsf{trunc}(i,l)}}}{w_i^{e_{\mathsf{trunc}(i,l-1)}}}\right) = \sum_{l=1}^{\widetilde{l}}\log\left(  \frac{w_{i}^{e_{\mathsf{trunc}(i,l)}}}{w_i^{e_{\mathsf{trunc}(i,l-1)}}}\right) \mathds{1}\left\{ e_{\mathsf{trunc}(i,l)}\neq e_{\mathsf{trunc}(i,l-1)} \right\}, 
 \end{align}
which in turn implies that
\begin{align}
\max_{1\leq l\leq \widetilde{l}}\log\left(  \frac{w_{i}^{e_{\mathsf{trunc}(i,l)}}}{w_i^{e_{\mathsf{trunc}(i,l-1)}}}\right) \geq \frac{x}{2\widetilde{l}}.
\end{align}
By defining 
$$\widetilde{\mathcal{W}}_j^p(l) \coloneqq \left\{i\in \widehat{\mathcal{W}}_j^p: \arg\max_{1\leq l'\leq \widetilde{l}} \log\left(  \frac{w_{i}^{e_{\mathsf{trunc}(i,l')}}}{w_i^{e_{\mathsf{trunc}(i,l'-1)}}}\right) = l \right\}$$
for each\footnote{Without loss of generality, we assume the $\arg\max$ function is a single-valued function.} $1 \leq l \leq \widetilde{l}$, we can demonstrate that
\begin{align}
\sum_{l=1}^{\widetilde{l}}\Big|\widetilde{\mathcal{W}}_j^p(l)\Big| \geq \widetilde{Q}-1, \label{eq:22}
\end{align}
thus implying the existence of some $1\leq l^{\star}\leq \widetilde{l}$ obeying 
\begin{align}
 	\Big|\widetilde{W}_j^p(l^{\star}) \Big|\geq \frac{\widetilde{Q}-1}{\widetilde{l}}\geq \frac{\widetilde{Q}}{2\widetilde{l}}.
	\label{eq:defn-widetilde-W-jp-lstar}
\end{align}

\medskip 
\noindent 
{\em Step c): constructing $\widetilde{\mathcal{W}}_j^p(l,o)$, $\widehat{s}(p,o)$ and $\widehat{e}(p,o)$.}  
 By definition, for any $i$, if $\mathsf{trunc}(i,l^{\star})\neq \mathsf{trunc}(i,l^{\star}-1)$, then one has 
 $$[e_{\mathsf{trunc}(i,l^{\star}-1)}, e_{\mathsf{trunc}(i,l^{\star})}]\in \mathcal{K}_{l^{\star}},$$
 where the set $\mathcal{K}_l$ has been defined in \eqref{eq:construction-K1-Kl}. 
 In addition, from the construction of $\widetilde{\mathcal{W}}_j^p(l^{\star})$ (see \eqref{eq:defn-widetilde-W-jp-lstar}), we know that $\mathsf{trunc}(i,l^{\star})\neq \mathsf{trunc}(i,l^{\star}-1)$ for any $i \in \widetilde{\mathcal{W}}_j^p(l^{\star})$. 
 For each $1 \leq o \leq 2^{l^{\star}-1}$, define 
 \begin{align}
\widetilde{\mathcal{W}}_j^p(l^{\star},o) \coloneqq \left\{i\in \widetilde{\mathcal{W}}_j^p(l^{\star}) \mid [e_{\mathsf{trunc}(i,l^{\star}-1)}, e_{\mathsf{trunc}(i,l^{\star})}] = [e_0 ,e_{2^{\widetilde{l}-l^{\star}}}] \oplus (o-1)2^{\widetilde{l}-l^{\star}+1} \right\},\label{eq:23}
 \end{align}
 %
%
where we employ the notation $l^{\star}$ and $\widetilde{l}$  to abbreviate $l^{\star}(p)$ and $\widetilde{l}(p)$, respectively.

In addition, for any $1\leq p\leq P$ and $1\leq o\leq 2^{l^{\star}(p)-1}$, we set
\begin{subequations}
	\label{eq:construction-spo-epo}
\begin{align}
	\widehat{s}(p,o) &=  e_{(o-1)2^{\widetilde{l}(p)-l^{\star}(p)+1} } , \\
	\widehat{e}(p,o) &=e_{2^{\widetilde{l}(p)-l^{\star}(p) } + (o-1)2^{\widetilde{l}(p)-l^{\star}(p)+1 } }.
\end{align}
\end{subequations}
In words, $[\widehat{s}(p,o),\widehat{e}(p,o)]$ can be understood as the $o$-th interval in the set $\mathcal{K}_{l^{\star}(p)}$.

\medskip 
\noindent 
{\em Step d): construction output.} 
With the above construction in mind, 
we would like to select 
$$
	\left\{ \big\{\widetilde{\mathcal{W}}_j^p \big(l^{\star}(p),o\big)\big\}_{o=1}^{2^{l^{\star}(p)-1}}\right\}_{p=1}^P
	\qquad \text{with intervals} \qquad
	\left\{ \big\{\widehat{s}(p,o), \widehat{e}(p,o) \big\}_{o=1}^{2^{l^{\star}(p)-1}}\right\}_{p=1}^P
$$
as the group of subsets we construct. With slight abuse of notation, we use $(p,o)$ as the index of the segments instead of $n$.
In what follows, we validate this construction.

\paragraph{Verification of the advertised properties.} 
We now proceed to justify the claimed properties. 

\paragraph{Property (\romannumeral1).} By construction, it is clearly seen that 
$$
\widetilde{\mathcal{W}}_{j}^{p}\big(l^{\star}(p),o\big)\subseteq\widetilde{\mathcal{W}}_{j}^{p}\big(l^{\star}(p)\big)\subseteq\widehat{\mathcal{W}}_{j}^{p}\subseteq\mathcal{W}_{j}^{p}\subseteq\mathcal{W}_{j}.	
$$
In addition, if 
$$
	\widetilde{W}_j^{p_1}\big(l^{\star}(p_1),o_1 \big)\cap \widetilde{W}_j^{p_2}\big(l^{\star}(p_2),o_2 \big)\neq \emptyset, 
$$ 
then one has $\mathcal{W}_j^{p_2} \cap \mathcal{W}_j^{p_2}\neq \emptyset$, and as a result, $p_1 = p_2$ (otherwise it violates the condition that $\mathcal{W}_j^{p_2} \cap \mathcal{W}_j^{p_2}= \emptyset$ for $p_1\neq p_2$). 
 It also follows from the definition in \eqref{eq:23} that $o_1 = o_2$. 
 Therefore, for any $(p_1,o_1)$ that does not equal $(p_2,o_2)$, we have 
 $\widetilde{W}_j^{p_1}\big(l^{\star}(p_1),o_1 \big)\cap \widetilde{W}_j^{p_2}\big(l^{\star}(p_2),o_2 \big)= \emptyset$.

\paragraph{Property (\romannumeral2).} 
By construction, we have
\begin{align}
\sum_{o=1}^{2^{l^{\star}(p)-1}}\Big|\widetilde{\mathcal{W}}_{j}^{p}\big(l^{\star}(p),o\big)\Big|=\Big|\widetilde{\mathcal{W}}_{j}^{p}\big(l^{\star}(p)\big)\Big|\geq\frac{|\widehat{\mathcal{W}}_{j}^{p}|}{4\log_{2}\big(|\widehat{\mathcal{W}}_{j}^{p}|\big)}\geq\frac{|\mathcal{W}_{j}^{p}|}{8\log_{2}\big(|\widehat{\mathcal{W}}_{j}^{p}|\big)},
 \end{align}
where we have made use of \eqref{eq:defn-widetilde-W-jp-lstar} and \eqref{eq:size-Wjp-179}. 
Summing over $p$ and applying Lemma~\ref{lemma:part2} yield 
 \begin{align}
\sum_{p=1}^{P}\sum_{o=1}^{2^{l^{\star}(p)-1}}\Big|\widetilde{\mathcal{W}}_{j}^{p}\big(l^{\star}(p),o\big)\Big|
	 \geq\sum_{p=1}^{P}\frac{|\mathcal{W}_{j}^{p}|}{8\log_{2}(k)}\geq\frac{|\mathcal{W}_{j}|}{24\log_{2}(k)\big(\log_{2}(T)+1\big)}.
 \end{align}

\paragraph{Property (\romannumeral3)(a).} 
Let us set the parameters $\left\{\{g(p,o)\}_{o=1}^{2^{l^{\star}(p)}}\right\}_{p=1}^P$ as follows: 
$$
	g(p,o ) = \frac{\sum_{i\in \widetilde{\mathcal{W}}_j(l^{\star}(p),o)}w_i^{\widehat{s}(p,o)} }{2^{-(j+2)}\cdot \big|\widetilde{\mathcal{W}}_j^p\big(l^{\star}(p),o\big)\big| }\geq 1, 
$$
where the last inequality holds since, by construction, $ w_i^{ \widehat{s}(p,o)} \geq 2^{-(j+2)}$ (see Lemma~\ref{lemma:aux1}).
Then Property (\romannumeral3)(a) is satisfied since
\begin{align}
\sum_{i \in \widetilde{\mathcal{W}}_j(l^{\star}(p),o)} w_i^{\widehat{s}(p,o)} = \frac{g(p,o) \big|\widetilde{\mathcal{W}}_j(l^{\star}(p),o) \big|}{2^{j+2}}.\nonumber
\end{align}

\paragraph{Property (\romannumeral3)(b).} 
For any $i \in \widehat{\mathcal{W}}_j^p \subseteq \mathcal{W}_j^p $, we have 
$$
	s_i\leq e_{\mathsf{trunc}(i,l-1)}\leq e_i \qquad \text{for any }1\leq l\leq \widetilde{l}(p),
$$
which is valid since $\max_{i\in \mathcal{W}_j^p}s_i \leq \min_{i\in \mathcal{W}_j^p} e_i$ 
(see Lemma~\ref{lemma:part2}) and \eqref{eq:hat-W-ordering}. 
It then holds that $$s_i \leq \widehat{s}(p,o)\leq e_i \qquad \text{for any }i \in \widehat{\mathcal{W}}^p_j.$$ 
Also, the construction of $(s_i,e_i)$ (see Lemma~\ref{lemma:aux1}) tells us that $ w_i^{ \widehat{s}(p,o)} \geq 2^{-(j+2)}$. 

Moreover, by construction, we know that for any $i \in \widetilde{\mathcal{W}}_j^{p}\big(l^{\star}(p),o\big)$, 
\begin{align}
	\log\left(  \frac{w_i^{\widehat{e}(p,o)} }{w_i^{\widehat{s}(p,o)}}\right) &\geq \frac{x}{2\widetilde{l}(p)} 
	\qquad \text{and} \qquad
	 w_i^{ \widehat{s}(p,o)} \geq 2^{-(j+2)}.\nonumber
\end{align}
Recalling that $x = \log(2)$, one can further derive
%
\begin{align}
  \sum_{i\in \widetilde{\mathcal{W}}_j^p(l^{\star}(p),o) } w^{\widehat{s}(p,o)}_i 
	&\geq 2^{-(j+2)}\cdot \big|\widetilde{\mathcal{W}}_j^p \big(l^{\star}(p),o \big) \big| \nonumber
	\\  \log \left(   \frac{\sum_{i\in \widetilde{\mathcal{W}}_j^p(l^{\star}(p),o) } w^{\widehat{e}(p,o)}_i}{\sum_{i\in \widetilde{\mathcal{W}}_j^p(l^{\star}(p),o) } w^{\widehat{s}(p,o)}_i }\right) &\geq \log\left( \frac{\sum_{i\in \widetilde{\mathcal{W}}_j^p(l^{\star}(p),o) } w^{\widehat{s}(p,o)}_i\cdot \exp\left(\frac{x}{2\widetilde{l}(p)}\right) }{\sum_{i\in \widetilde{\mathcal{W}}_j^p(l^{\star}(p),o) } w^{\widehat{s}(p,o)}_i }\right)= \frac{x}{2\widetilde{l}(p)}\geq \frac{\log(2)}{2\log_2(k)}.\nonumber
\end{align}

\paragraph{Property (\romannumeral3)(c).} 
Note that for any $t$ obeying $\widehat{s}(p,o)\leq t\leq \widehat{e}(p,o)$, and any $i\in \widetilde{\mathcal{W}}_j^p\big(l^{\star}(p),o\big)$, it holds that $s_i\leq \widehat{s}(p,o) \leq t \leq \widehat{e}(p,o) \leq e_i$. 
Recall that $ w_i^{ t} \geq 2^{-(j+2)}$ for any $t\in [s_i,e_i]$ (see Lemma~\ref{lemma:aux1}). 
As a result,  
\begin{align}
\sum_{i\in \widetilde{\mathcal{W}}_j^p(l^{\star}(p),o) } w^{t}_i\geq \big| \widetilde{\mathcal{W}}_j^p\big(l^{\star}(p),o\big)\big|
	\cdot 2^{-(j+2)}.\nonumber
\end{align}

\paragraph{Proper ordering.} 
To finish up, it remains to verify that 
the intersection of $[\widehat{s}(p_1,o_1),\widehat{e}(p_1,o_1)]$ and $[\widehat{s}(p_2,o_2), \widehat{e}(p_2,o_2)]$ 
is either empty or contains only the boundary points, unless $(p_1,o_1)=(p_2,o_2)$.  
To show this, note that in the case where $p_1\neq p_2$ (assuming $p_1<p_2$), we have
$$
	\widetilde{s}_{p_1}\leq \widehat{s}(p_1,o_1)<\widehat{e}(p_1,o_1)\leq \widetilde{e}_{p_1}\leq \widetilde{s}_{p_2}\leq \widehat{s}(p_2, o_2)< \widehat{e}(p_2,o_2),
$$
which arises from Lemma~\ref{lemma:part2}. 
Also, in the case where $p_1 = p_2=p$ and $o_1 \neq o_2$ (assuming $o_1<o_2$), we have 
$$
	\widehat{s}(p,o_1)<\widehat{e}(p,o_1) < \widehat{s}(p,o_2) < \widehat{e}(p,o_2) ,
$$
which comes from the construction \eqref{eq:construction-spo-epo}.

\bigskip \noindent 
We have thus completed the proof of this lemma.
%

\subsection{Proof of Lemma~\ref{lemma:cut}}\label{sec:pfcut}

Throughout this proof, we find it convenient to denote $Z^t = \sum_{i=1}^k W_i^t$.

\paragraph{Part 1.}
We start by proving the first claim \eqref{eq:segment-length-t2-t1}. 
Recall that $[t_1,t_2]$ is assumed to be a $(p,q,x)$-segment. 
From the definition of the segment (see Definition~\ref{def:seg}), there exists $i\in [k]$ such that
\begin{align}
\log\left(  \frac{w^{t_2}_i}{w^{t_1}_i}\right) \geq x. \nonumber
\end{align}
Given that $W^{t_2}_i=W^{t_1}_i\exp\big(\eta\sum_{\tau = t_1}^{t_2 - 1} \widehat{r}_i^{\tau } \big)$ and $w_t=W_t/Z_t$ (see lines~\ref{line:updateS11} and \ref{line:wt-update-Wt} of Algorithm~\ref{alg:main1}), the above inequality can be equivalently expressed as 
\begin{align}
 \eta \sum_{\tau = t_1}^{t_2 - 1} \widehat{r}_i^{\tau } -\log(Z^{t_2}/Z^{t_1})\geq x. 
	\label{eq:w-t2-t1-gap-equiv}
\end{align}
Moreover, recognizing that
\begin{align*}
\log(Z^{t_{2}}/Z^{t_{1}}) & =\log\left(\frac{\sum_{i\in[k]}W_{i}^{t_{1}}\exp\big(\eta\sum_{\tau=t_{1}}^{t_{2}-1}\widehat{r}_{i}^{\tau}\big)}{\sum_{i\in[k]}W_{i}^{t_{1}}}\right)\geq-\eta(t_{2}-t_{1})
\end{align*}
and $\widehat{r}_i^{\tau}\leq 1$ for any $1\leq \tau \leq T$, we can invoke \eqref{eq:w-t2-t1-gap-equiv} to show that
\begin{align}
x\leq 2(t_2-t_1)\eta, 
\end{align}
from which the claimed inequality \eqref{eq:segment-length-t2-t1} follows.

\paragraph{Part 2.}
%
%
We now turn to the remaining claims of Lemma~\ref{lemma:cut}. 
For each hypothesis $h\in \mathcal{H}$, let us introduce the following vector $v_h\in \mathbb{R}^k$:
\begin{align}
	v_h=[v_{h,i}]_{i\in [k]} \qquad \text{with}~~ v_{h,i} = L(h, \basis_i) - \mathsf{OPT}.
\end{align}
Given the $\varepsilon_1$-optimality of $h^t$ (see Lemma~\ref{lemma:main}), 
we have the following property that holds for any $1\leq \tau,t \leq T$:  
\begin{align}
	\langle v_{h^{\tau}}, w^{t} \rangle \geq 
	\min_{h\in \mathcal{H}}\langle v_{h}, w^{t} \rangle \geq
	\langle v_{h^t},  w^{t} \rangle -\varepsilon_1= v^t - \varepsilon_1,\label{eq:s2}
\end{align}
where we recall the definition of $v^t$ in \eqref{eq:defn-vt-gap-restated}. 
%
%
In the sequel, we divide the proof into a couple of steps.

\paragraph{Step 1: decomposing the KL divergence between $w^{t}$ and $w^{t_2}$.} 
Let us write
$$
	W^t_i = \exp\left(\eta \sum_{\tau = 1}^t  \widehat{r}_i^{\tau} \right)
	= \exp\left(\eta \sum_{\tau = 1}^t \big( v_{h^{\tau},i} + \mathsf{OPT}+ \xi^{\tau}_{i}\big) \right)
	\qquad \text{with }\xi_i^{\tau}= \widehat{r}_i^{\tau}- v_{h^{\tau},i}-\mathsf{OPT}, 
$$
where $\xi_i^{\tau} = \widehat{r}_i^{\tau} - L(h^{\tau},\basis_i)$ is clearly a zero-mean random variable. 
Define
$$
	\Delta_{t_1,t_2}=\sum_{\tau = t_1}^{t_2-1}\xi^{\tau}.
$$
Taking $W^t=[W^t_i]_{i\in [k]}$ and denoting by $\log(x/y)$ the vector $\{\log(x_i/y_i)\}_{i\in [k]}$ for two $k$-dimensional vectors $(x,y)$, 
one can then deduce that
\begin{align}
\left\langle \frac{1}{\eta}\log\left(\frac{W^{t_{2}}}{W^{t_{1}}}\right)-\Delta_{t_{1},t_{2}}, \,w^{t}\right\rangle -(t_{2}-t_{1})\mathsf{OPT}=\sum_{\tau=t_{1}}^{t_{2}-1}\langle v_{h^{\tau}},w^{t}\rangle
	\geq(t_{2}-t_{1})(v^{t}-\varepsilon_{1}),
	\label{eq:log-Wt2-Wt1-LB-123}
\end{align}
where the last inequality results from \eqref{eq:s2}.

Recall that $Z^t = \sum_{i=1}^k W^t_i$ and $w^t_i = \frac{W^t_i}{Z^t}$. 
By taking $t_1=t$, we can derive from \eqref{eq:log-Wt2-Wt1-LB-123} that
\begin{align}
\left\langle \log\left(\frac{w^{t_{2}}}{w^{t}}\right)-\eta\Delta_{t,t_{2}},w^{t}\right\rangle +\log\left(\frac{Z^{t_{2}}}{Z^{t}}\right)-\eta(t_{2}-t)\mathsf{OPT}\geq\eta(t_{2}-t)(v^{t}-\varepsilon_{1}).
\end{align}
As it turns out, this inequality allows us to bound the KL divergence between $w^t$ and $w^{t_2}$ as follows: 
\begin{align}
\mathsf{KL}(w^{t}\,\|\,w^{t_{2}}) & \coloneqq\left\langle w^{t},\log\left(\frac{w^{t}}{w^{t_{2}}}\right)\right\rangle \nonumber\\
 & \leq\log(Z^{t_{2}}/Z^{t})-\eta(t_{2}-t)\mathsf{OPT}-\eta\langle w^{t},\Delta_{t,t_{2}}\rangle+\eta(t_{2}-t)(\varepsilon_{1}-v^{t}).
	\label{eq:s3}
\end{align}
In what follows, we shall cope with the right-hand side of \eqref{eq:s3}.

\paragraph{Step 2: bounding the term $\log(Z^{t_2}/Z^{t})$.} 
%
With probability exceeding $1-2T^2k\delta'$, it holds that
\begin{align}
\log(Z^{t_{2}}/Z^{t}) & =\sum_{\tau=t}^{t_{2}-1}\log(Z^{\tau+1}/Z^{\tau})=\sum_{\tau=t}^{t_{2}-1}\log\left(\sum_{i\in[k]}\frac{W_{i}^{\tau}\exp(\eta\widehat{r}_{i}^{\tau})}{\sum_{j\in[k]}W_{i}^{\tau}}\right)\nonumber\\
 & \overset{\mathrm{(i)}}{=}\sum_{\tau=t}^{t_{2}-1}\log\left(\sum_{i=1}^{k}w_{i}^{\tau}\exp(\eta\widehat{r}_{i}^{\tau})\right)\overset{\mathrm{(ii)}}{\leq}\sum_{\tau=t}^{t_{2}-1}\log\left(\sum_{i=1}^{k}w_{i}^{\tau}+\sum_{i=1}^{k}w_{i}^{\tau}(\eta\widehat{r}_{i}^{\tau})+2\sum_{i=1}^{k}w_{i}^{\tau}\eta^{2}(\widehat{r}_{i}^{\tau})^{2}\right)\nonumber\\
	& \overset{\mathrm{(iii)}}{\leq}\sum_{\tau=t}^{t_{2}-1}\log\left(1+\eta\sum_{i=1}^{k}w_{i}^{\tau}\widehat{r}_{i}^{\tau}+2\eta^{2}\right)\leq\sum_{\tau=t}^{t_{2}-1}\left(\eta\sum_{i=1}^{k}w_{i}^{\tau}\widehat{r}_{i}^{\tau}+2\eta^{2}\right)\nonumber\\
 & \overset{\mathrm{(iv)}}{=}\eta\sum_{\tau=t}^{t_{2}-1}v^{\tau}+\eta(t_{2}-t)\mathsf{OPT}+\eta\sum_{\tau=t}^{t_{2}-1}\sum_{i=1}^{k}\big\langle w_{i}^{\tau},\widehat{r}_{i}^{\tau}-v_{h^{\tau},i}-\mathsf{OPT}\big\rangle+2(t_{2}-t)\eta^{2}\nonumber\\
 & \overset{\mathrm{(v)}}{\leq}\eta(t_{2}-t)\varepsilon_{1}+\eta(t_{2}-t)\mathsf{OPT}+\eta\sum_{\tau=t}^{t_{2}-1}\sum_{i=1}^{k}\big\langle w_{i}^{\tau},\widehat{r}_{i}^{\tau}-v_{h^{\tau},i}-\mathsf{OPT}\big\rangle+2(t_{2}-t)\eta^{2}.
	\label{eq:s4}
\end{align}
Here, (i) comes from line~\ref{line:wt-update-Wt} of Algorithm~\ref{alg:main1}, 
(ii) follows from the elementary inequality $\exp(x)\leq1+x+2x^{2}$ for any $x\leq1$, 
(iii) is valid since $\sum_i w_i^{\tau}=1$ and $|\widehat{r}_i^{\tau}|\leq 1$, 
(iv) holds due to the fact that $v^t = \langle w^t, v_{h^t}\rangle $, and (v) arises from the fact that 
$v^{\tau}\leq \varepsilon_1$ (see \eqref{eq:vt-UB-135}).

\paragraph{Step 3: bounding the weighted sum of $\{\xi_{i}^{\tau}\}$.}
Next, we intend to control the two random terms below: 
\begin{subequations}
	\label{eq:xx21}
\begin{align}
\eta\sum_{\tau=t}^{t_{2}-1}\sum_{i=1}^{k}\big\langle w_{i}^{\tau},\widehat{r}_{i}^{\tau}-v_{h^{\tau},i}-\mathsf{OPT}\big\rangle & =\eta\sum_{\tau=t}^{t_{2}-1}\sum_{i=1}^{k}w_{i}^{\tau}\xi_{i}^{\tau},\\
\eta\big\langle w^{t},\Delta_{t,t_{2}}\big\rangle & =\eta\sum_{\tau=t}^{t_{2}-1}\sum_{i=1}^{k}w_{i}^{t}\xi_{i}^{\tau}.
\end{align}
\end{subequations}
Let $\mathcal{F}^{\tau}$  denote what happens before the $\tau$-th round in Algorithm~\ref{alg:main1}. 
Two properties are worth noting. 
\begin{itemize}
	\item The variance of $\xi_i^{\tau}$ is at most $O\big( \frac{1}{k\overline{w}^{\tau}_i} \big)$, 
		according to the update rule (see line~\ref{line:updateS} in Algorithm~\ref{alg:main1}); 
	\item $\{\xi_i^{\tau}\}_{i\in [k]}$ are independent conditioned on $\mathcal{F}^{\tau}$.
\end{itemize}
Let us develop bounds on the two quantities in \eqref{eq:xx21} below. 
\begin{itemize}
	\item 
Letting  $q^{\tau} = \sum_{i=1}^k w_i^{t}\xi_i^{\tau}$, one sees that 
	\begin{align}
		|q^{\tau}|\leq 1, \qquad \mathbb{E}[q^{\tau}| \mathcal{F}^{\tau}]=0 \qquad \text{and} \qquad 
	\mathsf{Var}[q^{\tau}|\mathcal{F}^{\tau}] \leq \sum_{i=1}^k\frac{(w_i^{t})^2}{k\overline{w}_i^{\tau}} \leq \sum_{i=1}^k \frac{w_i^t}{k} = \frac{1}{k}.
	\end{align}
By virtue of Freedman's inequality (cf.~Lemma~\ref{freedman}), with probability at least $1-\delta'$ one has
\begin{align}
\left|\sum_{\tau=t}^{t_2-1}\sum_{i=1}^k w_i^t \xi_i^{\tau} \right| 
 \leq 2\sqrt{\frac{t_2-t}{k}\log(2/\delta') }+2\log(2/\delta');\label{eq:xx2}
\end{align}
 	\item
Regarding the other term, by letting $\widehat{q}^{\tau} = \sum_{i=1}^{k}w_i^{\tau}\xi_i^{\tau}$, we have 
$$
		|\widehat{q}^{\tau}|\leq 1, \qquad
		\mathbb{E}[\widehat{q}^{\tau}|\mathcal{F}^{\tau}]=0
		\qquad\text{and} \qquad
		\mathsf{Var}[\widehat{q}^{\tau}|\mathcal{F}^{\tau}]\leq \sum_{i=1}^k \frac{(w_i^{\tau})^2}{k\overline{w}_i^{\tau}} \leq \sum_{i=1}^k \frac{w_i^{\tau}}{k} = \frac{1}{k}.
$$
Invoke Freedman's inequality (cf.~Lemma~\ref{freedman}) once again to show that, with probability exceeding $1-\delta'$, 
\begin{align} \left|\sum_{\tau=t}^{t_2-1}\sum_{i=1}^k w_i^{\tau}\xi_i^{\tau} \right| \leq  2\sqrt{\frac{t_2-t}{k}\log(2/\delta') }+2\log(2/\delta').\label{eq:xx3}
\end{align}

\end{itemize}

\paragraph{Step 4: bounding the KL divergence between $w^{t}$ and $w^{t_2}$.} 
Combining \eqref{eq:s3}, \eqref{eq:s4}, \eqref{eq:xx2} and \eqref{eq:xx3}, and applying the union bound over $(t,t_2)$, 
we can demonstrate that with probability at least $1-6T^4k\delta'$, 
\begin{align}
\mathsf{KL}(w^t\,\|\, w^{t_2}) & \leq 2(t_2-t)\eta\varepsilon_1 - (t_2-t)\eta v^t  \nonumber
\\ & \quad \quad \quad + 4\eta\sqrt{\frac{(t_2-t)\log(2/\delta')}{k}}+2(t_2-t)\eta^2+4\eta\log(2/\delta')
	\label{eq:s6}
\end{align}
holds for any $1\leq t<t_2\leq T$.
The analysis below then operates under the condition that $\eqref{eq:s6}$ holds for any $1\leq t< t_2\leq T$.

\paragraph{Step 5: connecting the KL divergence with the advertised properties.}  
Set
\begin{subequations}
	\label{eq:defn-tauj-Lemma17}
\begin{align}
	\tau_{\widehat{j}} &\coloneqq \min\{\tau \mid t_1\leq \tau \leq t_2-1, \,  -v^{\tau}\leq 2^{-(\widehat{j}-1)}  \},
	\qquad 1\leq \widehat{j} \leq j_{\mathrm{max}}
	; \\
	\tau_{j_{\mathrm{max}+1}}& \coloneqq t_2. 
\end{align}
\end{subequations}
By definition, we know that $\tau_1 = t_1$ and $\tau_{j_2}\geq \tau_{j_1}$ for $j_2\geq j_1$. 
Let $\mathcal{I}$ be the index set associated with this segment $[t_1,t_2]$, and 
set $y_{j} \coloneqq \sum_{i\in \mathcal{I}}w_i^{\tau_j}$. 
We then claim that there exists $1\leq \widetilde{j} \leq j_{\mathrm{max}}$ such that
\begin{align}
  \log\left( \frac{y_{\widetilde{j}+1} }{y_{\widetilde{j}}} \right) &\geq \frac{x}{\log_2(1/\eta)+1}. \label{eq:ss1} 
\end{align}
%
%
\begin{proof}[Proof of \eqref{eq:ss1}]
	Suppose that none of $1\leq \widetilde{j}\leq j_{\mathrm{max}}$ satisfies \eqref{eq:ss1}.  Then for any $1\leq \widehat{j}\leq j_{\mathrm{max}}$, it holds that $\log\left(  \frac{y_{\widehat{j}+1}}{y_{\widehat{j}}}\right) < \frac{x}{\log_2(1/\eta)+1}$, which implies that $y_{\widehat{j}}> y_{\widehat{j}+1}\exp\left(-\frac{x}{\log_2(1/\eta)+1}\right)$. 
As a result, we have 
\begin{align*}
	y_{1} &> y_{j_{\mathrm{max}}+1}\cdot \exp\left( -j_{\mathrm{max}}\cdot \frac{x}{\log_2(1/\eta)+1}\right) 
	= \bigg( \sum_{i\in \mathcal{I}}w_i^{t_2} \bigg) \cdot \exp\left( -j_{\mathrm{max}}\cdot \frac{x}{\log_2(1/\eta)+1}\right) \\
	&\geq p\exp(x) \cdot \exp(-x) 
	= p,
\end{align*}
thus leading to contradiction (as according to the definition of the $(p,q,x)$-segment, one has $y_1\leq p$).
\end{proof}

Now, assume that $\widetilde{j}$ satisfies \eqref{eq:ss1}. 
From the definition of the $(p,q,x)$-segment, we have $y_{\widetilde{j}}\geq q$.
%
%
%
It follows from \eqref{eq:s6} that
\begin{align}
	\mathsf{KL}(w^{\tau_{\widetilde{j}}}\,\|\, w^{\tau_{\widetilde{j}+1}}) & \leq 2 (\tau_{\widetilde{j}+1}-\tau_{\widetilde{j}})\eta \varepsilon_1 + (\tau_{{\widetilde{j}}+1}-\tau_{\widetilde{j}} )\eta 2^{-(\widetilde{j}-1)} \nonumber
\\ & \quad + 4\eta \sqrt{\frac{(\tau_{\widetilde{j}+1}-\tau_{\widetilde{j}})\log(2/\delta')}{k}}+2(\tau_{\widetilde{j}+1}-\tau_{\widetilde{j}})\eta^2 +4\eta\log(2/\delta').
\end{align}
Since $\log\left( \frac{y_{\widetilde{j}+1}}{y_{\widetilde{j}}}\right)\geq \frac{x}{\log_2(1/\eta)+1}$ and $y_{\widetilde{j}}\geq q$, 
we can invoke Lemma~\ref{lemma:klbound} and Lemma~\ref{lemma:klcmp} to show that 
$$\mathsf{KL}(w^{\tau_{\widetilde{j}}} \,\|\,w^{\tau_{\widetilde{j}+1}})\geq 
\mathsf{KL}\left(\mathsf{Ber}\big(y_{\widetilde{j}}\big)\,\|\,\mathsf{Ber}\big(y_{\widetilde{j}+1}\big)\right) 
\geq \frac{qx^2}{4\big(\log_2(1/\eta)+1\big)^2},$$ 
where $\mathsf{Ber}(x)$ denote the Bernoulli distribution with  mean $x\in [0,1]$. 
As a result, we can obtain 
\begin{align}
	\frac{qx^2 }{4\big(\log_2(1/\eta)+1 \big)^2} & \leq 2 (\tau_{\widetilde{j}+1}-\tau_{\widetilde{j}})\eta \varepsilon_1 + (\tau_{{\widetilde{j}}+1}-\tau_{\widetilde{j}} )\eta 2^{-({\widetilde{j}}-1)} \nonumber
\\ & \quad + 4\eta \sqrt{\frac{(\tau_{\widetilde{j}+1}-\tau_{\widetilde{j}})\log(2/\delta')}{k}}+2(\tau_{\widetilde{j}+1}-\tau_{\widetilde{j}})\eta^2 +4\eta\log(2/\delta'), 
\end{align}
which in turn results in
\begin{align}
\tau_{\widetilde{j}+1}-\tau_{\widetilde{j}} & \geq\min\left\{ \frac{qx^{2}}{100\big(\log_{2}(1/\eta)+1\big)^{2}}\min\left\{ \frac{1}{\eta\varepsilon_{1}},\frac{2^{\widetilde{j}-1}}{\eta},\frac{1}{\eta^{2}}\right\} ,\frac{kq^{2}x^{4}}{10000\eta^{2}\log(1/\delta')\big(\log_{2}(1/\eta)+1\big)^{4}}\right\} \nonumber\\
%
 & =\frac{qx^{2}}{100\big(\log_{2}(1/\eta)+1\big)^{2}}\min\left\{ \frac{2^{\widetilde{j}-1}}{\eta},\frac{1}{\eta^{2}}\right\}  
	\label{eq:xx411} \\
&= \frac{qx^{2}}{100\big(\log_{2}(1/\eta)+1\big)^{2}}\frac{2^{\widetilde{j}-1}}{\eta}. 
	\label{eq:xx4111}
\end{align}
Here, to see why \eqref{eq:xx411} and \eqref{eq:xx4111} hold, 
it suffices to note that 
\[
	\frac{qx^{2}}{100}\cdot\frac{2^{\widetilde{j}-1}}{\eta}
	\leq
	\frac{qx^{2}}{100}\cdot\frac{1}{\eta^2}
	\leq 
	\frac{kq^{2}x^{4}}{10000\eta^{2}\log(1/\delta')\big(\log_{2}(1/\eta)+1\big)^{2}},
\]
a property that arises from 
the fact that $2^{\widetilde{j}-1}\leq 1/\eta $ 
and the assumption that $\frac{qx^2}{50(\log_2(1/\eta)+1)^2}\geq  \frac{1}{k}$.

With \eqref{eq:xx4111} in mind, we are ready to finish the proof by dividing into two cases.
\begin{itemize}

	\item {\em Case 1: $\widetilde{j}= j_{\mathrm{max}}$}. It follows from \eqref{eq:xx4111} that
\begin{align*}
t_{2}-t_{1}\geq\tau_{\widetilde{j}+1}-\tau_{\widetilde{j}} & \geq\frac{qx^{2}}{100\big(\log_{2}(1/\eta)+1\big)^{2}}\frac{2^{j_{\mathrm{max}}-1}}{\eta}
	\geq\frac{qx^{2}}{200\big(\log_{2}(1/\eta)+1\big)^{2}\eta^2}
	.
\end{align*}

	\item {\em Case 2: $1\leq \widetilde{j}\leq j_{\mathrm{max}}-1$}. It comes from the definition \eqref{eq:defn-tauj-Lemma17} that 
$$
	\sum_{\tau=t_1}^{t_2-1}\mathds{1}\big\{-v^{\tau}\geq 2^{-\widehat{j}}\big\}
	\geq \sum_{\tau=t_1}^{t_2-1}\mathds{1}\big\{-v^{\tau}> 2^{-\widehat{j}}\big\}
	\geq \tau_{\widehat{j}+1}-t_1\geq \tau_{\widehat{j}+1}-\tau_{\widehat{j}}
	\qquad \text{for any }1\leq \widehat{j} \leq j_{\mathrm{max}}-1.
$$
When $1\leq \widetilde{j}\leq j_{\mathrm{max}}-1$, the above display taken collectively with \eqref{eq:xx4111} gives 
\begin{align}
	\sum_{\tau=t_1}^{t_2-1}\mathds{1}\big\{ -v^{\tau}\geq 2^{-\widetilde{j}} \big\}
	\geq  \tau_{\widetilde{j}+1}-\tau_{\widetilde{j}}\geq \frac{qx^2\cdot 2^{\widetilde{j}-1} }{100\big(\log_2(1/\eta)+1\big)^2\eta}
		,
\end{align}
and as a result, 
\begin{align}
\sum_{\tau=t_{1}}^{t_{2}-1}\sum_{\widehat{j}=1}^{j_{\mathrm{max}}-1}\mathds{1}\big\{\max\{-v^{\tau},0\}\geq2^{-\widehat{j}}\big\} & 2^{-(\widehat{j}-1)}\geq\sum_{\tau=t_{1}}^{t_{2}-1}\mathds{1}\big\{-v^{\tau}\geq2^{-\widetilde{j}}\big\}2^{-(\widetilde{j}-1)}\geq\frac{qx^{2}}{100\big(\log_{2}(1/\eta)+1\big)^{2}\eta}.
\label{eq:lls-123}
\end{align}
By observing that 
$$
	\sum_{\widehat{j}=1}^{\infty}\mathds{1}\big\{ x\geq 2^{-\widehat{j}}\big\} \cdot 2^{-(\widehat{j}-1)}\leq 4x 
$$
holds for any $x\geq 0$, we can combine this fact with \eqref{eq:lls-123} to derive
\begin{align}
	4\sum_{\tau=t_{1}}^{t_{2}-1}\max\{-v^{\tau},0\} & \geq\frac{qx^{2}}{100\big(\log_{2}(1/\eta))+1\big)^{2}\eta}.
	\label{eq:sum-max-vtau-proof}
\end{align}
Furthermore, recalling that $v^{\tau}\leq \varepsilon_1$ (cf.~\eqref{eq:vt-UB-135}), one can deduce that
\[
	\sum_{\tau=t_{1}}^{t_{2}-1}\max\{-v^{\tau},0\}=\sum_{\tau=t_{1}}^{t_{2}-1}(-v^{\tau})-\sum_{\tau=t_{1}}^{t_{2}-1}\min\{-v^{\tau},0\}\leq\sum_{\tau=t_{1}}^{t_{2}-1}(-v^{\tau})+(t_{2}-t_{1})\varepsilon_{1},
\]
which combined with \eqref{eq:sum-max-vtau-proof} yields
\[
	4(t_{2}-t_{1})\varepsilon_{1}+4\sum_{\tau=t_{1}}^{t_{2}-1}(-v^{\tau})\geq\frac{qx^{2}}{100\big(\log_{2}(1/\eta)+1\big)^{2}\eta}.
\]

\end{itemize}
The proof is thus complete.

\subsection{Additional illustrative figures for segment construction}
\label{sec:additional-figs}

In this section, we provide several illustrative figures  to help the readers  understand 
the concept of ``segments'' and certain key properties,  
which play an important role in the segment construction described in Section~\ref{sec:tec22} and the proof of Lemma~\ref{lemma:part3}. 


\begin{itemize}
\item

 Figure~\ref{fig1} illustrates an example in which any two segments either coincide or are disjoint. 
 In this case, our heuristic arguments in Section~\ref{sec:tec22} about ``shared intervals'' and ``disjoint intervals'' become applicable.

%

\item Figure~\ref{fig2} gives 
 an example where two non-identical segments might overlap. Due to the presence of non-disjoint segments, our heuristic arguments in Section~\ref{sec:tec22} about ``shared intervals'' and ``disjoint intervals'' are not readily applicable.


\item 
 In Figure~\ref{fig5}, we provide an example of the partition of blocks (as in the proof of Lemma~\ref{lemma:part2}), 
 whereas in Figure~\ref{fig6}, we illustrate how to align one side of the segments using a common inner point (as in the proof of Lemma~\ref{lemma:part3}).

\item
In Figure~\ref{fig7} and Figure~\ref{fig8},
we illustrate how to construct disjoint segments from a group of segments with common starting points in the case with $k=8$. 
In this particular example, we have in total $5$ groups of disjoint segments, marked with different colors. 

\end{itemize}

\begin{figure}[!h]
    \centering
{
	    \includegraphics[width=0.85\textwidth]{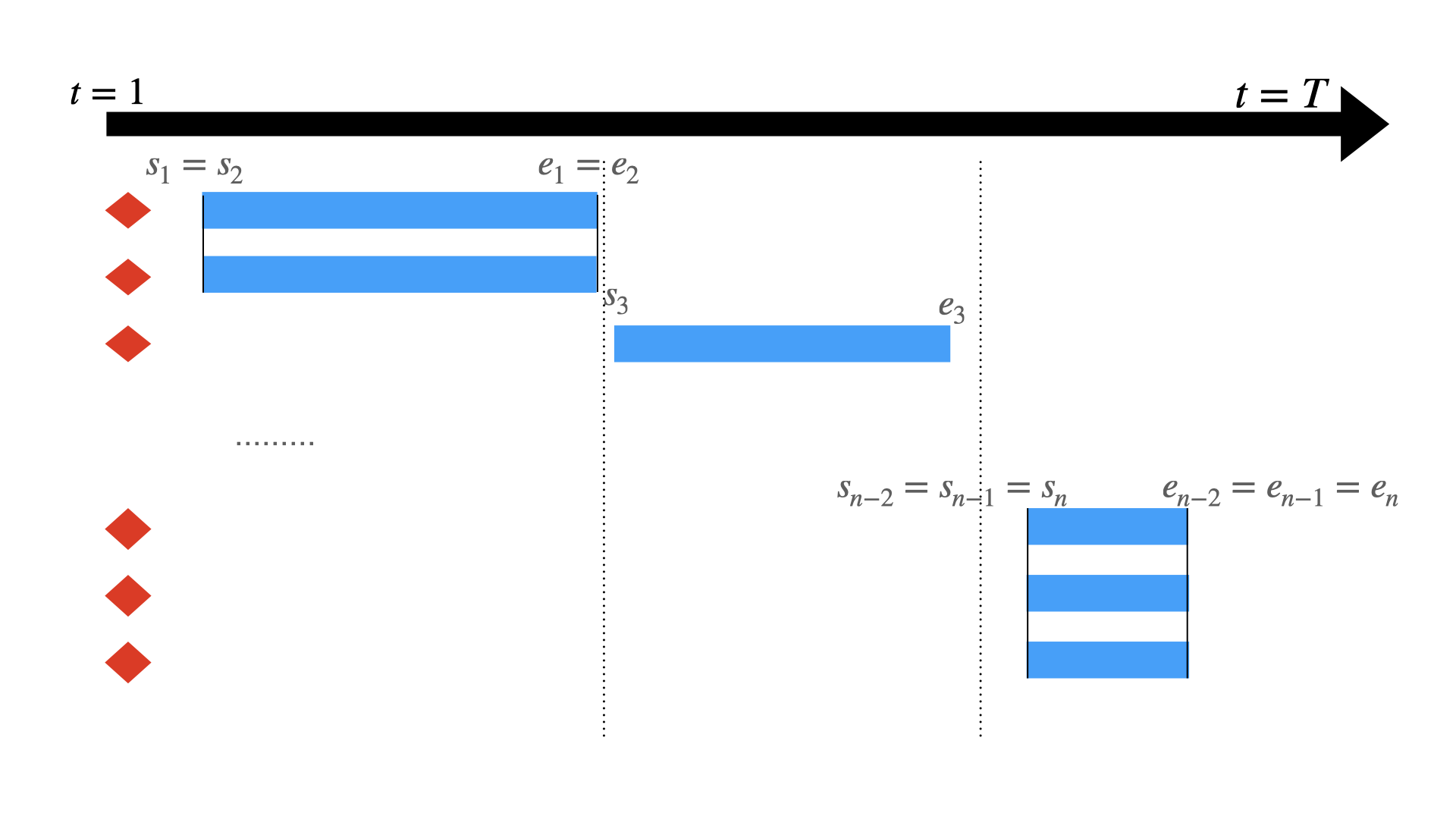}
     
    }
    \caption{An example where any two segments either coincide or are disjoint.}
   \label{fig1}
\end{figure}

\begin{figure}[!h]
    \centering

  {
	    \includegraphics[width=0.85\textwidth]{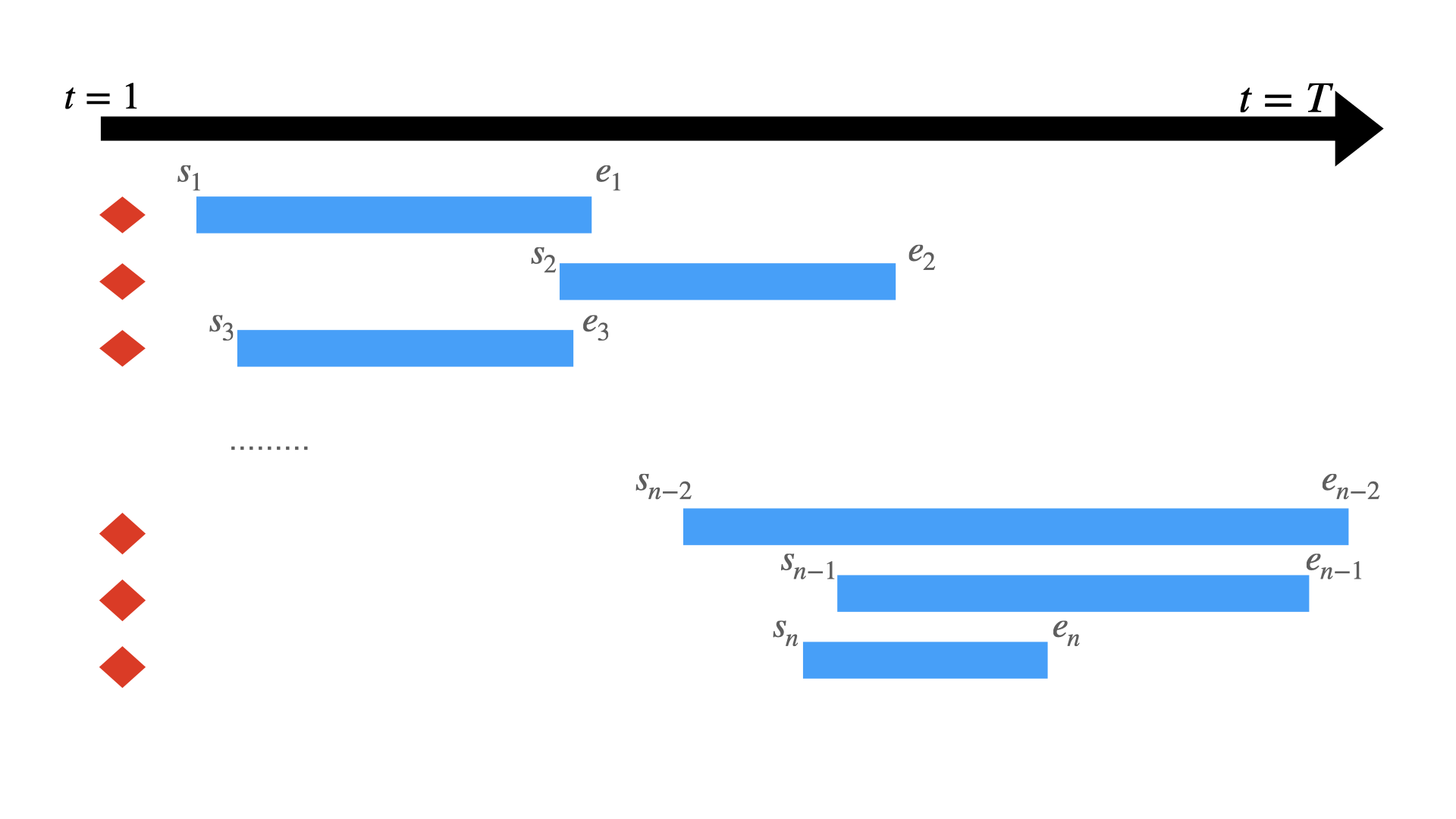}

    }
    \caption{An example where two non-identical segments might overlap.}
           \label{fig2}
\end{figure}

\begin{figure}[!h]
    \centering
  {
	    \includegraphics[width=0.85\textwidth]{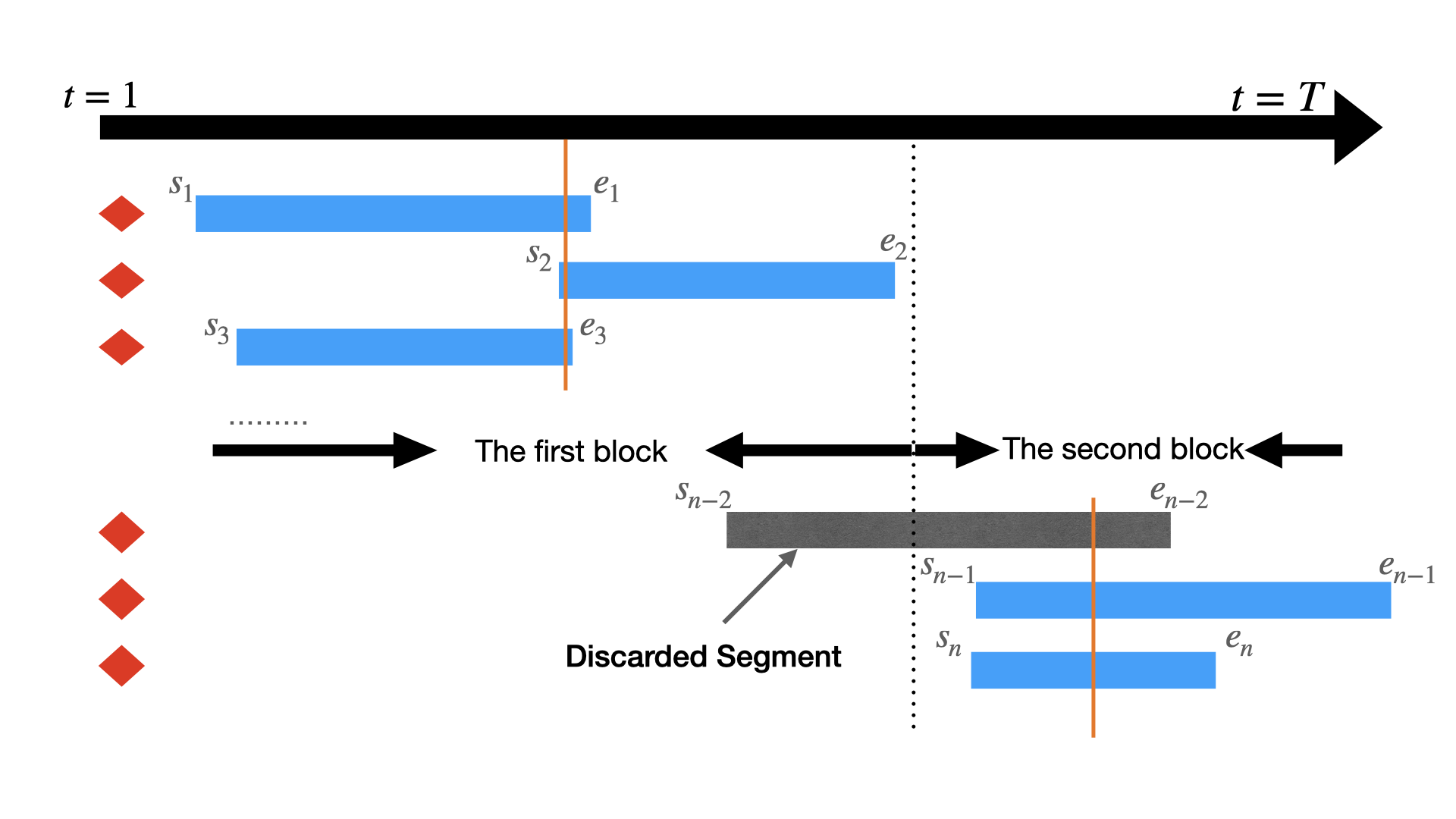}
      
    }
    \caption{Partition of blocks as in the proof of Lemma~\ref{lemma:part3}.}
     \label{fig5}
\end{figure}

\begin{figure}[!h]
    \centering
{
	    \includegraphics[width=0.85\textwidth]{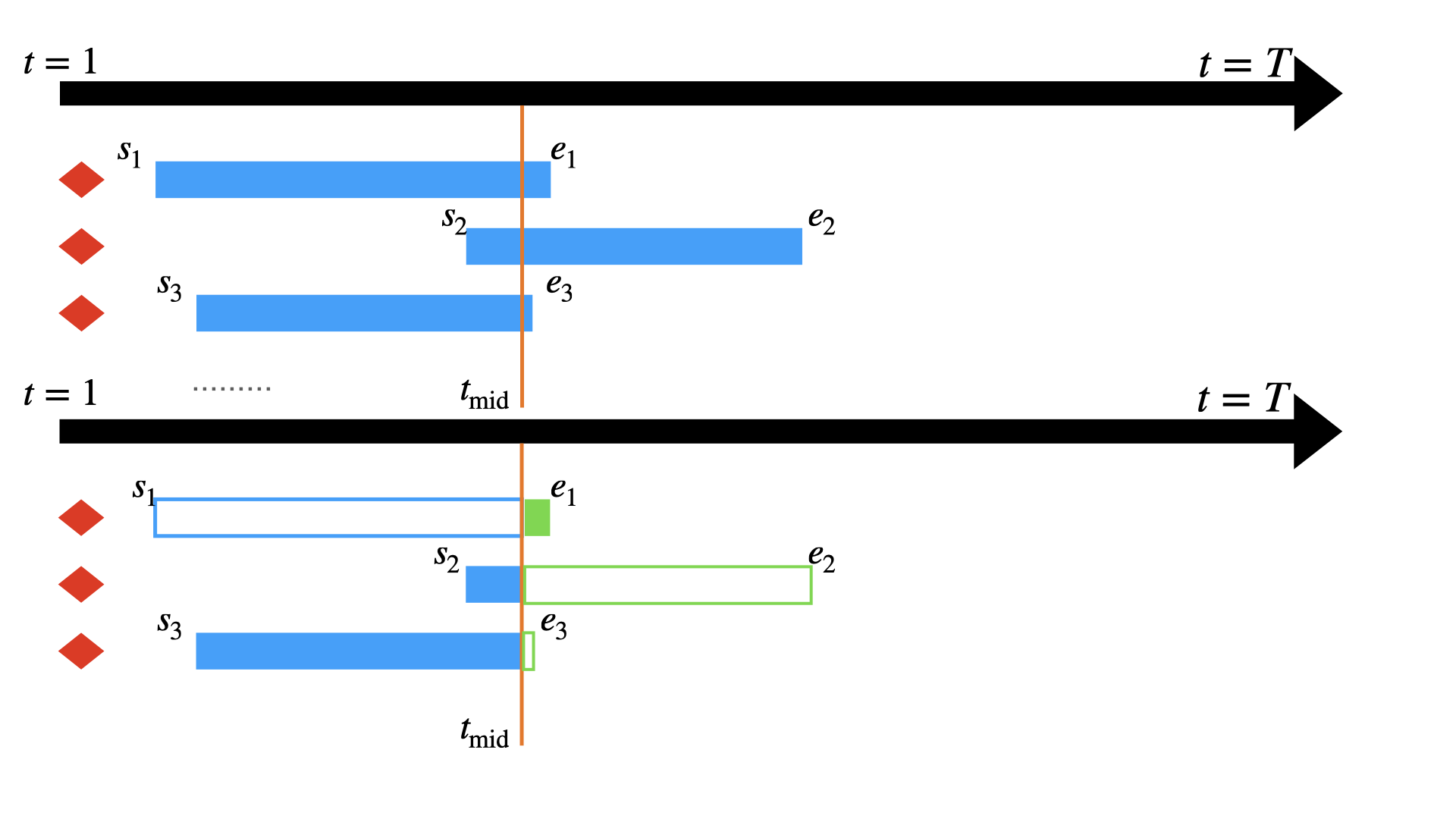}
      
    }
    \caption{
    An illustration of how we use a common inner point (i.e., $t_{\mathsf{mid}}$ to 
    align one side of the segments (as in the proof of Lemma~\ref{lemma:part3}). The unfilled part of the segments means that the weight changes from $w_i^{s_i}$ to $w_i^{t_{\mathrm{mid}}}$ (i.e., $\log\left( w_i^{t_{\mathrm{mid}}}/w_i^{s_i}\right)$) are not significant enough.}
     \label{fig6}
\end{figure}

\begin{figure}[!h]
    \centering
{
	    \includegraphics[width=0.85\textwidth]{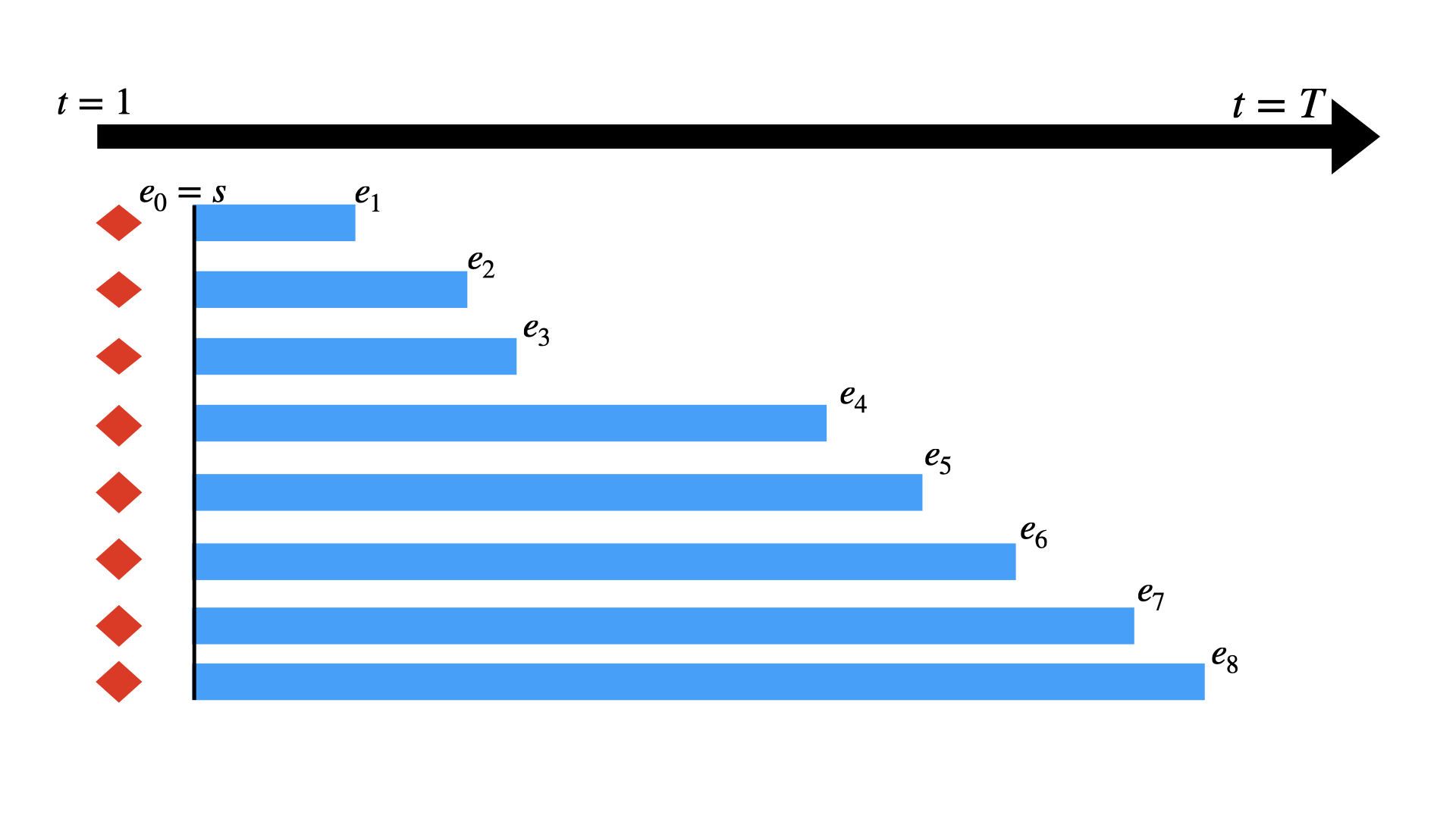}
    
    }
    \caption{A group of segments with common starting points.}
     \label{fig7}
\end{figure}

\begin{figure}[!h]
    \centering
{
	    \includegraphics[width=0.85\textwidth]{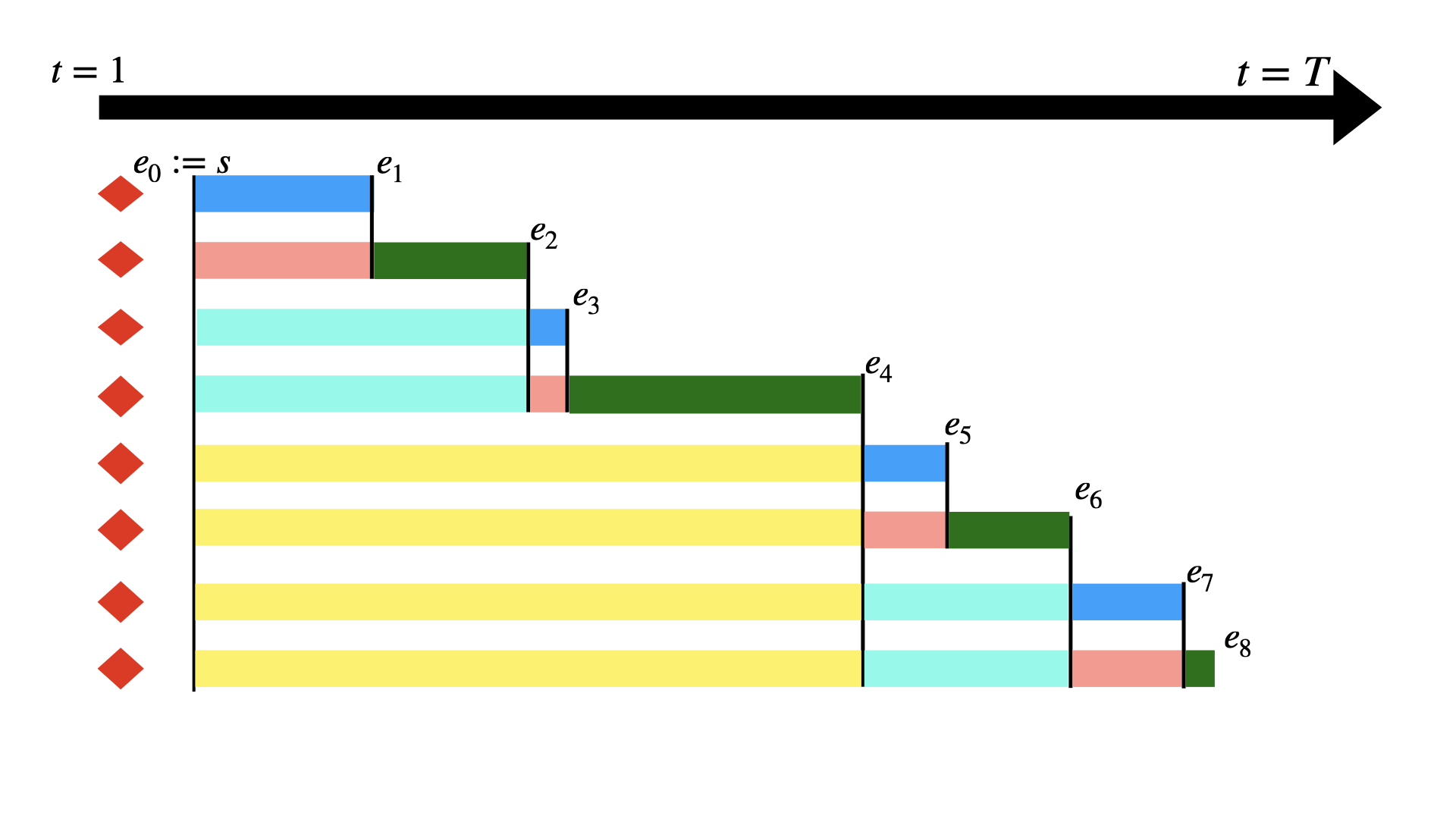}
 
    }
    \caption{Construction of 
    groups of disjoint segments, where each of the original segments is divided into at most $\log_2(k)+1$ sub-segments marked with different colors.}
          \label{fig8}
\end{figure}

\section{Proofs of the lower bound in Theorem~\ref{thm:lb}}
\label{sec:proof-thm:lb}

\subsection{Proof of Theorem~\ref{thm:lb}}
Note that $N_0 =\frac{2^d-1}{k}$. Set $N=kN_0+1=2^d $. 
  Set $\mathcal{X} = \{-1,0,1\}^{kN}$.  We set $\mathcal{Y} = \{1\}$ to be a set with only a single element. Without loss of generality, we write $\ell(h,(x,y))=\ell(h,x)$.

\paragraph{Our construction.}
We now introduce our construction, with the key components described below. 
\begin{itemize}
\item {\em Hypothesis class and loss function.} There are $N$ hypotheses in $\mathcal{H}$, where each hypothesis is assigned $k$ unique dimensions  (recall that $\mathcal{X}$ consists of $(KN)$-dimensional vectors). For each $h\in \mathcal{H}$, we let $\mathcal{I}_h = \{j_{h,i}\}_{i=1}^k$ denote the $k$ dimensions assigned to $h$, so that $\mathcal{I}_{h}\cap\mathcal{I}_{h'}=\emptyset$ for $h\neq h'$. 
We then define $h(x)$ and $\ell(h,x)$ as follows:
$$
	h(x) = \ell(h,x) = \begin{cases} 
 0, & \text{if }x_i = 0 \text{ for all }i\in \mathcal{I}_h; \\
 x_{i'}  \text{ with }i' = \arg\min_{i\in \mathcal{I}_h, x_i\neq 0} x_i, & \text{else}.
 \end{cases}
$$
%
Next, divide $\mathcal{H}$ arbitrarily into $(k+1)$ disjoint subsets as $$\mathcal{H}=(\cup_{i=1}^k \mathcal{H}_i) \cup \{h^{\star}\},$$ where each $\mathcal{H}_i$ contains $N_0$ hypotheses. 
In our construction, we intend to make $h^{\star}$ the unique optimal policy, and will design  properly so that the hypothesis $h\in \mathcal{H}_i$ performs poorly on the $i$-th distribution $\mathcal{D}_i$ ($i\in [k]$).

\item {\em Data distributions.}
We design the $k$ data distributions $\{\mathcal{D}_i\}_{i=1}^k$. For any given $i\in [k]$ and any $x\in \mathcal{X}$, we let 
$$
	\mathbb{P}_{\mathcal{D}_i}\{x\}= \prod_{l=1}^{kN}\mathbb{P}_{\mathcal{D}_{i,l}}\{x_{l}\}
$$
be a product distribution, where 
\begin{align*}
	\mathbb{P}_{\mathcal{D}_{i,l}}\{x_{l}\} & =\mathds1\{x_{l}=0\},\qquad &&l\notin\{j_{h,i}\mid h\in\mathcal{H}\},\\
	\mathbb{P}_{\mathcal{D}_{i,l}}\{x_{l}\} & =\frac{1}{2}\mathds1\{x_{l}=1\}+\frac{1}{2}\mathds1\{x_{l}=-1\},\qquad &&l\in\{j_{h,i}\mid h\notin\mathcal{H}_{i}\},\\
	\mathbb{P}_{\mathcal{D}_{i,l}}\{x_{l}\} & =\left(\frac{1}{2}+4\varepsilon\right)\mathds1\{x_{l}=1\}+\left(\frac{1}{2}-4\varepsilon\right)\mathds1\{x_{l}=-1\},\qquad &&l\in\{j_{h,i}\mid h\in\mathcal{H}_{i}\}.
\end{align*}
%
%

\end{itemize}
The above construction enjoys the following properties: 
\begin{enumerate}[label=(\roman*)]
	\item  $\ell(h,x)\in [-1,1]$ for any $h\in \mathcal{H}$ and $x\in \mathcal{X}$; \item $\mathbb{E}_{x\sim \mathcal{D}_i}[\ell(h,x)]=\mathbb{E}_{x\sim \mathcal{D}_{i,j_{h,i}}}[x_{j_{h,i}}] = 8\varepsilon \cdot \mathds{1}\{h\in \mathcal{H}_i\}$ for any $i\in [k]$ and $h\in \mathcal{H}$; 

	\item  the only $\varepsilon$-optimal hypothesis is $h^{\star}$, because for any $h\neq h^{\star}$, there exists some $i$ such that $h\in \mathcal{H}_i$; 

	\item $h(x)\in \{-1,1\}$ for $x\in \cup_{i=1}^k \mathrm{supp}(\mathcal{D}_i)$,\footnote{We denote by $\mathrm{supp}(\mathcal{D})$ the support of the distribution $\mathcal{D}$.} and  $|\mathcal{H}| =N =  kN_0 +1=2^d$, which imply that $$\mathsf{VC}\text{-}\mathsf{dim}(\mathcal{H})\leq \log_2(N)\leq d$$ over $\cup_{i=1}^k \mathrm{supp}(\mathcal{D}_i)$;

	\item $\ell(h,x)$ could be regarded as a function of $h(x)$ because  $\ell(h,x)=h(x)$.
\end{enumerate}

\paragraph{Sample complexity lower bound.} 
Before proceeding, 
let us introduce the notation $\mathsf{Query}(\mathcal{D}_i)$ such that: 
for each call  to $\mathsf{Query}(\mathcal{D}_i)$, we can obtain independent observations  $\{x_{j_{h,i}}\}_{h\in \mathcal{H}}$ where $x_{j_{h,i}}\sim \mathcal{D}_{i,j_{h,i}}$ for each $h\in \mathcal{H}$.
Now for $i\in [k]$, denote by $M_i$ the number of calls to $\mathsf{Query}(\mathcal{D}_i)$. Our aim is to show that: in order to distinguish $h^{\star}$ from $\mathcal{H}_i$, 
the quantity $M_i$ has to be at least $\Omega(d/\varepsilon^2)$.

Suppose now that there is an algorithm $\mathcal{G}$ with numbers of samples $\{M_i\}_{i=1}^k$ such that the output is $h^{\star}$ with probability at least $3/4$.  
Let $\mathbb{P}_{\mathcal{G}}\{\cdot\}$ and $\mathbb{E}_{\mathcal{G}}[\cdot]$ denote respectively the probability and expectation when Algorithm $\mathcal{G}$ is executed, 
and let $h_{\mathrm{out}}$ be the output hypothesis. It then holds that
\begin{align}
	\mathbb{P}_{\mathcal{G}}\{ h_{\mathrm{out}}=h^{\star}\} \geq \frac{3}{4}.\nonumber
\end{align}
Let $\Pi_{\mathcal{H}}$ 
be the set of permutations over $\mathcal{H}$, 
and let $\mathsf{Unif}(\Pi_{\mathcal{H}})$ be the uniform distribution over $\Pi(\mathcal{H})$.
With slight abuse of notation, for $x\in \{-1,0,1\}^{kN}$ and $\sigma\in \Pi_{\mathcal{H}}$, we define $\sigma(x)$ to be the vector $y$ such that $y_{j_{h,i}} = x_{j_{\sigma(h),i}}$ for all  $h\in \mathcal{H}$ and $i\in [k]$. 
Let $\mathcal{G}'$ be the algorithm with $\mathcal{H}$ replaced by $\sigma(\mathcal{H})$ in the input where $\sigma\sim \mathsf{Unif}(\Pi_{\mathcal{H}})$.
Recognizing that $\mathcal{G}$ returns the optimal hypothesis with probability at least $3/4$ for all problem instances, we can see that
\begin{align}
	\mathbb{P}_{\mathcal{G}'}\{ h_{\mathrm{out}}=h^{\star} \} \geq \frac{3}{4}.\nonumber
\end{align}

The lemma below then assists in bounding the probability of returning a sub-optimal hypothesis.

\begin{lemma}\label{lemma:lb1} Consider $\widetilde{i}\in [k]$.
	If $\mathbb{P}_{\mathcal{G}'}\{ h_{\mathrm{out}}=h^{\star}, M_i\leq m\} \geq 1/2$ for some $m\geq 0$, then for any $h\in \mathcal{H}_{\widetilde{i}}$, one has
\begin{align}
	\mathbb{P}_{\mathcal{G}'} \big\{ h_{\mathrm{out}}=h,M_{\widetilde{i}}\leq m \big\} 
	\geq  \frac{1}{2}\mathbb{P}_{\mathcal{G}'} \big\{ h_{\mathrm{out}}=h^{\star}, M_{\widetilde{i}} \leq m \big\}
	\exp\big( -80\sqrt{m}\varepsilon-40m\varepsilon^2 \big),\nonumber
\end{align}
and moreover, $m$ necessarily exceeds $m\geq \frac{\log(N_0/4)}{30000\varepsilon^2}$.
\end{lemma}
\begin{proof} See Appendix~\ref{sec:proof-lemma:lb1}.  \end{proof}

In view of Lemma~\ref{lemma:lb1}, we have 
 \begin{align}
	 \mathbb{P}_{\mathcal{G}'}\left\{ h_{\mathrm{out}}=h^{\star},M_{\widetilde{i}} < \frac{\log(N_0/4)}{30000\varepsilon^2}\right\} < \frac{1}{2}.
 \end{align}
Observing that $\mathbb{P}_{\mathcal{G}'}\{h_{\mathrm{out}}=h^{\star}\}\geq 3/4$, we can derive 
\begin{align*}
	\mathbb{P}_{\mathcal{G}'}\left\{ h_{\mathrm{out}}=h^{\star},M_{\widetilde{i}}\geq\frac{\log(N_{0}/4)}{30000\varepsilon^{2}}\right\}  & =\mathbb{P}_{\mathcal{G}'}\left\{ h_{\mathrm{out}}=h^{\star}\right\} -\mathbb{P}_{\mathcal{G}'}\left\{ h_{\mathrm{out}}=h^{\star},M_{\widetilde{i}}<\frac{\log(N_{0}/4)}{30000\varepsilon^{2}}\right\} \\
 & >\frac{3}{4}-\frac{1}{2}=\frac{1}{4},
\end{align*}
which implies that 
\begin{align*}
\mathbb{E}_{\mathcal{G}'}[M_{\widetilde{i}}] & \geq\frac{\log(N_{0}/4)}{30000\varepsilon^{2}}\cdot\mathbb{P}_{\mathcal{G}'}\left\{ M_{\widetilde{i}}\geq\frac{\log(N_{0}/4)}{30000\varepsilon^{2}}\right\} \geq\frac{\log(N_{0}/4)}{30000\varepsilon^{2}}\cdot\mathbb{P}_{\mathcal{G}'}\left\{ h_{\mathrm{out}}=h^{\star},M_{\widetilde{i}}\geq\frac{\log(N_{0}/4)}{30000\varepsilon^{2}}\right\} \\
 & \geq\frac{\log(N_{0}/4)}{120000\varepsilon^{2}}\geq\frac{d-\log_{2}(8k)}{120000\varepsilon^{2}}\geq\frac{d}{240000\varepsilon^{2}}.
\end{align*}
Summing over $i\in [k]$ gives
\begin{align}
\mathbb{E}_{\mathcal{G}'}\left[\sum_{i=1}^k M_{\widetilde{i}}\right]\geq \frac{dk}{240000\varepsilon^2},\label{eq:rer}
\end{align}
thereby concluding the proof.
%
%

\subsection{Proof of Lemma~\ref{lemma:lb1}}
\label{sec:proof-lemma:lb1}



	Consider any $h\in \mathcal{H}_{\widetilde{i}}$ (and hence $h\neq h^{\star}$). 
For $v = [v_p]_{1\leq p\leq m}\in \{-1,1\}^{m}$, we let $$ n^+(v)= \sum_{p=1}^m \mathds{1}\{v_p = 1\} $$ denote the number of $1$'s in the coordinates of $v$. 
Let $\mathcal{V}$ be a subset of $\{-1,1\}^{2m}$ defined as 
$$
	\mathcal{V} \coloneqq \big\{ v^1,v^2\in \{-1,1\}^m \mid n^+(v^1)-n^+(v^2)\leq 4\sqrt{m}+2m\varepsilon \big\}.
$$
	Let $\mathbb{P}_{\mathcal{C}}\{\cdot\}$ denote the probability distribution of $\left\{\{x_{j_{h,\widetilde{i}}}^{l}(\widetilde{i})\}_{l=1}^{m}, \{x_{j_{h^{\star},\widetilde{i}}}^{l}(\widetilde{i})\}_{l=1}^{m} \right\}$, and $\mathbb{P}_{\mathcal{C}'}\{\cdot\}$  the probability distribution of $\left\{\{x_{j_{h^{\star},\widetilde{i}}}^{l}(\widetilde{i})\}_{l=1}^{m}, \{x_{j_{h,\widetilde{i}}}^{l}(\widetilde{i})\}_{l=1}^{m} \right\}$.
%
Hoeffding's inequality then tells us that
$$
	\mathbb{P}_{\mathcal{C}}\{\mathcal{V}\}\geq \frac{3}{4}.
$$
Also, observe that 
%
\begin{align}
 & \mathbb{P}_{\mathcal{G}'}\left\{ h_{\mathrm{out}}=h^{\star},M_{\widetilde{i}}\leq m,\left\{ \{x_{j_{h,\widetilde{i}}}^{l}(\widetilde{i})\}_{l=1}^{m},\{x_{j_{h^{\star},\widetilde{i}}}^{l}(\widetilde{i})\}_{l=1}^{m}\right\} \in\mathcal{V}\right\} \notag\\
 & \qquad\geq\mathbb{P}_{\mathcal{G}'}\left\{ h_{\mathrm{out}}=h^{\star},M_{\widetilde{i}}\leq m\right\} -\left(1-\mathbb{P}_{\mathcal{C}}\left\{ \mathcal{V}\right\} \right)\geq\frac{3}{4}-\left(1-\frac{3}{4}\right)=\frac{1}{2},
	\label{eq:hardness_1}
\end{align}
namely,
\begin{align}
	\sum_{v=\{v^1,v^2\}\in \mathcal{V}}\mathbb{P}_{\mathcal{G}'}\left\{ h_{\mathrm{out}}=h^{\star}, M_{\widetilde{i}} \leq m, \{x_{j_{h,\widetilde{i}}}^{l}(\widetilde{i})\}_{l=1}^{m}= v^1, \{x_{j_{h^{\star},\widetilde{i}}}^{l}(\widetilde{i})\}_{l=1}^{m}= v^2 \right\}
	\geq \frac{1}{2}.\nonumber
\end{align}


In addition, for any $v=\{v^1,v^2\}\in \mathcal{V}$, it is readily seen that
\begin{align}
	\mathbb{P}_{\mathcal{C}'}\{v\} &=\mathbb{P}_{\mathcal{C}}\{v\} \cdot \left(1-8\varepsilon\right)^{n^+(v^1)-n^+(v^2)} (1+8\varepsilon)^{n^+(v^2)-n^+(v^1)} \nonumber
\\ & = \mathbb{P}_{\mathcal{C}}\{v\} \left( \frac{1-8\varepsilon}{1+8\varepsilon}\right)^{n^+(v^1)-n^+(v^2)} \nonumber
\\ & \geq \mathbb{P}_{\mathcal{C}}\{v\} \exp\Big(-20\big(n^+(v^1)-n^+(v^2)\big)\varepsilon \Big)\nonumber
\\ & \geq \mathbb{P}_{\mathcal{C}}\{v\} \exp(-80\sqrt{m}\varepsilon-40m\varepsilon^2),\nonumber
\end{align}
where the last line follows from the definition of $\mathcal{V}$. 
As a result, we can demonstrate that, for any $v=\{v^1,v^2\}\in \mathcal{V}$, 
\begin{align}
	& \mathbb{P}_{\mathcal{G}'}\left\{ h_{\mathrm{out}}=h, M_{\widetilde{i}} \leq m , \{x_{j_{h^{\star},\widetilde{i}}}^{l}(\widetilde{i})\}_{l=1}^{m}= v^1, \{x_{j_{h,\widetilde{i}}}^{l}(\widetilde{i})\}_{l=1}^{m}= v^2\right\} \nonumber
	\\ & = \mathbb{P}_{\mathcal{G}'}\left\{h_{\mathrm{out}}=h, M_{\widetilde{i}}\leq m \mid \{x_{j_{h^{\star},\widetilde{i}}}^{l}(\widetilde{i})\}_{l=1}^{m}= v^1, \{x_{j_{h,\widetilde{i}}}^{l}(\widetilde{i})\}_{l=1}^{m}= v^2\right\}\mathbb{P}_{\mathcal{C'}}\{v\} \nonumber
	\\ & \geq \mathbb{P}_{\mathcal{G}'}\left\{h_{\mathrm{out}}=h, M_{\widetilde{i}}\leq m \mid  \{x_{j_{h^{\star},\widetilde{i}}}^{l}(\widetilde{i})\}_{l=1}^{m}= v^1, \{x_{j_{h,\widetilde{i}}}^{l}(\widetilde{i})\}_{l=1}^{m}= v^2\right\} \mathbb{P}_{\mathcal{C}}\{v\} \exp(-80\sqrt{m}\varepsilon-40m\varepsilon^2)\nonumber
	\\& = \mathbb{P}_{\mathcal{G}'}\left\{ h_{\mathrm{out}}=h^{\star}, M_{\widetilde{i}}\leq m \mid  \{x_{j_{h,\widetilde{i}}}^{\ell}(\widetilde{i})\}_{\ell=1}^{m}= v^1, \{x_{j_{h^{\star},\widetilde{i}}}^{\ell}(\widetilde{i})\}_{\ell=1}^{m}= v^2\right\} \mathbb{P}_{\mathcal{C}}\{v\} \exp(-80\sqrt{m}\varepsilon-40m\varepsilon^2)\label{eq:lnk}
	\\ & = \mathbb{P}_{\mathcal{G}'}\left\{ h_{\mathrm{out}}=h^{\star}, M_{\widetilde{i}} \leq m, \{x_{j_{h,\widetilde{i}}}^{\ell}(\widetilde{i})\}_{\ell=1}^{m}= v^1, \{x_{j_{h^{\star},\widetilde{i}}}^{\ell}(\widetilde{i})\}_{\ell=1}^{m}= v^2\right\} \exp(-80\sqrt{m}\varepsilon-40m\varepsilon^2) .\nonumber
\end{align}
To see why \eqref{eq:lnk} holds, observe that (here,  for any $1\leq l\leq m$, let $v^{1}_{l}$ (resp.~$v^{2}_{l}$) be the $l$-th coordinate of $v^1$ (resp.~$v^2$)): 
%
\begin{align}
	& \mathbb{P}_{\mathcal{G}'}\left\{ h_{\mathrm{out}}=h, M_{\widetilde{i}}\leq m \mid  \{x_{j_{h^{\star},\widetilde{i}}}^{l}(\widetilde{i})\}_{l=1}^{m}= v^1, \{x_{j_{h,\widetilde{i}}}^{l}(\widetilde{i})\}_{l=1}^{m}= v^2\right\} \nonumber
	\\ & = \sum_{m'=1}^{m}\mathbb{P}_{\mathcal{G}'}\left\{ h_{\mathrm{out}}=h, M_{\widetilde{i}}=m' \mid  \{x_{j_{h^{\star},\widetilde{i}}}^{l}(\widetilde{i})\}_{l=1}^{m}= v^1, \{x_{j_{h,\widetilde{i}}}^{l}(\widetilde{i})\}_{l=1}^{m}= v^2\right\}  \nonumber
	\\ & = \sum_{m'=1}^{m}\mathbb{P}_{\mathcal{G}'}\left\{ h_{\mathrm{out}}=h, M_{\widetilde{i}}=m' \mid  \{x_{j_{h^{\star},\widetilde{i}}}^{l}(\widetilde{i})\}_{l=1}^{m'}= \{v^{1}_{l}\}_{l=1}^{m'}, \{x_{j_{h,\widetilde{i}}}^{l}(\widetilde{i})\}_{l=1}^{m'}= \{v^{2}_{l}\}_{l=1}^{m'}\right\}  \label{eq:new1}
	\\ & = \sum_{m'=1}^{m}\mathbb{P}_{\mathcal{G}'}\left\{ h_{\mathrm{out}}=h^{\star}, M_{\widetilde{i}}=m' \mid  \{x_{j_{h,\widetilde{i}}}^{l}(\widetilde{i})\}_{l=1}^{m'}= \{v^{1}_{l}\}_{l=1}^{m'}, \{x_{j_{h^{\star},\widetilde{i}}}^{l}(\widetilde{i})\}_{l=1}^{m'}= \{v^{2}_{l}\}_{l=1}^{m'}\right\}  \label{eq:new3}
	\\ & = \sum_{m'=1}^{m}\mathbb{P}_{\mathcal{G}'}\left\{ h_{\mathrm{out}}=h^{\star}, M_{\widetilde{i}}=m' \mid  \{x_{j_{h,\widetilde{i}}}^{l}(\widetilde{i})\}_{l=1}^{m}= v^1, \{x_{j_{h^{\star},\widetilde{i}}}^{l}(\widetilde{i})\}_{l=1}^{m}= v^2\right\} \label{eq:new2}
	\\ & = \mathbb{P}_{\mathcal{G}'}\left\{ h_{\mathrm{out}}=h^{\star}, M_{\widetilde{i}}\leq m \mid  \{x_{j_{h,\widetilde{i}}}^{l}(\widetilde{i})\}_{l=1}^{m}= v^1, \{x_{j_{h^{\star},\widetilde{i}}}^{l}(\widetilde{i})\}_{l=1}^{m}= v^2\right\} .\nonumber
\end{align}
where \eqref{eq:new3} results from Lemma~\ref{lemma:sym} in Appendix~\ref{sec:proof-lemma:sym}, and \eqref{eq:new1} and \eqref{eq:new2} hold since for $h'=h,h^{\star}$, the event $\{h_{\mathrm{out}}=h', M_{\widetilde{i}} = m'\}$ is independent of $\{x^{l}(\widetilde{i})\}_{l\geq m'+1}$.
Taking the sum over $v=\{v^1,v^2\}\in \mathcal{V}$, we obtain
\begin{align}
	& \mathbb{P}_{\mathcal{G}'}\{h_{\mathrm{out}}=h, M_{\widetilde{i}} \leq m \}\nonumber
	\\ &\geq \sum_{v\in \mathcal{V}}\mathbb{P}_{\mathcal{G}'}\left\{ h_{\mathrm{out}}=h, M_{\widetilde{i}} \leq m ,\{x_{j_{h^{\star},\widetilde{i}}}^{\ell}(\widetilde{i})\}_{\ell=1}^{m}= v^1, \{x_{j_{h,\widetilde{i}}}^{\ell}(\widetilde{i})\}_{\ell=1}^{m}= v^2 \right\} \nonumber
	\\ & \geq \sum_{v\in \mathcal{V}}\mathbb{P}_{\mathcal{G}'}\left\{ h_{\mathrm{out}}=h^{\star}, M_{\widetilde{i}} \leq m,  \{x_{j_{h,\widetilde{i}}}^{\ell}(\widetilde{i})\}_{\ell=1}^{m}= v^1, \{x_{j_{h^{\star},\widetilde{i}}}^{\ell}(\widetilde{i})\}_{\ell=1}^{m}= v^2  \right\} \exp(-80\sqrt{m}\varepsilon-40m\varepsilon^2) \nonumber
	\\ & =\mathbb{P}_{\mathcal{G}'}\left\{ h_{\mathrm{out}}=h^{\star}, M_{\widetilde{i}} \leq m, \left\{ \{x_{j_{h,\widetilde{i}}}^{\ell}(\widetilde{i})\}_{\ell=1}^{m}, \{x_{j_{h^{\star},\widetilde{i}}}^{\ell}(\widetilde{i})\}_{\ell=1}^{m}\right\} \in \mathcal{V} \right\} \exp(-80\sqrt{m}\varepsilon-40m\varepsilon^2) \nonumber
	\\ & \geq \frac{1}{2}\mathbb{P}_{\mathcal{G}'} \big\{ h_{\mathrm{out}}=h^{\star}, M_{\widetilde{i}} \leq m \big\} \exp(-80\sqrt{m}\varepsilon-40m\varepsilon^2) , \nonumber
\end{align}
where the last line arises from \eqref{eq:hardness_1}. 

Summing over all $h\in \mathcal{H}_{\widetilde{i}}$, we reach 
\begin{align}
	1\geq \mathbb{P}_{\mathcal{G}'}\big\{ h_{\mathrm{out}}\in \mathcal{H}_{\widetilde{i}}, M_{\widetilde{i}} \leq m \big\}
	\geq \frac{N_0}{2} \mathbb{P}_{\mathcal{G}'}\big\{ h_{\mathrm{out}}=h^{\star},M_{\widetilde{i}} \leq m \big\} 
	\exp(-40m\varepsilon^2 - 80\sqrt{m}\varepsilon),
\end{align}
given that each $ \mathcal{H}_{\widetilde{i}}$ contains $N_0$ hypotheses. 
This in turn reveals that
\begin{align}
	\frac{1}{2}\leq \mathbb{P}_{\mathcal{G}'}\left\{ h_{\mathrm{out}}=h^{\star},M_{\widetilde{i}} \leq m\right\} 
	\leq \frac{2}{N_0}\exp(40m\varepsilon^2  +80\sqrt{m}\varepsilon).
\end{align}
Consequently, we arrive at
$$
	40m\varepsilon^2 +80\sqrt{m}\varepsilon\geq \log(N_0/4),
$$
which implies that 
$$
	m \geq \min\bigg\{ \frac{\log(N_0/4)}{80\varepsilon^2} ,\frac{\log^2(N_0/4)}{30000\varepsilon^2} \bigg\}\geq \frac{\log(N_0/4)}{30000\varepsilon^2}. 
$$
%

\subsection{Statement and proof of Lemma~\ref{lemma:sym}}
\label{sec:proof-lemma:sym}

\begin{lemma}\label{lemma:sym} 
For any $i\in [k]$ and $l\geq 1$, let $x^{l}(i)$ denote the $l$-th sample from $\mathcal{D}_i$. 
For any $\widetilde{i}\in [k]$,  $h\in \mathcal{H}_{\widetilde{i}}$, $m>0$, and $v^1,v^2\in \{-1,1\}^{2m}$, one has
\begin{align}
&\mathbb{P}_{\mathcal{G}'}\left\{ h_{\mathrm{out}}=h^{\star}, M_i = m \mid \{x_{j_{h,\widetilde{i}}}^{l}(\widetilde{i})\}_{l=1}^{m}=v^1,   \{x_{j_{h^{\star},\widetilde{i}}}^{l}(\widetilde{i})\}_{l=1}^{m}=v^2\right\} \nonumber
\\ & = \mathbb{P}_{\mathcal{G}'}\left\{ h_{\mathrm{out}}=h, M_i = m \mid \{x_{j_{h,\widetilde{i}}}^{l}(\widetilde{i})\}_{l=1}^{m}=v^2,   \{x_{j_{h^{\star},\widetilde{i}}}^{l}(\widetilde{i})\}_{l=1}^{m}=v^1 \right\} . \nonumber
\end{align}
\end{lemma}
\begin{proof}
Let $\overline{\sigma}$ be the permutation over $\mathcal{H}$ with $\overline{\sigma}(h^{\star})=h, \overline{\sigma}(h)=h^{\star}$ and $\overline{\sigma}(h')=h'$ for all $h'\notin \{h,h^{\star}\}$. It then holds that $\overline{\sigma}^{-1}=\overline{\sigma}$.

Consider a given sequence $\{m_i\}_{i=1}^k$. Let $X(i) =\{X^{l}(i)\}_{l=1}^{m_i} \in \{-1,0,1\}^{kNm_i}$ for $i\in [k]$, and  
let $x(i) = \{x^{l}(i)\}_{l=1}^{m_i}$ be the datapoints of the first $m_i$ calls to $\mathsf{Query}(D_i)$. With slight abuse of notation, we take $\sigma(x(i)) = \{ \sigma(x^{l}(i))\}_{l=1}^{m_i}$ for each $i\in [k]$. 
It then follows from Lemma~\ref{fact:1} in Appendix~\ref{sec:proof-fact:1} that 
\begin{align}
	& \mathbb{P}_{\mathcal{G},\mathcal{H}}\left\{ h_{\mathrm{out}}= h, \{M_i\}_{i=1}^k = \{m_i\}_{i=1}^k  \mid x(i)= X(i) ,\forall i\in [k] \right\} \nonumber
	\\ & = \mathbb{P}_{\mathcal{G},\sigma(\mathcal{H})}\left\{ 
	h_{\mathrm{out}}=\sigma^{-1}(h), \{M_i\}_{i=1}^k = \{m_i\}_{i=1}^k \mid   \sigma^{-1}(x(i))= X(i) ,\forall i\in [k] \right\},\nonumber
\end{align}
%
which implies that
\begin{align}
 & \mathbb{P}_{\mathcal{G},\sigma(\mathcal{H})}\left\{ h_{\mathrm{out}}=h^{\star},\{M_{i}\}_{i=1}^{k}=\{m_{i}\}_{i=1}^{k},\sigma(x(i))=X(i),\forall i\in[k]\right\} \nonumber\\
 & =\mathbb{P}_{\mathcal{G},\overline{\sigma}\sigma(\mathcal{H})}\left\{ h_{\mathrm{out}}=h,\{M_{i}\}_{i=1}^{k}=\{m_{i}\}_{i=1}^{k},\overline{\sigma}\sigma(x(i))=X(i),\forall i\in[k]\right\} \cdot\frac{\mathbb{P}\left\{ \sigma(x(i))=X(i),\forall i\in[k]\right\} }{\mathbb{P}\left\{ \overline{\sigma}\sigma(x(i))=X(i),\forall i\in[k]\right\} }\nonumber\\
 & =\mathbb{P}_{\mathcal{G},\overline{\sigma}\sigma(\mathcal{H})}\left\{ h_{\mathrm{out}}=h,\{M_{i}\}_{i=1}^{k}=\{m_{i}\}_{i=1}^{k},\overline{\sigma}\sigma(x(i))=X(i),\forall i\in[k]\right\} \cdot\frac{\mathbb{P}\left\{ \sigma(x(\widetilde{i}))=X(\widetilde{i})\right\} }{\mathbb{P}\left\{ \overline{\sigma}\sigma(x(\widetilde{i}))=X(\widetilde{i})\right\} }\nonumber\\
 & =\mathbb{P}_{\mathcal{G},\overline{\sigma}\sigma(\mathcal{H})}\left\{ h_{\mathrm{out}}=h,\{M_{i}\}_{i=1}^{k}=\{m_{i}\}_{i=1}^{k},\overline{\sigma}\sigma(x(i)))=X(i),\forall i\in[k]\right\} \nonumber\\
 & \quad\qquad\qquad\qquad\qquad\cdot\frac{\mathbb{P}\left\{ \{x_{j_{h,i}}^{l}(\widetilde{i})\}_{l=1}^{m_{\widetilde{i}}}=\{X_{j_{\sigma(h),\widetilde{i}}}^{l}(\widetilde{i})\}_{l=1}^{m_{\widetilde{i}}},\{x_{j_{h^{\star},\widetilde{i}}}^{l}({\widetilde{i}})\}_{l=1}^{m_{\widetilde{i}}}=\{X_{j_{\sigma(h^{\star}),{\widetilde{i}}}}^{l}(\widetilde{i})\}_{l=1}^{m_{\widetilde{i}}}\right\} }{\mathbb{P}\left\{ \{x_{j_{h^{\star},\widetilde{i}}}^{l}(\widetilde{i})\}_{l=1}^{m_{\widetilde{i}}}=\{X_{j_{\sigma(h),\widetilde{i}}}^{l}(\widetilde{i})\}_{l=1}^{m_{\widetilde{i}}},\{x_{j_{h,\widetilde{i}}}^{l}(\widetilde{i})\}_{l=1}^{m_{\widetilde{i}}}=\{X_{j_{\sigma(h^{\star}),\widetilde{i}}}^{l}(\widetilde{i})\}_{l=1}^{m_{\widetilde{i}}}\right\} }.
\end{align}
Rearrange the equation to arrive at
\begin{align}
	&\mathbb{P}_{\mathcal{G},\sigma(\mathcal{H})}\Bigg\{ h_{\mathrm{out}}=h^{\star}, \{M_i\}_{i=1}^k = \{m_i\}_{i=1}^k, \sigma(x(i)) = X(i),\forall i\in [k] 
 \nonumber
	\\ & \qquad \qquad \qquad \qquad \Big|\,  \{x_{j_{h,\widetilde{i}}}^{l}(\widetilde{i})\}_{l=1}^{m_{\widetilde{i}}} = \{X^{l}_{j_{\sigma(h),\widetilde{i}}}(\widetilde{i})\}_{l=1}^{m_{\widetilde{i}}} ,\{x_{j_{h^{\star},\widetilde{i}}}^{l}(\widetilde{i})\}_{l=1}^{m_{\widetilde{i}}} = \{X^{l}_{j_{\sigma(h^{\star}),\widetilde{i}}}(\widetilde{i})\}_{l=1}^{m_{\widetilde{i}}}\Bigg\} \nonumber
	\\ & = \mathbb{P}_{\mathcal{G},\overline{\sigma}\sigma(\mathcal{H})}\Bigg\{ h_{\mathrm{out}}=h, \{M_i\}_{i=1}^k = \{m_i\}_{i=1}^k,  \overline{\sigma}\sigma(x(i)) = X(i),\forall i\in [k] 
 \nonumber
	\\ & \qquad \qquad \qquad \quad \Big|\,  \{x_{j_{h^{\star},\widetilde{i}}}^{l}(\widetilde{i})\}_{l=1}^{m_{\widetilde{i}}} = \{X^{l}_{j_{\sigma(h),\widetilde{i}}}(\widetilde{i})\}_{l=1}^{m_{\widetilde{i}}} ,\{x_{j_{h,\widetilde{i}}}^{l}(\widetilde{i})\}_{l=1}^{m_{\widetilde{i}}} = \{X^{l}_{j_{\sigma(h^{\star}),\widetilde{i}}}(\widetilde{i})\}_{l=1}^{m_{\widetilde{i}}}\Bigg \}. \nonumber
\end{align}
Taking the sum over all possible choices of $\{X(i)\}_{i\neq \widetilde{i}}$, $\left\{\{X^{l}_{j_{h',i'}}(\widetilde{i})\}_{l=1}^{m_{\widetilde{i}}}\right\}_{i'\in [k],h'\notin \{h,h^{\star}\}}$,\\
$\left\{  \{X^{l}_{j_{h',i'}}(\widetilde{i})\}_{l=1}^{m_{\widetilde{i}}}   \right\}_{h'\in \{h,h^{\star}\} , i'\neq \widetilde{i}} $  and $\{m_{i}\}_{i\neq \widetilde{i}}$, we reach 
\begin{align}
	&\mathbb{P}_{\mathcal{G},\sigma(\mathcal{H})}\Bigg\{ h_{\mathrm{out}}=h^{\star}, M_{\widetilde{i}}=m_{\widetilde{i}} 
 \nonumber
	\\ & \qquad \qquad \qquad \qquad \Big|\,  \{x_{j_{h,\widetilde{i}}}^{l}(\widetilde{i})\}_{l=1}^{m_{\widetilde{i}}} = \{X^{l}_{j_{\sigma(h),\widetilde{i}}}(\widetilde{i})\}_{l=1}^{m_{\widetilde{i}}} ,\{x_{j_{h^{\star},\widetilde{i}}}^{l}(\widetilde{i})\}_{l=1}^{m_{\widetilde{i}}} = \{X^{l}_{j_{\sigma(h^{\star}),\widetilde{i}}}(\widetilde{i})\}_{l=1}^{m_{\widetilde{i}}}\Bigg \} \nonumber
	\\ & = \mathbb{P}_{\mathcal{G},\overline{\sigma}\sigma(\mathcal{H})}\Bigg\{ h_{\mathrm{out}}=h, M_{\widetilde{i}}=m_{\widetilde{i}}
 \nonumber
	\\ & \qquad \qquad \qquad \qquad \Big|\,  \{x_{j_{h^{\star},\widetilde{i}}}^{l}(\widetilde{i})\}_{l=1}^{m_{\widetilde{i}}} = \{X^{l}_{j_{\sigma(h),\widetilde{i}}}(\widetilde{i})\}_{l=1}^{m_{\widetilde{i}}} ,\{x_{j_{h,\widetilde{i}}}^{l}(\widetilde{i})\}_{l=1}^{m_{\widetilde{i}}} = \{X^{l}_{j_{\sigma(h^{\star}),\widetilde{i}}}(\widetilde{i})\}_{l=1}^{m_{\widetilde{i}}}\Bigg\} \nonumber
\end{align}
for any $X(\widetilde{i})\in \{-1,0,1\}^{kNm_{\widetilde{i}}} $. 


Fix $m_{\widetilde{i}} = m$, and choose $ \{X^{l}_{j_{\sigma(h),\widetilde{i}}}(\widetilde{i})\}_{l=1}^{m_{\widetilde{i}}}=v_1$, $\{X^{l}_{j_{\sigma(h^{\star}),\widetilde{i}}}(\widetilde{i})\}_{l=1}^{m_{\widetilde{i}}} =v_2$. 
We then have 
\begin{align}
	& \mathbb{P}_{\mathcal{G}',\mathcal{H}}\left\{ h_{\mathrm{out}}=h^{\star}, M_i = m \mid \{x_{j_{h,\widetilde{i}}}^{l}(\widetilde{i})\}_{l=1}^{m}=v^1,   \{x_{j_{h^{\star},\widetilde{i}}}^{l}(\widetilde{i})\}_{l=1}^{m}=v^2\right\}  \nonumber
	\\ & = \frac{1}{|\Pi_{\mathcal{H}}|}\sum_{\sigma \in \Pi_{\mathcal{H}}}\mathbb{P}_{\mathcal{G},\sigma(\mathcal{H})}\left\{ h_{\mathrm{out}}=h^{\star}, M_i = m \mid  \{x_{j_{h,\widetilde{i}}}^{l}(\widetilde{i})\}_{l=1}^{m}=v^1,   \{x_{j_{h^{\star},\widetilde{i}}}^{l}(\widetilde{i})\}_{l=1}^{m}=v^2\right\} \nonumber
\\ & = \frac{1}{|\Pi_{\mathcal{H}}|}\sum_{\sigma \in \Pi_{\mathcal{H}}} 
	\mathbb{P}_{\mathcal{G},\overline{\sigma}\sigma(\mathcal{H})}\left\{ h_{\mathrm{out}}=\overline{\sigma}^{-1}(h^{\star}), M_i = m \mid  \{x_{j_{h,\widetilde{i}}}^{l}(\widetilde{i})\}_{l=1}^{m}=v^2,   \{x_{j_{h^{\star},\widetilde{i}}}^{l}(\widetilde{i})\}_{l=1}^{m}=v^1\right\} \nonumber
\\ & =  \frac{1}{|\Pi_{\mathcal{H}}|}\sum_{\sigma \in \Pi_{\mathcal{H}}} 
	\mathbb{P}_{\mathcal{G},\sigma(\mathcal{H})}\left\{ h_{\mathrm{out}}=h, M_i =m \mid \{x_{j_{h,\widetilde{i}}}^{l}(\widetilde{i})\}_{l=1}^{m}=v^2,   \{x_{j_{h^{\star},\widetilde{i}}}^{l}(\widetilde{i})\}_{l=1}^{m}=v^1\right\} \nonumber
	\\ & = \mathbb{P}_{\mathcal{G}',\mathcal{H}}\left\{ h_{\mathrm{out}}=h, M_i = m \mid  \{x_{j_{h,\widetilde{i}}}^{l}(\widetilde{i})\}_{l=1}^{m}=v^2,   \{x_{j_{h^{\star},\widetilde{i}}}^{l}(\widetilde{i})\}_{l=1}^{m}=v^1\right\}.\nonumber
\end{align}
This completes the proof.
\end{proof}

\subsection{Statement and proof of Lemma~\ref{fact:1}}
\label{sec:proof-fact:1}

\begin{lemma}\label{fact:1}  
Consider any $\{m_i\}_{i\in [k]}$, $\sigma \in \Pi_{\mathcal{H}}$  and $X\in \{-1,0,1\}^{kN\sum_{i=1}^k m_i}$.  Let $\{X(i)\}_{i\in [k]}$ and $\{x(i)\}$ be defined as in Lemma~\ref{lemma:sym}. 
It then holds that
\begin{align}
& \mathbb{P}_{\mathcal{G},\mathcal{H}}\left[h_{\mathrm{out}}= h, \{M_i\}_{i=1}^k = \{m_i\}_{i=1}^k  \mid x(i)=X(i),\forall i\in [k]\right] \nonumber
	\\ & = \mathbb{P}_{\mathcal{G},\sigma(\mathcal{H})}\left\{ 
	h_{\mathrm{out}}=\sigma^{-1}(h), \{M_i\}_{i=1}^k = \{m_i\}_{i=1}^k \mid \sigma^{-1}(x(i)) =X(i),\forall i\in [k]\right\} .\label{eq:sym}
\end{align}
\end{lemma}
\begin{proof}
Let $\mathcal{H}'= \sigma(\mathcal{H})$. Let $h_p(\cdot)$ denote the $p$-th hypothesis in the hypothesis set. Then one has
\begin{align}
	& \mathbb{P}_{\mathcal{G},\mathcal{H}}\left\{ h_{\mathrm{out}}= h_{p}(\mathcal{H}), \{M_i\}_{i=1}^k = \{m_i\}_{i=1}^k  \mid x(i)=X(i),\forall i\in [k]\right\} \nonumber
	\\ & =  \mathbb{P}_{\mathcal{G},\mathcal{H}}\left\{ h_{\mathrm{out}}= h_{p}(\mathcal{H}), \{M_i\}_{i=1}^k = \{m_i\}_{i=1}^k  \mid  \{\{x^{l}_{j_{h_{p'}(\mathcal{H}),i}}(i')\}_{p'=1,i=1}^{|\mathcal{H}|,k} \}_{l=1}^{m_{i'}} \}_{i'=1}^k= X\right\} \nonumber
	\\ & =  \mathbb{P}_{\mathcal{G},\mathcal{H}'}\left\{ h_{\mathrm{out}}= h_{p}(\mathcal{H}'), \{M_i\}_{i=1}^k = \{m_i\}_{i=1}^k  \mid  \{\{x^{l}_{j_{h_{p'}(\mathcal{H}'),i}}(i')\}_{p'=1,i=1}^{|\mathcal{H}'|,k} \}_{l=1}^{m_{i'}} \}_{i'=1}^k= X\right\} \label{eq:lcx1}
	\\ & =\mathbb{P}_{\mathcal{G},\sigma(\mathcal{H})}\left\{ h_{\mathrm{out}}= h_{p}(\sigma(\mathcal{H})), \{M_i\}_{i=1}^k = \{m_i\}_{i=1}^k  \mid \sigma^{-1}(x(i)) =X(i),\forall i\in [k]\right\}  \nonumber
	\\ & = \mathbb{P}_{\mathcal{G},\sigma(\mathcal{H})}\left\{ h_{\mathrm{out}}= \sigma^{-1}(h_{p}(\mathcal{H})), \{M_i\}_{i=1}^k = \{m_i\}_{i=1}^k  \mid \sigma^{-1}(x(i)) =X(i),\forall i\in [k]\right\}  .\label{eq:lcx2}
\end{align}
Here, $\eqref{eq:lcx1}$ holds since the algorithm $\mathcal{G}$ cannot distinguish $\mathcal{H}$ from $\mathcal{H}'$  using its own randomness.
\end{proof}

\section{Proofs of auxiliary lemmas for Rademacher classes}\label{sec:mp_rad}

\subsection{Proof of Lemma~\ref{lemma:conrad}}
%

	In this subsection, we shall follow the notation adopted in the proof of Lemma~\ref{lemma:opth} (e.g., the dataset $\widetilde{\mathcal{S}}$, the data subset $\widetilde{\mathcal{S}}(n)$ and its independent copy  $\widetilde{\mathcal{S}}^+(n)$, the Rademacher random variables $\{\sigma_i^j\}$).

Consider any given $n = \{n_i\}_{i=1}^k$ obeying $n_i \geq 12\log(2k)$ for all $i\in [k]$, and any given $w\in \Delta(k)$. Let $\kappa = \min_i \frac{n_i}{w_i}$.   Recall that $(x_{i,j},y_{i,j})$ is the $j$-th sample from $\mathcal{D}_i$, and $\{(x_{i,j}^+,y_{i,j}^+\}$ are independent copies. 
 Define
\begin{align*}
F(n,w) & \coloneqq\mathop{\mathbb{E}}_{\widetilde{\mathcal{S}}(n)}\left[\max_{h\in\mathcal{H}}\left(\sum_{i=1}^{k}\frac{w_{i}}{n_{i}}\sum_{j=1}^{n_{i}}\ell\big(h,(x_{i,j},y_{i,j})\big)-\sum_{i=1}^{k}w_{i}L(h,\basis_{i})\right)\right],
\end{align*}
which can be upper bounded by
\begin{align}
F(n,w) & =\mathop{\mathbb{E}}_{\widetilde{\mathcal{S}}(n)}\left[\max_{h\in\mathcal{H}}\left(\sum_{i=1}^{k}\frac{w_{i}}{n_{i}}\sum_{j=1}^{n_{i}}\left(\ell\big(h,(x_{i,j},y_{i,j})\big)-\mathop{\mathbb{E}}_{\widetilde{\mathcal{S}}^{+}(n)}\left[\ell\big(h,(x_{i,j}^{+},y_{i,j}^{+})\big)\right]\right)\right)\right]\notag\\
 & \leq\mathop{\mathbb{E}}_{\widetilde{\mathcal{S}}(n),\widetilde{\mathcal{S}}^{+}(n)}\left[\max_{h\in\mathcal{H}}\left(\sum_{i=1}^{k}\frac{w_{i}}{n_{i}}\sum_{j=1}^{n_{i}}\left(\ell\big(h,(x_{i,j},y_{i,j})\big)-\ell\big(h,(x_{i,j}^{+},y_{i,j}^{+})\big)\right)\right)\right]\nonumber\\
 & =\mathop{\mathbb{E}}_{\widetilde{\mathcal{S}}(n),\widetilde{\mathcal{S}}^{+}(n),\{\{\sigma_{i}^{j}\}_{j=1}^{n_{i}}\}_{i=1}^{k}}\left[\max_{h\in\mathcal{H}}\left(\sum_{i=1}^{k}\frac{w_{i}}{n_{i}}\sum_{j=1}^{n_{i}}\sigma_{i}^{j}\left\{ \ell\big(h,(x_{i,j},y_{i,j})\big)-\ell\big(h,(x_{i,j}^{+},y_{i,j}^{+})\big)\right\} \right)\right]\nonumber\\
 & \leq2\mathop{\mathbb{E}}_{\widetilde{\mathcal{S}}(n),\{\{\sigma_{i}^{j}\}_{j=1}^{n_{i}}\}_{i=1}^{k}}\left[\max_{h\in\mathcal{H}}\left(\sum_{i=1}^{k}\frac{1}{\kappa}\sum_{j=1}^{n_{i}}\sigma_{i}^{j}\ell\big(h,(x_{i,j},y_{i,j})\big)\right)\right] \notag\\
 & = 2\frac{\sum_{i=1}^{k}n_{i}}{\kappa}\widetilde{\mathrm{\mathsf{Rad}}}_{\{n_{i}\}_{i=1}^{k}}. 
	\label{eq:rere}
\end{align}
Here, the first inequality arises from Jensen's inequality, whereas the penultimate line applies Lemma~\ref{lemma:add12}.

Invoking the Mcdiarmid inequality (see Lemma~\ref{lemma:mcinequality}) with the choice $c = 1/\kappa$, we obtain
\begin{align}
 & \mathbb{P}\left\{ \left|\max_{h\in\mathcal{H}}\left(\sum_{i=1}^{k}\frac{w_{i}}{n_{i}}\sum_{j=1}^{n_{i}}\ell\big(h,(x_{i,j},y_{i,j})\big)-\sum_{i=1}^{k}w_{i}L(h,\basis_{i})\right)-F(n,w)\right|\geq\varepsilon\right\} \leq2\exp\left(-\frac{2\kappa^{2}\varepsilon^{2}}{\sum_{i=1}^{k}n_{i}}\right).
	\label{eq:pppp2}
\end{align}
This taken together with \eqref{eq:rere} reveals that: for any $\delta'\in (0,1]$, with probability at least $1-\delta'$ we have
\begin{align}
&\max_{h\in \mathcal{H}}\left( \sum_{i=1}^k \frac{w_i}{n_i}\ell(h,(x_{i,j},y_{i,j})) - \sum_{i=1}^k w_i L(h,\basis_i) \right) 
 \leq 2 \frac{\sum_{i=1}^k n_i}{\kappa}\widetilde{\mathsf{Rad}}_{ \{n_i\}_{i=1}^k }+ \frac{\sum_{i=1}^k n_i}{\kappa}\sqrt{\frac{\log(2/\delta')}{2\sum_{i=1}^k n_i}}.
\label{eq:pppp3}
\end{align}
Evidently, this inequality continues to hold if we replace $(\ell,L)$ with $(-\ell,-L)$. As a consequence, with probability at least $1-2\delta'$ one has
\begin{align}
&\max_{h\in \mathcal{H}}\left| \sum_{i=1}^k \frac{w_i}{n_i}\ell\big(h,(x_{i,j},y_{i,j})\big) - \sum_{i=1}^k w_i L(h,\basis_i) \right|
 \leq
2 \frac{\sum_{i=1}^k n_i}{\kappa}\widetilde{\mathsf{Rad}}_{ \{n_i\}_{i=1}^k }+ \frac{\sum_{i=1}^k n_i}{\kappa}\sqrt{\frac{\log(2/\delta')}{2\sum_{i=1}^k n_i}}.
\label{eq:pppp4}
\end{align}

Now,  fix $\kappa\geq 0$, and define 
$$
	\widetilde{\mathcal{L}} =\left\{  n =\{n_i\}_{i=1}^k, w = \{w_i\}_{i=1}^k\in \Delta_{\varepsilon_1/(8k)}(k) \mid  \Tone w_i \leq 2n_i, 12\log(2k)\leq n_i \leq \Tone, \forall i\in [k], \sum_{i=1}^k n_i \leq 2\Tone \right\},
$$ 
where $\Delta_{\varepsilon_1/(8k)}(k)$ (i.e., an $\varepsilon_1/(8k)$-net of $\Delta(k)$) has been defined in Appendix~\ref{sec:proof-lemma:opth}. 
Inequality \eqref{eq:pppp4} combined with the union bound tells us that: for any $\delta'>0$, with probability at least $1-\delta'$ 
\begin{align}
	&\max_{h\in \mathcal{H}}\left| \sum_{i=1}^k \sum_{j=1}^{n_i} \frac{w_i}{n_i}\ell\big(h,(x_{i,j},y_{i,j})\big) - \sum_{i=1}^k w_i L(h,\basis_i) \right| \notag\\
	&\qquad \leq  \frac{\sum_{i=1}^k n_i}{\Tone/2}\widetilde{\mathsf{Rad}}_{ \{n_i\}_{i=1}^k } 
+ \frac{\sum_{i=1}^k n_i}{\Tone/2}\sqrt{\frac{\log(|\widetilde{\mathcal{L}}|)+\log(2/\delta')}{2\sum_{i=1}^k n_i}} \nonumber
 \\ &\qquad  \leq 4\widetilde{\mathsf{Rad}}_{ \{n_i\}_{i=1}^k }
+4 \sqrt{\frac{\log(2|\widetilde{\mathcal{L}}|)+\log(2/\delta')}{\Tone}} \nonumber
 \\ &\qquad  \leq 4\widetilde{\mathsf{Rad}}_{ \{n_i\}_{i=1}^k }
+4 \sqrt{\frac{2k\log(16k\Tone/\varepsilon_1)+\log(2/\delta')}{\Tone}} \nonumber
\\ &\qquad  \leq 600 C_{\Tone}
+4 \sqrt{\frac{2k\log(16k\Tone/\varepsilon_1)+\log(2/\delta')}{\Tone}} \nonumber
\end{align}
holds simultaneously for all $\{n,w\} \in \widetilde{\mathcal{L}}$; see the use of the union bound in Appendix~\ref{sec:proof-lemma:opth} too. 
Here, we have made use of Assumption~\ref{assump:rad}, Lemma~\ref{fact:wrc}, Lemma~\ref{lemma:bdwrc} and the fact that $\sum_{i=1}^k n_i \geq \Tone/2$.

Note that in Algorithm~\ref{alg:rad}, we take
\begin{align}
\widehat{L}^t(h,w^t) =  \sum_{i=1}^k \frac{w_i^t}{n_i^{t,\mathsf{rad}} }\cdot \sum_{j=1}^{n_i^{t,\mathsf{rad}}} \ell(h,(x_{i,j},y_{i,j})).\nonumber
\end{align}
Given our choice that $n_i^{t,\mathsf{rad}} = \min\{ \left\lceil \Tone w_i^t + 12\log(2k) \right\rceil ,\Tone\}$ for $i\in [k]$, 
we see that 
$$
	\Tone w_i^t \leq n_i^{t,\mathsf{rad}}-1 \qquad  \text{and} \qquad 12\log(2k)\leq n_i^{t,\mathsf{rad}} \leq \Tone
$$
for all $i\in [k]$. In addition, it is seen from our choice of $n_i^{t,\mathsf{rad}}$ and $\Tone$ that
$$
	\sum_{i=1}^k n_i^{t,\mathsf{rad}} \leq \sum_{i=1}^k \left\lceil \Tone w_i^t + 12\log(2k) \right\rceil \leq \Tone + k + 12k\log(2k)\leq 2\Tone -2.
$$ 
Therefore, there exists some $\widetilde{w}^t\in \Delta(k)$ satisfying 
$$
	\big\{ \{n_i^{t,\mathsf{rad}}\}_{i=1}^k, \widetilde{w}^t \big\}\in \widetilde{\mathcal{L}}
	\qquad \text{and}\qquad 
	\|w^t-\widetilde{w}^t\|_1 \leq \frac{\varepsilon_1}{8k}
$$
for each $1\leq t \leq T$. Taking $\delta' = \delta/4$, we obtain that with probability at least $1-\delta/4$, 
\begin{align}
 \max_{h\in\mathcal{H}}\big|\widehat{L}^{t}(h,w^{t})-L(h,w^{t})\big| & \leq\max_{h\in\mathcal{H}}\big|\widehat{L}^{t}(h,\widetilde{w}^{t})-L(h,\widetilde{w}^{t})\big|+\max_{h\in\mathcal{H}}\big|\widehat{L}^{t}(h,\widetilde{w}^{t})-\widehat{L}^{t}(h,w^{t})\big|\notag\\
	&\qquad \qquad +\max_{h\in\mathcal{H}}\big|L(h,\widetilde{w}^{t})-L(h,w^{t})\big| \notag\\
 & \leq600C_{\Tone}+4\sqrt{\frac{2k\log(16k\Tone/\varepsilon_{1})+\log(2/\delta')}{\Tone}}+\frac{\varepsilon_{1}}{8k}+\frac{\varepsilon_{1}}{8k} \notag\\
 & \leq\frac{\varepsilon_{1}}{2}\nonumber
\end{align}
for any $1\leq t\leq T$, where the last inequality results from the definition of $\Tone$.

Finally, the fact that $h^t =\arg\min_{h\in \mathcal{H}}\widehat{L}^t(h,w^t)$ allows one to derive
\begin{align}
 L(h^t,w^t)\leq \widehat{L}^t(h^t,w^t)+\frac{\varepsilon_1}{2}= \min_{h\in \mathcal{H}} \widehat{L}^t(h,w^t) + \frac{\varepsilon_1}{2}\leq \min_{h\in \mathcal{H}}L(h,w^t)+\varepsilon_1,\nonumber
\end{align}
which concludes the proof.

\subsection{Proof of Lemma~\ref{fact:wrc}}
\label{sec:proof-fact:wrc}
%
%
In what follows, assume that each $z_i^j$ obeys $z_i^j\sim \mathcal{D}_i$, and each $\sigma_i^j$ is a zero-mean Rademacher random variable.  
Direct computation then gives 
\begin{align}
	& \left(\sum_{i=1}^{k}n_{i}\right)\widetilde{\mathsf{Rad}}_{\{n_{i}\}_{i=1}^{k}}
  =\mathop{\mathbb{E}}\limits _{\{z_{i}^{j}\}_{j=1}^{n_{i}},\forall i\in[k]}\left[\mathop{\mathbb{E}}\limits _{\{\sigma_{i}^{j}\}_{j=1}^{n_{i}},\forall i\in[k]}\left[\max_{h\in\mathcal{H}}\sum_{i=1}^{k}\sum_{j=1}^{n_{i}}\sigma_{i}^{j}\ell(h,z_{i}^{j})\right]\right]\nonumber\\
 &\qquad \overset{\mathrm{(i)}}{=}\mathop{\mathbb{E}}\limits _{\{z_{i}^{j}\}_{j=1}^{n_{i}},\forall i\in[k]}\left[\mathop{\mathbb{E}}\limits _{\{\sigma_{i}^{j}\}_{j=1}^{n_{i}},\forall i\in[k]}\left[\max_{h\in\mathcal{H}}\mathop{\mathbb{E}}\limits _{\{z_{i}^{j}\}_{j=n_{i}+1}^{n_{i}+m_{i}},\{\sigma_{i}^{j}\}_{j=n_{i}+1}^{n_{i}+m_{i}},\forall i\in[k]}\left[\sum_{i=1}^{k}\sum_{j=1}^{n_{i}+m_{i}}\sigma_{i}^{j}\ell(h,z_{i}^{j})\right]\right]\right]\nonumber\\
 &\qquad \overset{\mathrm{(ii)}}{\leq}\mathop{\mathbb{E}}\limits _{\{z_{i}^{j}\}_{j=1}^{n_{i}+m_{i}},\forall i\in[k]}\left[\mathop{\mathbb{E}}\limits _{\{\sigma_{i}^{j}\}_{j=1}^{n_{i}+m_{i}},\forall i\in[k]}\left[\max_{h\in\mathcal{H}}\sum_{i=1}^{k}\sum_{j=1}^{n_{i}+m_{i}}\sigma_{i}^{j}\ell(h,z_{i}^{j})\right]\right]\nonumber\\
 &\qquad =\left(\sum_{i=1}^{k}(n_{i}+m_{i})\right)\widetilde{\mathsf{Rad}}_{\{n_{i}+m_{i}\}_{i=1}^{k}}\nonumber\\
	&\qquad \overset{\mathrm{(iii)}}{\leq}\mathop{\mathbb{E}}\limits _{\{z_{i}^{j}\}_{j=1}^{n_{i}},\forall i\in[k]}\left[\mathop{\mathbb{E}}\limits _{\{\sigma_{i}^{j}\}_{j=1}^{n_{i}},\forall i\in[k]}\left[\max_{h\in\mathcal{H}}\sum_{i=1}^{k}\sum_{j=1}^{n_{i}}\sigma_{i}^{j}\ell(h,z_{i}^{j}) \right]\right]\nonumber\\
	&\qquad \qquad\qquad\qquad\qquad+\mathop{\mathbb{E}}\limits _{\{z_{i}^{j}\}_{j=n_{i}+1}^{n_{i}+m_{i}},\forall i\in[k]}\left[\mathop{\mathbb{E}}\limits _{\{\sigma_{i}^{j}\}_{j=n_{i}+1}^{n_{i}+m_{i}},\forall i\in[k]}\left[\max_{h\in\mathcal{H}}\sum_{i=1}^{k}\sum_{j=n_{i}+1}^{n_{i}+m_{i}}\sigma_{i}^{j}\ell(h,z_{i}^{j})\right]\right] \notag\\
 &\qquad =\left(\sum_{i=1}^{k}n_{i}\right)\widetilde{\mathsf{Rad}}_{\{n_{i}\}_{i=1}^{k}}+\left(\sum_{i=1}^{k}m_{i}\right)\widetilde{\mathsf{Rad}}_{\{m_{i}\}_{i=1}^{k}}.\nonumber
\end{align}
Here, (i) is valid due to the zero-mean property of $\{\sigma_i^j\}$, 
	(ii) comes from Jensen's inequality,  and (iii) follows since $\max_x \big(f_1(x)+f_2(x)\big)\leq \max_x f_1(x) + \max_x f_2(x) $.

\subsection{Proof of Lemma~\ref{lemma:bdwrc}}
\label{sec:proof-fact:bdwrc}


Set $n = \sum_{i=1}^k n_i$.
	Let $\{X_j\}_{j=1}^n$ be $n$ i.i.d.~multinomial random variables with parameter $\{w_i\}_{i=1}^k$, and take $$\widehat{n}_i = \sum_{j=1}^n \mathds{1}\{X_j = i\}$$ for each $i\in [k]$.  
From \eqref{eq:defn-Dw} and Definition~\ref{def:rad}, it is easily seen that 
\begin{align}
	\mathsf{Rad}_{n}\big(D(w)\big) & =\mathop{\mathbb{E}}\limits _{\{X_{i}\}_{i=1}^{n}}\left[\mathop{\mathbb{E}}\limits _{\{z_{i}^{j}\}_{j=1}^{\widehat{n}_{i}},\forall i\in[k]}\left[\mathop{\mathbb{E}}\limits _{\{\sigma_{i}^{j}\}_{j=1}^{\widehat{n}_{i}},\forall i\in[k]}\left[\frac{1}{n}\max_{h\in\mathcal{H}}\sum_{i=1}^{k}\sum_{j=1}^{\widehat{n}_{i}}\sigma_{i}^{j}\ell\big(h,z_{i}^{j}\big)\right]\right]\right],\nonumber	
\end{align}
where each $z_i^j$ is independently drawn from $\mathcal{D}_i$, and each $\sigma_i^j$ is an independent  Rademacher random variable.

In addition, Lemma~\ref{lemma:con1} tells us that: for any $i\in [k]$, one has
\begin{align}
\widehat{n}_{i}\geq\frac{1}{3}n_{i}-2\log(2k)\geq\frac{1}{6}n_{i}\qquad\Longrightarrow\qquad\widehat{n}_{i}\geq\left\lceil \frac{1}{6}n_{i}\right\rceil \eqqcolon\widetilde{n}_{i} 
	\label{eq:s}
\end{align}
	with probability exceeding $1-1/(2k)$. 
Defining $\mathcal{E}$ to be the event that $\widehat{n}_i \geq n_i/6$ holds for all $i\in [k]$, 
we can invoke the union bound to see that 
$$
	\mathbb{P}(\mathcal{E})\geq 1/2.
$$
Consequently, we can derive
\begin{align}
\mathsf{Rad}_{n}\big(\mathcal{D}(w)\big) & \geq\mathbb{P}(\mathcal{E})\cdot\mathop{\mathbb{E}}\limits _{\{X_{i}\}_{i=1}^{n}}\left[\mathop{\mathbb{E}}\limits _{\{z_{i}^{j}\}_{j=1}^{\widehat{n}_{i}},\forall i\in[k]}\left[\mathop{\mathbb{E}}\limits _{\{\sigma_{i}^{j}\}_{j=1}^{\widehat{n}_{i}},\forall i\in[k]}\left[\frac{1}{n}\max_{h\in\mathcal{H}}\sum_{i=1}^{k}\sum_{j=1}^{\widehat{n}_{i}}\sigma_{i}^{j}\ell\big(h,z_{i}^{j}\big)\mid\mathcal{E}\right]\right]\right]\nonumber\\
 & \geq\frac{1}{2}\cdot\mathop{\mathbb{E}}\limits _{\{z_{i}^{j}\}_{j=1}^{\widehat{n}_{i}},\forall i\in[k]}\left[\mathop{\mathbb{E}}\limits _{\{\sigma_{i}^{j}\}_{j=1}^{\widehat{n}_{i}},\forall i\in[k]}\left[\frac{1}{n}\max_{h\in\mathcal{H}}\sum_{i=1}^{k}\sum_{j=1}^{\widetilde{n}_{i}}\sigma_{i}^{j}\ell\big(h,z_{i}^{j}\big)\mid\mathcal{E}\right]\right]\nonumber\\
 & =\frac{1}{2}\cdot\frac{\sum_{i=1}^{k}\widetilde{n}_{i}}{n}\widetilde{\mathsf{Rad}}_{\{\widetilde{n}_{i}\}_{i=1}^{k}}\nonumber\\
 & \geq\frac{1}{12}\cdot\frac{1}{6}\widetilde{\mathsf{Rad}}_{\{n_{i}\}_{i=1}^{k}}, 
\end{align}
%
thus concluding the proof.
%

\subsection{Necessity of Assumption~\ref{assump:rad}}\label{app:radlb}

In this subsection, 
we study whether Assumption~\ref{assump:rad} can be replaced by the following weaker assumption, the latter of which only assumes that the Rademacher complexity on each $\mathcal{D}_i$ is well-bounded.
\begin{assum}\label{assump:com}
For each $n\geq 1$, there exists a quantity $\widetilde{C}_n> 0$ (known to the learner) such that
\begin{align}
	\widetilde{C}_n \geq \max_{1\leq i\leq k} \mathsf{Rad}_n(\mathcal{D}_i).
\end{align}
\end{assum}

Formally, we have the following results.
\begin{lemma}\label{lemma:rad11} Let $w^0 = [1/k,1/k,\ldots, 1/k]^{\top}$. 
There exist a group of distributions $\{\mathcal{D}_i\}_{i=1}^k$ and a hypothesis set $\mathcal{H}$ such that
\begin{align}
\mathsf{Rad}_n \big(\mathcal{D}(w^0) \big)\geq \Omega \left(\frac{1}{k}\sum_{i=1}^k \mathsf{Rad}_{n/k}(\mathcal{D}_i) \right)
\end{align}
for $n\geq 12k\log(k)$.
\end{lemma}

\begin{proof}
Without loss of generality, consider the case where $\mathcal{Y}=\{0\}$ and $\ell(h,(x,y)) =h(x)-y=h(x)$. We can then view $\mathcal{D}_i$ as a distribution over $\mathcal{X}_i$, as there is only one element in $\mathcal{Y}$.

Pick $k$  subsets of $\mathcal{X}$ as $\{\mathcal{X}_i \}_{i=1}^k$. For each $i\in [k]$, we choose the distribution $\mathcal{D}_i$ to be an arbitrary distribution supported on $\mathcal{X}_i$. In addition, we define $\mathcal{H}_i$ to be a set of hypothesis obeying $h(x)=0$ for all $x\notin \mathcal{X}_i$ for each $i\in [k]$. For a collection of hypothesis $\{h_i\}_{i=1}^k$ such that $h_i\in \mathcal{H}_i$, we define $\mathsf{joint}( \{h_i\}_{i=1}^k)$ to be the hypothesis $h$ such that 
\begin{align}
	h(x) =\begin{cases} h_i(x) &\text{if }x\in \mathcal{X}_i,~i\in[k]; \\ 0 &\text{if }x\notin \cup_i \mathcal{X}_i.\end{cases}
\end{align}
%
%
The hypothesis set $\mathcal{H}$ is then constructed as
\begin{align}
\mathcal{H} = \big\{ \mathsf{joint}(\{h_i\}_{i=1}^k) \mid  h_i\in \mathcal{H}_i, \forall i\in [k] \big\}.
\end{align}

Recalling the definition of $\widetilde{\mathsf{Rad}}_{\{n_i\}_{i=1}^k  }$, we see that 
\begin{align}
\widetilde{\mathsf{Rad}}_{ \{n_i\}_{i=1}^k} &  =\mathop{\mathbb{E}}\limits _{\{x_{i}^{j}\}_{j=1}^{n_{i}},\forall i\in[k]}\left[\mathop{\mathbb{E}}\limits _{\{\sigma_{i}^{j}\}_{j=1}^{n_{i}},\forall i\in[k]}\left[\frac{1}{\sum_{i=1}^{k}n_{i}}\max_{h\in\mathcal{H}}\sum_{i=1}^{k}\sum_{j=1}^{n_{i}}\sigma_{i}^{j}h(x_i^j)\right]\right]  \nonumber
\\ &  = \mathop{\mathbb{E}}\limits _{\{x_{i}^{j}\}_{j=1}^{n_{i}},\forall i\in[k]}\left[\mathop{\mathbb{E}}\limits _{\{\sigma_{i}^{j}\}_{j=1}^{n_{i}},\forall i\in[k]}\left[\frac{1}{\sum_{i=1}^{k}n_{i}}\sum_{i=1}^{k}\max_{h_i \in \mathcal{H}_i}\sum_{j=1}^{n_{i}}\sigma_{i}^{j}h_i(x_i^j)\right]\right]  \label{eq:radlb1}
\\ & = \frac{1}{\sum_{i=1}^k n_i} \sum_{i=1}^{k} n_i \mathop{\mathbb{E}}\limits _{\{z_{i}^{j}\}_{j=1}^{n_{i}},\forall i\in[k]}\left[\mathop{\mathbb{E}}\limits _{\{\sigma_{i}^{j}\}_{j=1}^{n_{i}}}\left[\frac{1}{n_i}\max_{h_i \in \mathcal{H}_i}\sum_{j=1}^{n_{i}}\sigma_{i}^{j}h_i(x_i^j)\right]\right] \nonumber
\\ & =\frac{1}{\sum_{i=1}^k n_i}\sum_{i=1}^k n_i \mathsf{Rad}_{n_i}(\mathcal{D}_i),\nonumber
\end{align}
where \eqref{eq:radlb1} results from the definition of $\mathcal{H}$.
By taking $n_i = \frac{n}{k}$ for all $i\in [k]$ and applying Lemma~\ref{lemma:bdwrc}, we reach 
\begin{align}
 \frac{1}{k}\sum_{i=1}^k \mathsf{Rad}_{n/k}(\mathcal{D}_i) = \widetilde{\mathsf{Rad}}_{\{n_i\}_{i=1}^k }\leq 72 \mathsf{Rad}_{n}\big(\mathcal{D}(w^0)\big)
\end{align}
as claimed. 
\end{proof}

By virtue of Lemma~\ref{lemma:rad11}, if we set $\widetilde{C}_n =\widetilde{\Theta}( \sqrt{d/n})$ in Assumption~\ref{assump:com}, 
then the best possible upper bound on $\mathsf{Rad}_n(\mathcal{D}(w^0))$ is $\mathsf{Rad}_n(\mathcal{D}(w^0)) =\widetilde{\Theta}( \sqrt{dk/n} )$, which implies that more samples are needed to learn the mixed distribution $\mathcal{D}(w^0)$ than learning each individual distribution.
Moreover, under the construction in Lemma~\ref{lemma:rad11}, 
if we further assume that $\min_{h_i\in \mathcal{H}_i} \mathop{\mathbb{E}}_{x\sim \mathcal{D}_i}[h_i(x)]=1/2$ for all $i\in [k]$, then to find $h$ such that
\begin{align}
\max_{1\leq i\leq k}\mathop{\mathbb{E}}_{x\sim \mathcal{D}_i}[h(x)]\leq \frac{1}{2}+\varepsilon,
\end{align}
we need to find, for each $i\in[k]$, a hypothesis $h_i\in \mathcal{H}_i$ such that
\begin{align}
	\mathop{\mathbb{E}}_{x\sim \mathcal{D}_i}[h_i(x)] \leq \frac{1}{2}+\varepsilon. 
\end{align}
Following this intuition, we can construct a  counter example under Assumption~\ref{assump:com}, 
with a formal theorem stated as follows.
\begin{theorem}\label{thm:radlb}
	There exist a group of distributions $\{\mathcal{D}_i\}_{i=1}^k$ and a hypothesis set $\mathcal{H}$ such that Assumption~\ref{assump:com} holds with $\widetilde{C}_n = \widetilde{O}\left(\sqrt{\frac{d}{n}}\right)$, and it takes at least $\widetilde{\Omega}\left( \frac{dk}{\varepsilon^2} \right)$ samples to find some $h\in \mathcal{H}$ obeying
\begin{align}
\max_{i\in [k]}L(h,\basis_i)\leq \min_{h'\in \mathcal{H}}\max_{i\in [k]}L(h',\basis_i)+\varepsilon.\nonumber
\end{align}
\end{theorem}
\begin{proof}
With the construction in Lemma~\ref{lemma:rad11}, it suffices to find some $\mathcal{H}'$ and $\mathcal{D}'$ such that the following three conditions hold: 
\begin{enumerate}
\item The following inequality holds: 
\begin{align}
	\mathsf{Rad}_n(\mathcal{D}',\mathcal{H}') \coloneqq \frac{1}{n} \mathop{\mathbb{E}}_{ \{x^j\}_{j=1}^n \overset{\mathrm{i.i.d.}}{\sim } \mathcal{D}', \,\{\sigma^j\}_{j=1}^n \overset{\mathrm{i.i.d.}}{\sim } \{\pm 1\} }\left[\max_{h'\in \mathcal{H}'}\sum_{j=1}^n\sigma^j h'(x^j) \right] \leq \widetilde{C}_n\label{eq:radcond1} ;
\end{align}
\item $\min_{h'\in \mathcal{H}}\mathbb{E}_{x\sim \mathcal{D}'}[h(x)]=\frac{1}{2}$;
\item It takes at least $\widetilde{\Omega}\left(\frac{d}{\varepsilon^2}\right)$ samples to find some $h$ such that $\mathbb{E}_{x\sim \mathcal{D}'}[h(x)]\leq \frac{1}{2}+\varepsilon$.
\end{enumerate}

This construction is also straightforward. Set $N=2^d$ and $\mathcal{X}' = \{0,1\}^{N}$. Let $\mathcal{D}'$ be the distribution 
$$
	\mathbb{P}_{ \mathcal{D}'}\{x\} = \prod_{n=1}^N \mathbb{P}_{\mathcal{D}'_n}\{x_n\}
$$  
where
\begin{align}
	\mathbb{P}_{\mathcal{D}'_{n^*}}\{x_{n^*}\} &=  \frac{1}{2}\mathds{1}\{x_{n^*}=1\}+\frac{1}{2}\mathds{1}\{x_{n^*}=0\}  &&\text{for some }n^*; \\
		\mathbb{P}_{\mathcal{D}'_{n}}\{x_{n}\} &= 
		\left(\frac{1}{2}+2\varepsilon \right)\mathds{1}\{x_{n^*}=1\}+ \left(\frac{1}{2}-2\varepsilon \right)\mathds{1}\{x_{n^*}=0\}  &&\text{for all }n\neq n^*.
\end{align}
We then choose $\mathcal{H}'=\{h^n\}_{n=1}^N$ with $h^n(x)=x_n$ for each $n\in [N]$. It is then easy to verify that the first two conditions hold. Regarding the third condition,
following the arguments in Theorem~\ref{thm:lb}, we need at least $\widetilde{\Omega}\left(d/{\varepsilon^2} \right)$ i.i.d. samples from $\mathcal{D}'$ to identify $n^*$. The proof is thus completed.
\end{proof}

\section{Proof for multi-loss multi-distribution learning in Theorem~\ref{thm:multi-obj}}\label{app:multiobj}


The proof of Theorem~\ref{thm:multi-obj} is similar as that of Theorem~\ref{thm:main}, except that we need to establish uniform convergence for the new loss estimators in Algorithm~\ref{alg:obj}. 

In Algorithm~\ref{alg:obj}, we re-define the total stepsize and auxiliary step size as 
\begin{align}
T\coloneqq\frac{20000\log\left(\frac{kR}{\delta\varepsilon}\right)}{\varepsilon^2}, \qquad \qquad \qquad T_1 \coloneqq
 \frac{40000\left(k\log ( \frac{kR}{\varepsilon_1} )+d\log\big( (\frac{kd}{\varepsilon_1} )+\log(\frac{1}{\delta}) \big) \right)}{\varepsilon_1^2}.\nonumber
 \end{align}

\subsection{Main steps of the proof}

Let us begin by presenting the key lemmas needed to establish Theorem~\ref{thm:multi-obj}. 
Recall that $u^{t}=\{u_{i,\ell}^t\}_{\ell\in \mathcal{L}, i\in [k]}$. For any $i\in [k]$, $u\in \Delta([k]\times \mathcal{L})$, $h\in \mathcal{H}$ and $\ell\in \mathcal{L}$, recall that 
\begin{align}
L_i^{\ell}(h) = \mathop{\mathbb{E}}\limits_{(x,y)\sim \mathcal{D}_i} \big[\ell\big(h, (x,y)\big) \big] \qquad \text{and}\qquad L(h,u) =  \sum_{i,\ell}u_{i,\ell} L_i^{\ell}(h).
\end{align}
Given $\pi\in \Delta(\mathcal{H})$, we define $L_i^{\ell}(h_{\pi})  = \mathbb{E}_{h\sim \pi}[L_i^{\ell}(h)]$.
The first step is to establish the goodness of the empirical minimizer $h^t$, akin to Lemma~\ref{lemma:opth}. 
\begin{lemma}\label{lemma:opth_multi}
With probability at least $1-\delta/4$, it holds that
\begin{align}
L(h^t,u^t)\leq \min_{h\in \mathcal{H}}L(h,u^t)+\varepsilon_1
\end{align}
for any $1\leq t\leq T$, 
where $h^t$ (resp.~$u^t$) is the hypothesis (resp.~weight vector) computed in round $t$ of Algorithm~\ref{alg:obj}.
\end{lemma}

Similar to Lemma~\ref{lemma:opt}, the next step is to prove that the output hypothesis $h^{\mathsf{final}}$ is $\varepsilon$-optimal. 
\begin{lemma}\label{lemma:opt_multi}
With probability at least $1-\delta/2$, the output policy $h^{\mathsf{final}}$ is $\varepsilon$-optimal in the sense that
\begin{align}
\max_{i\in [k],\ell\in \mathcal{L}}\frac{1}{T}\sum_{t=1}^T L^{\ell}_i(h^t)\leq \min_{h\in \mathcal{H}}\max_{i\in [k],\ell\in \mathcal{L}}L^{\ell}_i(h)+\varepsilon.\nonumber
\end{align}
\end{lemma}


Furthermore, the following lemma  upper bounds the sample complexity of the proposed algorithm. 
\begin{lemma}\label{lemma:sc_obj}
With probability at least $1-\delta/2$, the sample complexity of Algorithm~\ref{alg:obj} is bounded by 
\begin{align}
O\left( \frac{\left(d\log\left(\frac{d}{\varepsilon}\right)+k\log\left(\frac{kR}{\delta\varepsilon}\right)\right) \min\{\log(R),k\}}{\varepsilon^2} \cdot \log^5(k)\log^3\left( \frac{k\log(R)}{\delta\varepsilon}\right)  \right).
\label{eq:sample-complexity-lemma-obj}
\end{align}
\end{lemma}

Armed with the above lemmas, we can readily establish Theorem~\ref{thm:multi-obj}. 
From Lemma~\ref{lemma:opth_multi} and Lemma~\ref{lemma:opt_multi}, we know that with probability exceeding $1-\delta$, the output hypothesis $h^{\mathsf{final}}$ is $\varepsilon$-optimal. Then by virtue of Lemma~\ref{lemma:sc_obj}, 
the sample complexity is no greater than \eqref{eq:sample-complexity-lemma-obj}, as advertised in the theorem.
The rest of this section is thus dedicated to proving the above lemmas.



\subsection{Proof of Lemma~\ref{lemma:opth_multi}}

By definition of $h^t$, it suffices to show that
\begin{align}
\left| \widehat{L}^t(h,u^t)-L(h,u^t) \right|\leq \frac{1}{2}\varepsilon_1
\end{align}
holds simultaneously for every $h\in \mathcal{H}$ and $1\leq t\leq T$. 
Take 
\begin{equation}
\widehat{L}^{\ell}_i(h) =  \frac{1}{n_i^t} \sum_{j=1}^{n_i^t} \ell(h, (x_{i,j}, y_{i,j}) ) .
\end{equation} 
By definition,  we have 
\begin{align}
 \widehat{L}^t(h,u^t)-L(h,u^t)   &  =\sum_{i,\ell}u_{i,\ell}^t \widehat{L}^{\ell}_i(h)
 -\sum_{i,\ell}u^t_{i,\ell}L^{\ell}_i(h) 
 = \sum_{i,\ell}u_{i,\ell}^t \big( \widehat{L}^{\ell}_i(h) - L^{\ell}_i(h)\big) \nonumber
 \\ & \leq \sum_{i=1}^k \left(\sum_{\ell\in \mathcal{L}}u_{i,\ell}^t \right) \cdot \max_{\ell\in \mathcal{L}}\big(\widehat{L}^{\ell}_i(h)-L^{\ell}_i(h)\big),
 \label{eq:rxu-obj}
\end{align}
and similarly, we can also lower bound
\begin{align}
& \widehat{L}^t(h,u^t) - L(h,u^t)\geq \sum_{i=1}^k \left(\sum_{\ell\in \mathcal{L}}u_{i,\ell}^t\right) \cdot \min_{\ell \in \mathcal{L}}\big( \widehat{L}^{\ell}_i(h)-L^{\ell}_i(h) \big).\label{eq:lxu-obj}
\end{align}
Consequently, it boils down to developing a uniform upper (resp.~lower) bound on the right-hand side of \eqref{eq:rxu-obj} (resp.~\eqref{eq:lxu-obj}).

To begin with, let us fix any $n=\{n_i\}_{i=1}^k\in \mathbb{N}^k$, $\{w_i\}_{i=1}^k\in \Delta(k) $ and $\{m_i\}_{i=1}^k \in [R]^k$. 
Recall the definitions of $(x_{i,j},y_{i,j})$,  $(x_{i,j}^+,y^+_{i,j})$, $\widetilde{\mathcal{S}}(n)$, $\widetilde{\mathcal{S}}^+(n)$, $\mathcal{C}$, and $\sigma(n)=\big\{\{\sigma_{i,j}\}_{j=1}^{n_i}\big\}_{i=1}^k$ in the proof of Lemma~\ref{lemma:opth} (see Appendix~\ref{sec:proof-lemmas-VC}).  
For any $\lambda \in [0,\min_{i}\frac{n_i}{w_i}]$, define 
\begin{align}
 E\left(\lambda, \{n_i\}_{i=1}^k, \{w_i\}_{i=1}^k, \{m_i\}_{i=1}^k \right) \coloneqq \mathop{\mathbb{E}}\limits_{\widetilde{\mathcal{S}}(n)}\left[\exp\left(\lambda \cdot \max_{h\in \mathcal{H}}\left(\sum_{i=1}^k w_i \cdot \left(   \frac{1}{n_i}\sum_{j=1}^{n_i} \ell^{m_i}\big(h,(x_{i,j},y_{i,j}) \big)   - L^{\ell_{m_i}}_i (h)   \right)\right)\right)\right].\nonumber
\end{align}
%
Repeating the same symmetrization arguments in the proof of Lemma~\ref{lemma:opth}, we can derive
\begin{align}
 & E\left(\lambda,\{n_{i}\}_{i=1}^{k},\{w_{i}\}_{i=1}^{k},\{m_{i}\}_{i=1}^{k}\right)\nonumber\\
 & =\mathop{\mathbb{E}}\limits_{\widetilde{\mathcal{S}}(n)}\left[\max_{h\in\mathcal{H}}\exp\left(\lambda\sum_{i=1}^{k}\frac{w_{i}}{n_{i}}\sum_{j=1}^{n_{i}}\left(\ell^{m_{i}}\big(h,(x_{i,j},y_{i,j})\big)-\mathop{\mathbb{E}}\limits_{(x_{i,j}^{+},y_{i,j}^{+})}\big[\ell^{m_{i}}\big(h,(x_{i,j}^{+},y_{i,j}^{+})\big)\big]\right)\right)\right]\nonumber\\
 & \leq\mathop{\mathbb{E}}\limits_{\widetilde{\mathcal{S}}(n)}\left[\max_{h\in\mathcal{H}}\mathop{\mathbb{E}}\limits_{\widetilde{\mathcal{S}}^{+}(n)}\left[\exp\left(\lambda\sum_{i=1}^{k}\frac{w_{i}}{n_{i}}\sum_{j=1}^{n_{i}}\left\{ \ell^{m_{i}}\big(h,(x_{i,j},y_{i,j})\big)-\ell^{m_{i}}\big(h,(x_{i,j}^{+},y_{i,j}^{+})\big)\right\} \right)\right]\right]\nonumber\\
 & \leq\mathop{\mathbb{E}}\limits_{\widetilde{\mathcal{S}}(n),\widetilde{\mathcal{S}}^{+}(n)}\left[\max_{h\in\mathcal{H}}\exp\left(\lambda\sum_{i=1}^{k}\frac{w_{i}}{n_{i}}\sum_{j=1}^{n_{i}}\left\{ \ell^{m_{i}}\big(h,(x_{i,j},y_{i,j})\big)-\ell^{m_{i}}\big(h,(x_{i,j}^{+},y_{i,j}^{+})\big)\right\} \right)\right]\nonumber\\
 & =\mathop{\mathbb{E}}\limits_{\widetilde{\mathcal{S}}(n),\widetilde{\mathcal{S}}^{+}(n)}\left[\mathop{\mathbb{E}}\limits_{\sigma(n)}\left[\max_{h\in\mathcal{H}_{\min,\mathcal{C}}}\exp\left(\lambda\sum_{i=1}^{k}\frac{w_{i}}{n_{i}}\sum_{j=1}^{n_{i}}\sigma_{i,j}\left\{ \ell^{m_{i}}\big(h,(x_{i,j},y_{i,j})\big)-\ell^{m_{i}}\big(h,(x_{i,j}^{+},y_{i,j}^{+})\big)\right\} \right)\bigg|\,\mathcal{C}\right]\right]\nonumber\\
 & \leq\mathop{\mathbb{E}}\limits_{\widetilde{\mathcal{S}}(n),\widetilde{\mathcal{S}}^{+}(n)}\left[\mathop{\mathbb{E}}\limits_{\sigma(n)}\left[|\mathcal{H}_{\min,\mathcal{C}}|\max_{h\in\mathcal{H}_{\min,\mathcal{C}}}\exp\left(\lambda\sum_{i=1}^{k}\frac{w_{i}}{n_{i}}\sum_{j=1}^{n_{i}}\sigma_{i,j}\left\{ \ell^{m_{i}}\big(h,(x_{i,j},y_{i,j})\big)-\ell^{m_{i}}\big(h,(x_{i,j}^{+},y_{i,j}^{+})\big)\right\} \right)\bigg|\,\mathcal{C}\right]\right]\nonumber\\
 & \leq\left(2kT_{1}+1\right)^{d}\mathop{\mathbb{E}}\limits_{\widetilde{\mathcal{S}}(n),\widetilde{\mathcal{S}}^{+}(n)}\left[\mathop{\mathbb{E}}\limits_{\sigma(n)}\left[\max_{h\in\mathcal{H}}\exp\left(\lambda\sum_{i=1}^{k}\frac{w_{i}}{n_{i}}\sum_{j=1}^{n_{i}}\sigma_{i,j}\left\{ \ell^{m_{i}}\big(h,(x_{i,j},y_{i,j})\big)-\ell^{m_{i}}\big(h,(x_{i,j}^{+},y_{i,j}^{+})\big)\right\} \right)\bigg|\,\mathcal{C}\right]\right].\label{eq:wxx1-obj}
\end{align}
For any given $\mathcal{C}$ and $h\in \mathcal{H}$, by observing that $\lambda \frac{w_i}{n_i}\leq 1$ we have (similar to the arguments for \eqref{eq:xxxx2})
\begin{align}
 & \mathop{\mathbb{E}}\limits_{\sigma(n)}\left[\exp\left(\lambda\sum_{i=1}^{k}\frac{w_{i}}{n_{i}}\sum_{j=1}^{n_{i}}\sigma_{i,j}\left\{ \ell^{m_{i}}\big(h,(x_{i,j},y_{i,j})\big)-\ell^{m_{i}}\big(h,(x_{i,j}^{+},y_{i,j}^{+})\big)\right\} \right)\bigg|\,\mathcal{C}\right]\nonumber\\
 & \qquad =\prod_{i=1}^{k}\prod_{j=1}^{n_{i}}\mathop{\mathbb{E}}\limits_{\sigma_{i,j}}\left[\lambda\sigma_{i,j}\cdot\frac{w_{i}}{n_{i}}\left\{ \ell^{m_{i}}\big(h,(x_{i,j},y_{i,j})\big)-\ell^{m_{i}}\big(h,(x_{i,j}^{+},y_{i,j}^{+})\big)\right\} \bigg|\,\mathcal{C}\right] \notag\\
 & \qquad \leq\exp\left(2\lambda^{2}\cdot\sum_{i=1}^{k}\frac{w_{i}^{2}}{n_{i}}\right).\label{eq:wxx2-obj}
\end{align}
Substituting \eqref{eq:wxx1-obj} into \eqref{eq:wxx2-obj} gives
\begin{align}
E\left(\lambda, \{n_i\}_{i=1}^k, \{w_i\}_{i=1}^k, \{m_i\}_{i=1}^k \right)\leq 
\left(2kT_{1}+1\right)^{d} \exp\left(2\lambda^2 \cdot \sum_{i=1}^k \frac{(w_i)^2}{n_i}\right).\label{eq:wxx3-obj}
\end{align}
It then follows that, for any $0<\varepsilon'<1$, 
\begin{align}
  & \mathbb{P}\left(\max_{h\in\mathcal{H}}\sum_{i=1}^{k}w_{i}\left\{ \frac{1}{n_{i}}\sum_{j=1}^{n_{i}}\ell^{m_{i}}\big(h,(x_{i,j},y_{i,j})\big)-L_{i}^{\ell^{m_{i}}}(h)\right\} \geq\varepsilon'\right)\nonumber\\
 & \qquad\leq \left(2kT_{1}+1\right)^{d}\min_{\lambda\in[0,\min_{i}\frac{n_{i}}{m_{i}}]}\exp\left(2\lambda^{2}\sum_{i=1}^{k}\frac{(w_{i})^{2}}{n_{i}}-\lambda\varepsilon'\right),\label{eq:wxx3.5-obj}
\end{align}
where we have repeated the arguments for \eqref{eq:lxs}.

 Next, for any $\kappa>0$, we define $\mathcal{B}(\kappa)$ to be the set of tuples as follows:
\begin{align}
    \mathcal{B}(\kappa) \coloneqq \left\{ n=\{n_i\}_{i=1}^k, w=\{w_i\}_{i=1}^k , m=\{m_i\}_{i=1}^k
 \mid n_i\geq \kappa w_i, n_i\in [T_1], m_i \in [R],\forall i\in[k], w\in \Delta_{\varepsilon_1/(8k)}(k)  \right\}.\nonumber
\end{align}
Then it can be easily verified that
$$
|\mathcal{B}(\kappa)|\leq (T_1 R)^k \cdot |\Delta_{\varepsilon_1/(8k)}(k)|\leq \left(\frac{8kT_1 R}{\varepsilon_1}\right)^k
$$ 
for any $\kappa >0$. 
Taking this together with \eqref{eq:wxx3.5-obj}, we can reach
\begin{align}
 & \sum_{n,w,m\in\mathcal{B}(\kappa)}\mathbb{P}\left(\max_{h\in\mathcal{H}} \sum_{i=1}^{k}w_{i}\left(\frac{1}{n_{i}}\sum_{j=1}^{n_{i}}\ell^{m_{i}}\big(h,(x_{i,j},y_{i,j})\big)-L_{i}^{\ell^{m_{i}}}(h,w)\right)\geq\varepsilon'\right)\nonumber\\
 & \qquad \leq \left(2kT_{1}+1\right)^{d}\cdot\left(\frac{8kT_{1}R}{\varepsilon_{1}}\right)^{k}\cdot\min_{\lambda\in[0,\kappa]}\exp(2\lambda^{2}/\kappa-\lambda\varepsilon')\nonumber\\
 & \qquad \leq \left(2kT_{1}+1\right)^{d}\cdot\left(\frac{8kT_{1}R}{\varepsilon_{1}}\right)^{k}\cdot\exp(-\kappa(\varepsilon')^{2}/8).\label{eq:wxx5-obj}
\end{align}
By setting $\varepsilon' = {\varepsilon_1}/{8}$ and $\kappa_0 = \frac{d\log(2kT_1)+k\log(8kT_1R/\varepsilon_1)+\log(\delta/8)}{8(\varepsilon')^2}\leq T_1/2$, we can show that: with probability exceeding $1-\delta/8$,  
\begin{align}
\max_{h\in\mathcal{H}}\sum_{i=1}^{k}w_{i}\left(\frac{1}{n_{i}}\sum_{j=1}^{n_{i}}\ell^{m_{i}}\big(h,(x_{i,j},y_{i,j})\big)-L_{i}^{\ell^{m_{i}}}(h,w)\right)<\varepsilon'
\end{align}
holds simultaneously for all $ (n,w,m )\in \mathcal{B}(\kappa_0) $. 
Moreover, observing that $n_i^t\geq 2\kappa_0\sum_{\ell\in \mathcal{L}}u_{i,\ell}^t$, we see that: with probability at least $1-\delta/8$, for any $1\leq t\leq T$ one has
\begin{align}
\max_{h\in\mathcal{H}}\sum_{i=1}^{k}\left(\sum_{\ell\in\mathcal{L}}u_{i,\ell}^{t}\right)\max_{\ell\in\mathcal{L}}\left(\frac{1}{n_{i}^{t}}\sum_{j=1}^{n_{i}^{t}}\ell\big(h,(x_{i,j},y_{i,j})\big)-L_{i}^{\ell}(h)\right)<\varepsilon'+\frac{\varepsilon_{1}}{4}\leq\frac{\varepsilon_{1}}{2}.\label{eq:ee1-obj}
\end{align}

Applying the same arguments, we can also show that: with probability exceeding $1-\delta/8$, for any $1\leq t \leq T$ one has
\begin{align}
\max_{h\in\mathcal{H}}\sum_{i=1}^{k}\left(\sum_{\ell\in\mathcal{L}}u_{i,\ell}^{t}\right)\max_{\ell\in\mathcal{L}}\left(L_{i}^{\ell}(h)-\frac{1}{n_{i}^{t}}\sum_{j=1}^{n_{i}^{t}}\ell\big(h,(x_{i,j},y_{i,j})\big)\right)\leq\frac{\varepsilon_{1}}{2}.\label{eq:ee2-obj}
\end{align}
Combining \eqref{eq:lxu-obj}, \eqref{eq:rxu-obj}, \eqref{eq:ee1-obj} and \eqref{eq:ee2-obj} allows one to conclude that: with probability at least $1-\delta/4$, 
\begin{align}
\left|\widehat{L}(h,u^t)-L(h,u^t) \right|\leq \frac{\varepsilon_1}{2}\nonumber
\end{align}
holds simultaneously for every $1\leq t\leq T$ and every $h\in \mathcal{H}$. 
The proof is thus completed.

\subsection{Proof of Lemma~\ref{lemma:opt_multi}}

In the rest of Appendix~\ref{app:multiobj}, we take $$\mathsf{OPT} \coloneqq \max_{u\in \Delta([k]\times \mathcal{L})}\min_{h\in \mathcal{H}}L(h,u) = \min_{\pi\in \Delta(\mathcal{H})}\max_{i\in [k],\ell\in \mathcal{L}}L_i^{\ell}(h_{\pi})\leq  \min_{h\in \mathcal{H}}\max_{i\in [k],\ell\in \mathcal{L}} L^{\ell}_i(h).$$ 
Let 
$$
    v^t = L(h^t,u^t)-\mathsf{OPT} 
    \qquad 
    \text{and}
    \qquad 
    f^t = \min_{h\in \mathcal{H}}L(h,u^t)- \mathsf{OPT}\leq 0.
$$
By virtue of Lemma~\ref{lemma:opth_multi}, we know that with probability at least $1-\delta/4$, 
$$L(h^t,u^t)\leq \min_{h\in \mathcal{H}}L(h,u^t)+\varepsilon_1,$$ 
and we have $v^t \leq f^t + \varepsilon_1 \leq  \varepsilon_1$ for any $1\leq t\leq T$.

Let $\delta'= \frac{\delta}{4(T+kR+1)}$.  Note that $\eta \leq 1$ and $|\widehat{r}_{i,\ell}^t|\leq 1$ for any proper $(i,\ell,t)$. 
Direct computation gives
\begin{align}
\log\left(\frac{\sum_{i=1}^k\sum_{\ell\in \mathcal{L}} U_{i,\ell}^{t+1}}{\sum_{i=1}^k\sum_{\ell\in \mathcal{L}} U_{i,\ell}^t}\right)  & = \log\left(\sum_{i=1}^k \sum_{\ell\in \mathcal{L}}u_{i,\ell}^t \exp(\eta \widehat{r}^t_{i,\ell}) \right)\leq \log\left( \sum_{i=1}^k\sum_{\ell \in \mathcal{L}} u_{i,\ell}^t \big(1+ \eta\widehat{r}_{i,\ell}^t + \eta^2 (\widehat{r}_{i,\ell}^t)^2 \big)  \right)  
\notag\\
& \leq \eta \sum_{i=1}^k \sum_{\ell\in \mathcal{L}}u_{i,\ell}^t \widehat{r}_{i,\ell}^t + \eta^2.\label{eq:bb1}
\end{align}
Let  $\hat{r}^t = [\hat{r}_{i,\ell}^t]_{i\in [k],\ell\in \mathcal{L}}$. As a result, we can obtain
\begin{align}
\eta\sum_{t=1}^T  \langle u^t , \widehat{r}^t \rangle &  \geq \log \left(\sum_{i=1}^k \sum_{\ell\in \mathcal{L}}U_{i,\ell}^{T+1}\right)- T\eta^2 -\log(kR) \nonumber
\\ & \geq \max_{i,\ell}\log(U_{i,\ell}^{T}) - T\eta^2 - \log(k) \nonumber
\\ & \geq \eta \max_{i,\ell} \sum_{t=1}^T \widehat{r}_{i,\ell}^t - T\eta^2 -\log(kR).
\end{align}
Dividing both side with $\eta$ yields
\begin{align}
\sum_{t=1}^T \langle u^t , \widehat{r}^t \rangle \geq \max_{i,\ell} \sum_{t=1}^T \widehat{r}^t_{i,\ell} - \left(   \frac{\log(kR)}{\eta} +\eta T \right)
.\end{align}

In addition,  Azuma's inequality tells us that: with probability at least $1-(kR+1)\delta'$, for any $i\in [k],\ell\in \mathcal{L}$ one has
%
\begin{align}
 \left| \sum_{t=1}^T \langle u^t,  \widehat{r}^t \rangle -\sum_{t=1}^T  L(h^t,u^t)\right|&\leq 2\sqrt{T\log(1/\delta')};\nonumber
 \\  \left| \sum_{t=1}^T \widehat{r}_{i,\ell}^t - \sum_{t=1}^T L^{\ell}_i(h^t)\right|&\leq 2\sqrt{T\log(1/\delta')}.\nonumber
\end{align}
Thus, with probability exceeding $1-(k+1)\delta'$, it holds that
\begin{align}
\sum_{t=1}^T L(h^t,u^t) &\geq \max_{i,\ell} \sum_{t=1}^T L^{\ell}_i(h^t) - \left( \frac{\log(kR)}{\eta}+\eta T + 2\sqrt{T\log(1/\delta')} \right)\nonumber
\\ & \geq - \left( \frac{\log(kR)}{\eta}+\eta T + 2\sqrt{T\log(1/\delta')} \right).
\end{align}

Let us overload the notation by letting $\text{e}^{\mathsf{basis}}_i$ be the $i$-th standard basis in $\mathbb{R}^{kR}$. 
In view of the fact that $\varepsilon_1 = \eta = \frac{1}{100}\varepsilon$ and $T = \frac{20000\log(\frac{kR}{\delta' \varepsilon})}{\varepsilon^2}$, we can demonstrate that 
\begin{align}
\max_i \sum_{t=1}^T L(h^t,\text{e}^{\mathsf{basis}}_i) & \leq T\mathsf{OPT}+ \sum_{t=1}^T v^t+\left( \frac{\log(kR)}{\eta}+\eta T + 2\sqrt{T\log(1/\delta')} \right) \nonumber
\\ & \leq T\mathsf{OPT}+ T\varepsilon_1  +\left( \frac{\log(kR)}{\eta}+\eta T + 2\sqrt{T\log(1/\delta')} \right)  \nonumber
\\ & \leq T\mathsf{OPT}+ T\varepsilon.
\end{align}
As a result, we arrive at
\begin{align}
\max_{i\in [k],\ell\in \mathcal{L}}L^{\ell}_i(h^{\mathsf{final}}) = \max_{i\in [k],\ell\in \mathcal{L}}\frac{1}{T}\sum_{t=1}^T L^{\ell}_i(h^t)
=\max_i \frac{1}{T} \sum_{t=1}^T L(h^t,\text{e}^{\mathsf{basis}}_i) \leq \mathsf{OPT}+\varepsilon.
\end{align}
We can then conclude the proof by observing that $\delta'=\frac{\delta}{4(T+kR+1)}$.

\subsection{Proof of Lemma~\ref{lemma:sc_obj}}

It is easily seen that the sample complexity of Algorithm~\ref{alg:obj} is bounded by
\begin{align}
O\left( T_1 \sum_{i=1}^k \overline{w}^{T}_i + Tk \left(\sum_{i=1}^k \overline{w}^T_i + 1\right) \right) = O\left(\frac{d\log\left( \frac{d}{\varepsilon}\right)+k\log\left(\frac{Rk}{\delta\varepsilon} \right)}{\varepsilon^2} \cdot \sum_{i=1}^k \overline{w}^{T}_i \right).\label{eq:objsc1}
\end{align}
As a result, it comes down to bounding $\sum_{i=1}^k \overline{w}^T_{i}$, which we accomplish below.

With slight abuse of notation,  define $$\mathcal{W}_j = \big\{  i\in [k] \mid \overline{w}^{T}_i \in (2^{-j},2^{-(j-1)}]   \big\}$$ for any $1\leq j \leq \left\lfloor \log(1/k) \right\rfloor+1 $, and we would like to bound the size of each $\mathcal{W}_j$ separately. This is the focus of the following lemma. Recall that $\tilde{j} = \left\lfloor \log_2\left( \frac{k\log^2(2)}{50(\log_2(1/\eta)+1)^2\log_2^2(k)}\right) \right \rfloor-2$.
\begin{lemma}\label{lemma:bdwj}
Assume the conditions in Lemma~\ref{lemma:opth_multi} and Lemma~\ref{lemma:opt_multi} hold. Let $\delta' = \frac{\delta}{32T^4k^2}$.
For any $1\leq j \leq \tilde{j}$, with probability at least $1-8T^4k\delta'$, it holds that
\begin{align}
|\mathcal{W}_j|\leq  8\cdot 10^7 \cdot \left( (\log_2(k)+1)^5 (\log(kR)+\log(1/\delta'))(\log_2(T)+1) \right)\cdot 2^j.\nonumber
\end{align}
\end{lemma}

Armed with Lemma~\ref{lemma:bdwj}, we can demonstrate that: with probability exceeding $1-8T^4k^2\delta' = 1-\delta/4$, one has
\begin{align}
\sum_{i=1}^k \overline{w}^T_i  &\leq k \cdot 2^{-\tilde{j}-1} + \sum_{j=1}^{\tilde{j}} |\mathcal{W}_j|\cdot 2^{-(j-1)} \nonumber
\\ & \leq 2\cdot 10^8\cdot \left(   (\log_2(k)+1)^6 (\log(kR)+\log(1/\delta)) (\log_2(T)+1  \right).\nonumber
\end{align}

The proof of Lemma~\ref{lemma:sc_obj} is thus complete by noting that $T = O\left( \frac{d\log(d/\varepsilon)+k\log(Rk/(\varepsilon\delta))}{\varepsilon^2}\right)$.

\begin{remark}
The proof of Lemma~\ref{lemma:bdwj} is similar to that of Lemma~\ref{lemma:bdcount} by regarding $\sum_{\ell\in \mathcal{L}}u_{i,\ell}^t$ as $w_i^t$. 
\end{remark}

\subsection{Proof of Lemma~\ref{lemma:bdwj}}

 With slight abuse of notation, we re-define $$w^t_i = \sum_{\ell\in \mathcal{L}}u_{i,\ell}^t$$ by running Algorithm~\ref{alg:obj} for all $i\in [k]$ and $1\leq t\leq T$. Recall the definition of segments in Definition~\ref{def:seg}. Similar to Lemma~\ref{lemma:aux1}, for each $i\in \mathcal{W}_j$, there exists a $\left( \frac{1}{2^{j+1}},\frac{1}{2^{j+2}},\log(2) \right)$-segment with the index set $\{i\}$. We can then present the following counterpart of Lemma~\ref{lemma:aux1}.
\begin{lemma}\label{lemma:aux2}
 For each $i\in \mathcal{W}_j$, there exist $1\leq s_1<e_i \leq T$ satisfying $\frac{1}{2^{j+2}}<w_i^{s_i}\leq \frac{1}{2^{j+1}}$, $\frac{1}{2^j}<w_i^{e_i}$, and $w_i^t >2^{-(j+2)}$ for any $s_i \leq t\leq e_i$.
 \end{lemma}
 \begin{proof}
Repeating the arguments in proof of Lemma~\ref{lemma:aux1}, we can readily complete the proof by noting that
\begin{align}
\log_2\left(\frac{w_i^{t+1}}{w_i^t}\right) = \log_2\left(\frac{\sum_{\ell}u_{i,\ell}^{t+1}}{\sum_{\ell}u^t_{i,\ell}} \right) \leq \frac{\eta}{\log(2)}< \frac{1}{4}.\nonumber
\end{align}
 \end{proof}

Armed with  Lemma~\ref{lemma:aux2}, we continue to present the counterpart of Lemma~\ref{lemma:part3} in the following lemma. Note that the proof of Lemma~\ref{lemma:part3obj} is exactly the same as that of Lemma~\ref{lemma:part3} with the new definitions of $\mathcal{W}_j$ and $\{w_i^t\}_{i\in [k], t\in [T]}$.
\begin{lemma}\label{lemma:part3obj}
Given $\mathcal{W}_j$ and $(s_i,e_i)$ for $i\in \mathcal{W}_j$ defined above,  there exists  a group of subsets $\{\mathcal{V}_j^n\}_{n=1}^N$ such that the conditions below hold
\begin{enumerate}[label=(\roman*).]
\item $\mathcal{V}_j^n\subset \mathcal{W}_j$, $\mathcal{V}_j^n\cap \mathcal{V}_j^{n'}=\emptyset$, $\forall n\neq n'$; 
\item $\sum_{n=1}^N |\mathcal{V}_j^n|\geq \frac{|\mathcal{W}_j|}{24\log_2(k)(\log_2(T)+1)}$; 
\item There exists $1\leq \widehat{s}_1<\widehat{e}_1\leq \widehat{s}_2 <\widehat{e}_2\leq \dots \leq \widehat{s}_N <\widehat{e}_N\leq T$, and $\{g_n\}_{n=1}^N\in \left[1,\infty\right)^N$ such that for each $1\leq n \leq N$,  $(\widehat{s}_n,\widehat{e}_n)$ is a $\left(2^{-(j+1)}g_n|\mathcal{V}_j^n|,  2^{-(j+2)}|\mathcal{V}_j^n| , \frac{\log(2)}{2\log_2(k)}  \right)$-segment with index set as $\mathcal{V}_j^n$. That is, the following hold for each $1\leq n \leq N$:\begin{itemize}
    \item $\frac{g_n |\mathcal{V}_j^n|}{2^{j+2}}<\sum_{i\in \mathcal{V}_j^n} w_i^{\widehat{s}_n}\leq \frac{g_n |\mathcal{V}_j^n|}{2^{j+1}}$ ;
    \item $\frac{g_n|\mathcal{V}_j^n|}{2^{j}}\cdot \exp\left( \frac{\log(2)}{2\log_2(k)}\right)<\sum_{i\in \mathcal{V}_j^n}w_i^{\widehat{e}_n}$;
    \item $\sum_{i\in \mathcal{V}_j^n}w_i^{t}\geq \frac{|\mathcal{V}_j^n|}{2^{j+2}}$ for any $\widehat{s}_n \leq t\leq \widehat{e}_n$.
\end{itemize}

\end{enumerate}
\end{lemma}

The next step is to introduce the counterpart of Lemma~\ref{lemma:cut} as follows.  Note that most analysis arguments for establishing Lemma~\ref{lemma:cut_obj} are the same as the ones used in the proof of Lemma~\ref{lemma:cut}, except that the noise term is a slightly different. The proof is provided in Appendix~\ref{sec:proof-lemma:cut-obj}.

\begin{lemma}\label{lemma:cut_obj} Assume the conditions in  Lemma~\ref{lemma:opth_multi} and Lemma~\ref{lemma:opt_multi} hold. Recall the definition of $h^t$ and $u^t$ in Algorithm~\ref{alg:obj}, and the definition that $\mathsf{OPT}=\min_{h\in \mathcal{H}}\max_{i\in [k],\ell\in \mathcal{L}}L^{\ell}_i(h)$. 
Also recall that $v^{t} =L(h^t,u^t)-\mathsf{OPT}$.
Suppose $(t_1,t_2)$ is a $\left(p,q,x\right)$-segment such that $p\geq 2q$. Then one has
\begin{align}t_2 - t_1 \geq \frac{x}{2\eta}.\label{eq:cl11}
\end{align}

Moreover, if we assume that $\frac{qx^2}{50(\log_2(k)+1)^2}\geq  \frac{1}{k}$, it then holds that:
 with probability at least $1-6T^4k\delta'$, at least one of the two claims holds: 
 \begin{itemize}
 \item[(a)] the length of the segment satisfies
 \begin{align}
t_2-t_1 \geq \frac{qx^2}{200(\log_2(1/\eta)+1)^2\eta^2};
 \end{align}
 \item[(b)] the quantities $\{v^t\}$ obey
 \begin{align}
4\sum_{\tau=t_1}^{t_2-1}(-v^{\tau}+\varepsilon_1)\geq \frac{qx^2}{100 (\log_2(1/\eta)+1)^2\eta}.
 \end{align}
\end{itemize}

\end{lemma}

With the preceding lemmas in place, we are now ready to prove Lemma~\ref{lemma:bdwj}.
In what follows, we denote by $\{\mathcal{V}_j^n\}_{n=1}^N$ and $\{(\widehat{s}_n,\widehat{e}_n)\}_{n=1}^N$  the construction in Lemma~\ref{lemma:part3obj}. 

First of all, we make the observation that: for any $1\leq j \leq \tilde{j}$, one has
\begin{align}
2^{-(j+2)}|\mathcal{V}_j^n|\cdot \frac{\log^2(2)}{50 (\log_2(1/\eta)+1)^2\log_2^2(k)}\geq 2^{-(\tilde{j}+2)}\cdot \frac{\log^2(2)}{50(\log_2(1/\eta)+1)^2\log_2^2(k)}\geq \frac{1}{k}.\nonumber
\end{align}
Recall that $v^{\tau}\leq \varepsilon_1$. Combining this fact with Lemma~\ref{lemma:cut_obj} (by setting $q= 2^{-(j+2)}|\mathcal{V}_j^n|$ and $x = \frac{\log(2)}{\log_2(k)}$) allows us to show that, for each $1\leq n \leq N$, 
\begin{align}
T\eta + \left\{ 4T \varepsilon_1 + 4\sum_{t=1}^T(-v^t)\right\} & \geq \sum_{n=1}^N \left(\widehat{e}_n - \widehat{s}_n \right)\eta + 4\sum_{n=1}^N \sum_{\tau = \widehat{s}_n}^{\widehat{e}_n-1}(-v^{\tau}+\varepsilon_1) \nonumber
\\ & \geq \frac{2^{-(j+2)}\sum_{n=1}^N |\mathcal{V}_j^n|\log^2(2)}{800\log^2_2(k)(\log_2(1/\eta)+1)^2\eta} ,\label{eq:ff1}
\end{align}
where the first inequality makes use of the disjoint nature of the segments $\{(\widehat{s}_n, \widehat{e}_n)\}_{1\leq n \leq N}$, and the second inequality follows from  Lemma~\ref{lemma:cut_obj}.
In addition,  Lemma~\ref{lemma:opt_multi} tells us that
\begin{align}
\sum_{t= 1}^T (-v^{t}) \leq  100 \left( \frac{\log(kR)}{\eta}+\eta T + 2\sqrt{T\log(1/\delta')}\right).
\end{align}
Taking this collectively with \eqref{eq:ff1} gives
 \begin{align}
\sum_{n=1}^{N} |\mathcal{V}_j^n| & \leq  \frac{3200(\log_2(k)+1)^4\eta \cdot 2^{j+2}}{\log^2(2)}\cdot \left( 100 \left( \frac{\log(kR)}{\eta}+\eta T +2\sqrt{T\log(1/\delta')}\right)\right)+
\\ & \qquad \qquad \qquad \qquad \qquad \qquad \qquad \qquad +\frac{4000T(\log_2(k)+1)^4\cdot   2^{j+2}\eta^2}{\log^2(2)}  .\nonumber
 \end{align}
  Finally,  
it follows from Property (ii) of Lemma~\ref{lemma:part3obj} that
\begin{align}
|\mathcal{W}_j| & \leq 24\log_2(k)(log_2(T)+1)\left(\sum_{n=1}^N |\mathcal{V}_j^n| \right) \nonumber
\\ & \leq 8\cdot 10^7\cdot \left((\log_2(k)+1)^5\left(\log(kR)+\log(1/\delta')\right)(\log_2(T)+1)\right) \cdot 2^j,\nonumber
\end{align}
thus concluding the proof.

\subsection{Proof of Lemma~\ref{lemma:cut_obj}}
\label{sec:proof-lemma:cut-obj}

Throughout this proof, let us take 
$$
    Z^t = \sum_{i\in [k],\ell \in \mathcal{L}}U^t_{i,\ell}.
$$
We begin by establishing the first claim (\ref{eq:cl11}). By definition, there exists $i\in [k]$ such that
\begin{align}
\log\left(  \frac{w^{t_2}_i}{w^{t_1}_i}\right) \geq x,\nonumber
\end{align}
which corresponds to 
\begin{align}
 \eta \sum_{\tau = t_1}^{t_2 - 1}\max_{\ell\in \mathcal{L}} \widehat{r}_{i,\ell}^{\tau } -\log\left(\frac{Z^{t_2}}{Z^{t_1}}\right)\geq x. \nonumber
\end{align}
Given that  $\log(Z^{t_2}/Z^{t_1})\geq -\eta (t_2-t_1)$ and $\widehat{r}_{i,\ell}^{\tau}\leq 1$ for any $1\leq \tau \leq T$, we can demonstrate that 
\begin{align}
x\leq 2(t_2-t_1)\eta, \nonumber
\end{align}
from which the claim (\ref{eq:cl11}) follows.

We now turn to the remaining claims of Lemma~\ref{lemma:cut_obj}. For each hypothesis $h\in \mathcal{H}$, let us introduce the following vector $\overline{\lambda}_{h}\in\mathbb{R}^{kR}$:
\begin{align}
\overline{\lambda}_h = [\overline{\lambda}_{h,i,\ell}]_{i\in [k],\ell\in \mathcal{L}}\qquad \qquad \text{with } ~\overline{\lambda}_{h,i,\ell} = L^{\ell}_i(h)-\mathsf{OPT}.\label{eq:defoverlam}
\end{align}
Give the $\varepsilon$-optimality of $h^t$ ( see Lemma~\ref{lemma:opth_multi}), we have the following property that holds for any $1\leq \tau,t \leq T$,
\begin{align}
\langle \overline{\lambda}_{h^{\tau}}, u^{t} \rangle\geq \langle \overline{\lambda}_{h^t}, u^{t}\rangle -\varepsilon_1= v^t - \varepsilon_1.\label{eq:s2obj}
\end{align}
Summing over $\tau$ leads to
\begin{align}
\sum_{\tau=t_1}^{t_2-1}\langle \overline{\lambda}_{h^\tau}, u^{t} \rangle\geq (t_2-t_1)(v^t-\varepsilon_1).\label{eq:ksx1}
\end{align}
In the sequel, we divide the remaining proof into several steps.

\medskip
\noindent \textbf{Step 1: decomposing the KL divergence between $u^t$ and $u^{t_2}$.}
Write
\begin{align}
U^t_{i,\ell} = \exp\left( \eta \sum_{\tau=1}^t \hat{r}^{\tau}_{i,\ell}\right) = \exp\left(\eta \sum_{\tau = 1}^t \left(\overline{\lambda}_{h^\tau,i,\ell} + \mathsf{OPT}+ \xi^{\tau}_{i,\ell}\right)\right) \qquad \qquad \text{with }~\xi_{i,\ell}^{\tau}= \widehat{r}_{i,\ell}^{\tau}- \overline{h}^{\tau}_{i,\ell}-\mathsf{OPT},\end{align}
where $\xi_{i,\ell}^{\tau}$ is a zero-mean random variable. Re-define
\begin{align}
\Delta_{t_1,t_2} = \sum_{\tau = t_1}^{t_2-1}\xi^{\tau} \qquad \qquad \text{with } 
\xi^{\tau} = [\xi^{\tau}_{i,\ell}]_{i\in [k],\ell\in \mathcal{L}}
\end{align}

Taking $U^t = [U^t_{i,\ell}]_{i\in [k],\ell\in \mathcal{L}}$ and  denoting $\log(x/y)$  the vector $\{\log(x_{i,\ell}/y_{i,\ell})\}_{i\in [k],\ell\in \mathcal{L}}$ for two $kR$-dimensional vectors $(x,y)$, one can then deduce that 
\begin{align}
\left\langle \frac{1}{\eta}\log\left( \frac{W^{t_2}}{W^{t_1}}\right)  - \Delta_{t_1,t_2}, u^t \right \rangle - (t_2-t_1)\mathsf{OPT} = \sum_{\tau  =t_1}^{t_2-1}\langle \overline{\lambda}_{h^{\tau}}, u^t \rangle \geq (t_2-t_1)(v^t - \epsilon_1),\label{eq:cwf1}
\end{align}
where the last inequality results from \eqref{eq:ksx1}. Recall that $Z^t = \sum_{i\in [k],\ell\in \mathcal{L}}U^t_{i,\ell}$ and $u^t_{i,\ell} = \frac{U^t_{i,\ell}}{Z^t}$. By taking $t_1 = t$, we can derive from \eqref{eq:cwf1} that 
\begin{align}
\left\langle  \log\left(\frac{w^{t_2}}{w^{t}}\right) -\eta \Delta_{t,t_2} , u^t\right \rangle + \log\left( \frac{Z^{t_2}}{Z^t}\right) -\eta (t_2 -t)\mathsf{OPT}\geq \eta(t_2-t)(v^t - \varepsilon_1).
\end{align}
This result in turn assists in bounding the KL divergence between $u^t$ and $u^{t_2}$:
\begin{align}
\mathsf{KL}(u^t \,\|\,u^{t_2})  & :=u^t \cdot \left(\log\left(\frac{u^t}{u^{t_2}}\right) \right) \nonumber
\\ & \leq \log(Z^{t_2}/Z^{t})-\eta(t_2-t)\mathsf{OPT} -\eta u^t \cdot \Delta_{t,t_2}+(t_2-t)\eta(\varepsilon_1-v^t).\label{eq:s3obj}
\end{align}
In the following, we shall cope with the right-hand side of \eqref{eq:s3obj}

\medskip
\noindent \textbf{Step 2: bounding the term $\log(Z^{t_2}/Z^t)$.} 
 With probability at least $1-2T^2k\delta'$, it holds that
\begin{align}
 \log(Z^{t_2}/Z^t)  &:=\sum_{\tau = t}^{t_2-1}\log(Z^{\tau+1}/Z^{\tau}) \nonumber
 \\ & = \sum_{\tau=t}^{t_2-1}\log\left(\sum_{i=1}^k \sum_{\ell\in \mathcal{L}}u^{\tau}_{i,\ell} \exp(\eta \widehat{r}_{i,\ell}^{\tau} ) \right)\nonumber
 \\ & \leq  \sum_{\tau = t}^{t_2-1}\left( \eta \sum_{i=1}^k u_{i,\ell}^{\tau} \widehat{r}_{i,\ell}^{\tau}+ 2\eta^2\right)\nonumber
 \\ & \leq \eta \sum_{\tau = t}^{t_2-1} v^{\tau}  +\eta(t_2-t)\mathsf{OPT} + \eta \sum_{\tau=t}^{t_2-1}\sum_{i=1}^k\sum_{\ell\in \mathcal{L}}\big( u_{i,\ell}^{\tau} (\widehat{r}_{i,\ell}^{\tau}-\overline{\lambda}_{h^{\tau},i,\ell} -\mathsf{OPT}) \big)+ 2(t_2-t)\eta^2\label{eq:eexp3obj}
 \\ & \leq \eta (t_2-t)\varepsilon_1 +\eta(t_2-t)\mathsf{OPT} +\eta \sum_{\tau=t}^{t_2-1}\sum_{i=1}^k\sum_{\ell\in \mathcal{L}}\big( u_{i,\ell}^{\tau} (\widehat{r}_{i,\ell}^{\tau}-\overline{\lambda}_{h^{\tau},i,\ell}-\mathsf{OPT}) \big)+2(t_2-t)\eta^2.\label{eq:s4obj}
\end{align}
Here, \eqref{eq:eexp3obj} holds true due to the fact that $v^t = \langle u^t, \overline{\lambda}_{h^t}\rangle$, and \eqref{eq:s4obj} is comes from the fact that $v^{\tau}\leq  \varepsilon_1$.

\medskip
\noindent \textbf{Step 3: bounding the weighted sum of $\{\xi_{i,\ell}^{\tau}\}$.}
Next, we intend to control the two random terms below
\begin{align}
\eta \sum_{\tau=t}^{t_2-1}\sum_{i=1}^k\sum_{\ell\in \mathcal{L}} u_{i,\ell}^{\tau} \cdot (\widehat{r}_{i,\ell}^{\tau}-\overline{\lambda}_{h^{\tau},i,\ell} -\mathsf{OPT}) &= \eta \sum_{\tau=t}^{t_2-1} \sum_{i=1}^k\sum_{\ell\in \mathcal{L}} u_{i,\ell}^{\tau}\xi_{i,\ell}^{\tau},\label{eq:xxx21term1}
\\-\eta \langle u^t,\Delta_{t,t_2} \rangle&=-\eta \sum_{\tau= t}^{t_2-1} \sum_{i=1}^k \sum_{\ell\in \mathcal{L}} u_{i,\ell}^t \xi_{i,\ell}^{\tau}.\label{eq:xx21obj}
\end{align}
Let $\overline{\mathcal{F}}^{\tau}$ denote all events happening before the $\tau$-th round in Algorithm~\ref{alg:obj}. 
Direct computation yields
\begin{align}
&  \sum_{\tau=t_1}^{t_2-1}\sum_{i=1}^k \sum_{\ell}u_{i,\ell}^{t_1} \left( \widehat{r}_{i,\ell}^{\tau}-\overline{\lambda}_{h,i,\ell}^{\tau} - \mathsf{OPT} \right)  \nonumber
\\  &\qquad = \sum_{\tau=t_1}^{t_2-1}\sum_{i=1}^k \frac{1}{\left\lceil k\overline{w}_i^{\tau}\right\rceil} \sum_{j=1}^{\left\lceil k\overline{w}_i^{\tau}\right\rceil}\left( \sum_{\ell}u_{i,\ell}^{t_1} \ell(h^{\tau}, (x_{i,j}^{\tau}, y_{i,j}^{\tau})) - \sum_{\ell}u_{i,\ell}^{t_1} L_i^{\ell}(h^{\tau}) \right).
\end{align}
Note that $\{\{(x_{i,j},y_{i,j})\}_{j=1}^{\left\lceil k\overline{w}_i^{\tau}\right\rceil   }\}_{i=1}^k$ are independently sampled conditioned on $\overline{\mathcal{F}}^{\tau}$. Recalling the fact that $ |\sum_{\ell}w_{i,\ell}^{t_1} \ell(h^{\tau}, (x_{i,j}^{\tau}, y_{i,j}^{\tau}))|\leq w_{i}^{t_1}\leq  \overline{w}_i^{\tau}$ for any $t_1\leq \tau \leq t_2-1$,  one can apply Freedman's inequality (cf.~Lemma~\ref{freedman}) to see that
\begin{align}
& \left|\sum_{\tau=t_1}^{t_2-1}\sum_{i=1}^k \frac{1}{\left\lceil k\overline{w}_i^{\tau}\right\rceil} \sum_{j=1}^{\left\lceil k\overline{w}_i^{\tau}\right\rceil}\left( \sum_{\ell}w_{i,\ell}^{t_1} \ell(h^{\tau}, (x_{i,j}^{\tau}, y_{i,j}^{\tau})) - \sum_{\ell}w_{i,l}^{t_1} L_i^{\ell}(h^{\tau}) \right)\right|\nonumber
\\ & \leq 2\sqrt{   \log(2/\delta') \sum_{\tau=t_1}^{t_2-1} \sum_{i=1}^k \sum_{j=1}^{ \left\lceil k\overline{w}_i^{\tau}\right\rceil }\frac{(w_i^{t_1})^2}{ (\left\lceil k\overline{w}_i^{\tau}\right\rceil )^2}  } + 2\log(2/\delta') \nonumber 
\\ & =  2\sqrt{   \log(2/\delta') \sum_{\tau=t_1}^{t_2-1} \sum_{i=1}^k\frac{(w_i^{t_1})^2}{ \left\lceil k\overline{w}_i^{\tau}\right\rceil }  } + 2\log(2/\delta') \nonumber 
\\ & \leq   2\sqrt{   \log(2/\delta') \sum_{\tau=t_1}^{t_2-1} \sum_{i=1}^k\frac{w_{i}^{t_1}}{ k  } }+ 2\log(2/\delta') \nonumber 
\\  & \leq 2\sqrt{   \log(2/\delta') \sum_{\tau=t_1}^{t_2-1}\frac{1}{k} } + 2\log(2/\delta')\label{eq:xx2obj} 
\end{align}
holds with probability exceeding $1-\delta'$. 
The bound of \eqref{eq:xxx21term1} is thus finished. We then use similar arguments to show that, with probability exceeding $1-\delta'$,
\begin{align}
& \left|\sum_{\tau=t_1}^{t_2-1}\sum_{i=1}^k \frac{1}{\left\lceil k\overline{w}_i^{\tau}\right\rceil} \sum_{j=1}^{\left\lceil k\overline{w}_i^{\tau}\right\rceil}\left( \sum_{\ell}w_{i,\ell}^{\tau} \ell(h^{\tau}, (x_{i,j}^{\tau}, y_{i,j}^{\tau})) - \sum_{l}w_{i,\ell}^{\tau} L_i^{\ell}(h^{\tau}) \right)\right|\nonumber
\\ & \leq 2\sqrt{   \log(2/\delta') \sum_{\tau=t_1}^{t_2-1} \sum_{i=1}^k \sum_{j=1}^{ \left\lceil k\overline{w}_i^{\tau}\right\rceil }\frac{(w_i^{\tau})^2}{ (\left\lceil k\overline{w}_i^{\tau}\right\rceil )^2}  } + 2\log(2/\delta') \nonumber 
\\ & =  2\sqrt{   \log(2/\delta') \sum_{\tau=t_1}^{t_2-1} \sum_{i=1}^k\frac{(w_i^{\tau})^2}{ \left\lceil k\overline{w}_i^{\tau}\right\rceil }  } + 2\log(2/\delta') \nonumber 
\\ & \leq   2\sqrt{   \log(2/\delta') \sum_{\tau=t_1}^{t_2-1} \sum_{i=1}^k\frac{w_{i}^{\tau}}{ k  } }+ 2\log(2/\delta') \nonumber 
\\  & \leq 2\sqrt{   \log(2/\delta') \sum_{\tau=t_1}^{t_2-1}\frac{1}{k} } + 2\log(2/\delta'). \label{eq:xx3obj} 
\end{align}

\medskip
\noindent 
\textbf{Step 4: bounding the KL divergence between $u^t$ and $u^{t_2}$.}
Combining \eqref{eq:s3obj}, \eqref{eq:s4obj}, \eqref{eq:xx2obj} and \eqref{eq:xx3obj}, and applying the union bounds over $(t,t_2)$, we see that with probability exceeding $1-6T^4k\delta'$,  
\begin{align}
\mathsf{KL}(w^t\,\|\, w^{t_2})\leq \mathsf{KL}(u^t\,\|\, u^{t_2}) & \leq 2(t_2-t)\eta\varepsilon_1 - (t_2-t)\eta v^t  \nonumber
\\ & \quad \quad \quad + 4\eta\sqrt{\frac{(t_2-t)\log(2/\delta')}{k}}+2(t_2-t)\eta^2+4\log(2/\delta')\eta\label{eq:s6obj}
\end{align}
holds for any $1\leq t<t_2\leq T$.
The analysis below then operates under the condition that \eqref{eq:s6obj} holds for any $1\leq t< t_2 \leq T$.

\medskip
\noindent \textbf{Step 5: connecting the KL divergence with the advertised properties.} Recall that  $ j_{\mathrm{max}}=\left\lfloor \log_2(1/\eta)+1 \right \rfloor$. Set
\begin{align}
  \tau_{\widehat{j}} &\coloneqq\min \big\{\tau \mid t_1\leq \tau \leq t_2-1, \,  -v^{\tau}\leq 2^{-(\widehat{j}-1)}  \big\}, \qquad 1\leq \widehat{j} \leq j_{\mathrm{max}} \label{eq:v1j}
 \\  \tau_{j_{\mathrm{max}+1}} &\coloneqq t_2.\nonumber
 \end{align}
By definition, we know that $\tau_1 = t_1$ and $\tau_{j_2}\geq \tau_{j_1}$ for $j_2\geq j_1$. Let $\mathcal{I}$ be the index set of this segment.
Let $y_{j}:= \sum_{i\in \mathcal{I}}w_i^{\tau_j}$. 

We then claim that there exists $1\leq \widetilde{j} \leq j_{\mathrm{max}}$ such that
\begin{align}
  \log\left( \frac{y_{\widetilde{j}+1} }{y_{\widetilde{j}}} \right) &\geq \frac{x}{\log_2(1/\eta)+1} .\label{eq:ss1obj}
\end{align}
\begin{proof}[Proof of \eqref{eq:ss1obj}]Suppose that there exists no $1\leq \widetilde{j}\leq j_{\mathrm{max}}$ satisfying  \eqref{eq:ss1obj}.  Then for any $j$, it holds that $\log\left(  \frac{y_{j+1}}{y_j}\right)\leq \frac{x}{\log_2(1/\eta)+1}$, which implies that $y_{j}\geq y_{j+1}\exp\left(\frac{x}{\log_2(1/\eta)+1}\right)$. As a result, we have  
$$y_{1}\geq y_{j_{\mathrm{max}}+1}\cdot \exp\left( -j_{\mathrm{max}}\cdot \frac{x}{\log_2(1/\eta)+1}\right)\geq p,$$ 
which leads to contradiction (since according to the definition of the $(p,q,x)$-segment, one has $y_1\leq p$).
\end{proof}

Now, assume that $\widetilde{j}$ satisfies \eqref{eq:ss1obj}. From the definition of the $(p,q,x)$-segment, we have $y_{\widetilde{j}}\geq q$. It then follows from
 \eqref{eq:s6obj} that
\begin{align}
\mathsf{KL}(w^{\widetilde{j}}\,\|\, w^{\tau_{\widetilde{j}+1}}) & \leq 2 (\tau_{\widetilde{j}+1}-\tau_{\widetilde{j}})\eta \varepsilon_1 + (\tau_{{\widetilde{j}}+1}-\tau_{\widetilde{j}} )\eta 2^{-(\widetilde{j}-1)} \nonumber
\\ & \quad + 4\eta \sqrt{\frac{(\tau_{\widetilde{j}+1}-\tau_{\widetilde{j}})\log(2/\delta')}{k}}+2(\tau_{\widetilde{j}+1}-\tau_{\widetilde{j}})\eta^2 +4\log(2/\delta')\eta.\nonumber
\end{align}
Since $\log\left( \frac{y_{\widetilde{j}+1}}{y_{\widetilde{j}}}\right)\geq \frac{x}{\log_2(1/\eta)+1}$ and $y_{\widetilde{j}}\geq q$, 
we can invoke Lemma~\ref{lemma:klbound} and Lemma~\ref{lemma:klcmp} to show that  $$\mathsf{KL}(w^{\tau_{\widetilde{j}}} \,\|\,w^{\tau_{\widetilde{j}+1}})\geq 
\mathsf{KL}\left(\mathrm{Ber}\left(y_{\tilde{j}}\right)\,\|\,\mathrm{Ber}\left(y_{\tilde{j}+1}\right)\right) \geq \frac{qx^2}{4(\log_2(1/\eta)+1)^2},$$ 
where $\mathrm{Ber}(x)$ denote the Bernoulli distribution with parameter $x\in [0,1]$. 
As a result, we obtain 
\begin{align}
\frac{qx^2 }{4(\log_2(1/\eta)+1)^2} & \leq 2 (\tau_{\widetilde{j}+1}-\tau_{\widetilde{j}})\eta \varepsilon_1 + (\tau_{{\widetilde{j}}+1}-\tau_{\widetilde{j}} )\eta 2^{-({\widetilde{j}}-1)} \nonumber
\\ & \quad + 4\eta \sqrt{\frac{(\tau_{\widetilde{j}+1}-\tau_{\widetilde{j}})\log(2/\delta')}{k}}+2(\tau_{\widetilde{j}+1}-\tau_{\widetilde{j}})\eta^2 +4\log(2/\delta')\eta.\label{eq:uadobj}
\end{align}
Combining \eqref{eq:uadobj} with the facts that: (i) $\frac{qx^2}{50 (\log_2(1/\eta)+1)^2}\geq \frac{1}{k}$; (ii) $\frac{2^{j_{\mathrm{max}}}}{\eta}\geq \frac{1}{\eta^2}$, we can demonstrate that
\begin{align}
\tau_{\widetilde{j}+1}-\tau_{\widetilde{j}} & \geq  \min\left\{\frac{qx^2 }{100(2\log_2(k)+1)^2}\min\left\{ \frac{1}{\eta\varepsilon_1} ,\frac{2^{\widetilde{j}-1}}{\eta}, \frac{1}{\eta^2} \right\} , \frac{kx^2 2^{-2l_1}}{10000\eta^2 \log(1/\delta')(\log_2(k)+1)^4} \right\} \nonumber
\\ & \geq \frac{qx^2 \cdot 2^{\widetilde{j}-1}}{100 (\log_2(1/\eta) +1)^2\eta  }.\label{eq:xx41obj}
\end{align}
With \eqref{eq:xx41obj} in place, we are ready to finish the proof by dividing into two cases.

\begin{itemize}
\item Case 1: $\tilde{j}=j_{\mathrm{max}}.$ It follows from \eqref{eq:xx41obj} that
$$t_2-t_1\geq\tau_{\tilde{j}+1}-\tau_{\tilde{j}}\geq \frac{qx^2}{100(\log_2(1/\eta)+1)}\frac{2^{j_{\mathrm{max}-1}}}{\eta}\geq \frac{qx^2}{200(\log_2(1/\eta)+1)^2\eta^2}.$$
\item Case 2: $1\leq \tilde{j}\leq j_{\mathrm{max}}-1$. It comes from the definition \eqref{eq:v1j} that 
\begin{align}
\sum_{\tau=t_1}^{t_2-1}\mathds{1}\{-v^{\tau}\geq 2^{-\widehat{j}}\}\geq \sum_{\tau =t_1}^{t_2-1}\mathds{1}\{-v^{\tau}>2^{-\widehat{j}}\}\geq \tau_{\widehat{j}+1}-t_1\geq \tau_{\widehat{j}+1}-\tau_{\widehat{j}}, \qquad \qquad  \text{for any } 1\leq \widehat{j}\leq j_{\mathrm{max}}-1.\nonumber
\end{align}
When $1\leq \tilde{j}\leq j_{\mathrm{max}}-1$, the above display taken collectively with \eqref{eq:s6obj} gives 

\begin{align}
	\sum_{\tau=t_1}^{t_2-1}\mathds{1}\big\{ -v^{\tau}\geq 2^{-\widetilde{j}} \big\}
	\geq  \tau_{\widetilde{j}+1}-\tau_{\widetilde{j}}\geq \frac{qx^2\cdot 2^{\widetilde{j}-1} }{100\big(\log_2(1/\eta)+1\big)^2\eta}
		,
\end{align}
and as a result, 
\begin{align}
\sum_{\tau=t_{1}}^{t_{2}-1}\sum_{\widehat{j}=1}^{j_{\mathrm{max}}-1}\mathds{1}\big\{\max\{-v^{\tau},0\}\geq2^{-\widehat{j}}\big\} & 2^{-(\widehat{j}-1)}\geq\sum_{\tau=t_{1}}^{t_{2}-1}\mathds{1}\big\{-v^{\tau}\geq2^{-\widetilde{j}}\big\}2^{-(\widetilde{j}-1)}\geq\frac{qx^{2}}{100\big(\log_{2}(1/\eta)+1\big)^{2}\eta}.
\label{eq:lls-123-obj}
\end{align}
By observing that 
$$
	\sum_{\widehat{j}=1}^{\infty}\mathds{1}\big\{ x\geq 2^{-\widehat{j}}\big\} \cdot 2^{-(\widehat{j}-1)}\leq 4x 
$$
holds for any $x\geq 0$, one can combine this fact with \eqref{eq:lls-123-obj} to show that
\begin{align}
	4\sum_{\tau=t_{1}}^{t_{2}-1}\max\{-v^{\tau},0\} & \geq\frac{qx^{2}}{100\big(\log_{2}(1/\eta))+1\big)^{2}\eta}.
	\label{eq:sum-max-vtau-proof-obj}
\end{align}
Additionally, recalling that $v^{\tau}\leq \varepsilon_1$, one can deduce that
\[
	\sum_{\tau=t_{1}}^{t_{2}-1}\max\{-v^{\tau},0\}=\sum_{\tau=t_{1}}^{t_{2}-1}(-v^{\tau})-\sum_{\tau=t_{1}}^{t_{2}-1}\min\{-v^{\tau},0\}\leq\sum_{\tau=t_{1}}^{t_{2}-1}(-v^{\tau})+(t_{2}-t_{1})\varepsilon_{1},
\]
which combined with \eqref{eq:sum-max-vtau-proof-obj} yields
\[
	4(t_{2}-t_{1})\varepsilon_{1}+4\sum_{\tau=t_{1}}^{t_{2}-1}(-v^{\tau})\geq\frac{qx^{2}}{100\big(\log_{2}(1/\eta)+1\big)^{2}\eta}.
\]

\end{itemize}
The proof is thus complete.

\bibliography{ref}
\bibliographystyle{apalike}


\end{document}